\documentclass[11pt]{article}
\pdfoutput=1

\usepackage{mathrsfs}
\usepackage{amsmath,amssymb}
\usepackage{bm}
\usepackage{natbib}
\usepackage[usenames]{color}
\usepackage{amsthm}
\usepackage{bbm}
\usepackage{multirow} 
\usepackage{enumitem}

\usepackage[colorlinks,
linkcolor=red,
anchorcolor=blue,
citecolor=blue
]{hyperref}

\numberwithin{equation}{section}
\newtheorem{theorem}{Theorem}
\newtheorem{lemma}{Lemma}
\newtheorem{proposition}{Proposition}

\newtheorem{definition}{Definition}

\newtheorem{condition}{Condition}
\usepackage{mylatexstyle}
\usepackage{dsfont}
\usepackage{setspace}
\usepackage[left=1in, right=1in, top=1in, bottom=1in]{geometry}

\usepackage{xcolor}

\ifdefined\final
\usepackage[disable]{todonotes}
\else
\usepackage[textsize=tiny]{todonotes}
\fi
\title{\huge This is the Title}

\author
{
}

\date{}

\author
{
    Zhongjie Shi \thanks{School of Mathematics, Georgia Institute of Technology, Atlanta, GA 30332, United States. The work was done while the author was affiliated with The University of Hong Kong; e-mail:  {\tt zshi332@gatech.edu}}
    ~~~and~~~
	Puyu Wang \thanks{Department of Computer Science, RPTU Kaiserslautern-Landau, Kaiserslautern, Germany; e-mail: {\tt wang.puyu@cs.rptu.de}} 
	~~~and~~~
	Chenyang Zhang \thanks{Department of Statistics and Actuarial Science, School of Computing and Data Science, The University of Hong Kong, Hong Kong; e-mail: {\tt chyzhang@connect.hku.hk}} 
	~~~and~~~
	Yuan Cao \thanks{Department of Statistics and Actuarial Science, School of Computing and Data Science, The University of Hong Kong, Hong Kong; e-mail:  {\tt yuancao@hku.hk}} 
}

\title{Towards Understanding Generalization in DP-GD: A Case Study in Training Two-Layer CNNs}

\begin{document}

\maketitle

\begin{abstract}
 Modern deep learning techniques focus on extracting intricate information from data to achieve accurate predictions. However, the training datasets may be crowdsourced and include sensitive information, such as personal contact details, financial data, and medical records. As a result, there is a growing emphasis on developing privacy-preserving training algorithms for neural networks that maintain good performance while preserving privacy. In this paper, we investigate the generalization and privacy performances of the differentially private gradient descent (DP-GD) algorithm, which is a private variant of the gradient descent (GD) by incorporating additional noise into the gradients during each iteration. Moreover, we identify a concrete learning task where DP-GD can achieve superior generalization performance compared to  GD in training two-layer Huberized ReLU convolutional neural networks (CNNs). Specifically, we demonstrate that, under mild conditions, a small signal-to-noise ratio can result in GD  producing training models with poor test accuracy, whereas DP-GD can yield training models with good test accuracy and privacy guarantees if the signal-to-noise ratio is not too small. This indicates that DP-GD has the potential to enhance model performance while ensuring privacy protection in certain learning tasks. Numerical simulations are further conducted to support our theoretical results.
\end{abstract}

%------------------------------------------------
\section{Introduction}
Modern deep learning (DL) algorithms are designed to extract fine-grained, high-dimensional patterns from data to achieve superior predictive performance. However, this ability to exploit intricate details can inadvertently expose sensitive information contained within the training data. In practical scenarios, datasets may include personally identifiable information such as health records, contact details, or financial transactions. Without appropriate safeguards, models trained using standard gradient descent (GD) methods are susceptible to various privacy attacks, such as membership inference or model inversion. As a result, it becomes critically important to study and understand the behavior of privacy-preserving variants of GD.

\textit{Differential Privacy} (DP) \cite{dwork2006calibrating,dwork2014algorithmic} is a widely adopted framework for designing privacy-preserving deep learning algorithms with strong theoretical guarantees. It ensures that the output of an algorithm is minimally influenced by any single data point in the input dataset, thereby safeguarding individual privacy.
Extensive research has been dedicated to developing efficient differentially private learning algorithms while maintaining strong statistical performance \cite{bassily2019private,bassily2020stability,bassily2014private,feldman2020private,wang2022differentially,chaudhuri2011differentially}.
In this paper, we focus on DP-GD, a commonly used private learning algorithm that introduces noise into the gradient updates to protect individual privacy. Understanding the interplay between generalization and privacy in such algorithms is crucial for advancing responsible and secure deep learning systems.

However, enforcing strong privacy guarantees often comes at the cost of model performance. Previous work has confirmed this tension between privacy and utility \citep{chaudhuri2011differentially,kifer2011no,bassily2021differentially,yang2021simple,carvalho2023towards}, highlighting that increasing the level of privacy protection, typically by injecting more noise, can degrade the predictive accuracy of machine learning models, since excessive noise may obscure crucial patterns in the data, leading to underfitting or poor generalization. On the other hand, Conversely, insufficient privacy safeguards risk exposing sensitive information, eroding public trust, and violating legal obligations. This trade-off is particularly pronounced in high-stakes applications, such as healthcare diagnostics or financial forecasting, where both accuracy and data confidentiality are paramount.
%As such, the design of trustworthy AI systems requires a nuanced understanding of how privacy-preserving mechanisms interact with the learning dynamics of neural networks. Developing principled approaches to quantify, evaluate, and mitigate the privacy-utility trade-off is essential for ensuring both responsible AI deployment and high-quality decision-making.

%The necessity for privacy-preserving machine learning is further underscored by legal and ethical constraints imposed by privacy regulations, such as the General Data Protection Regulation (GDPR) and the Health Insurance Portability and Accountability Act (HIPAA). These frameworks impose strict limitations on the usage, sharing, and storage of personal data. However, enforcing strong privacy guarantees often comes at the cost of model performance. Prior work has confirmed this tension between privacy and utility \citep{fung2010privacy,kifer2011no}, highlighting that increasing the level of privacy protection—typically by injecting more noise—can degrade the predictive accuracy of machine learning models. Therefore, balancing these competing objectives remains a fundamental challenge in the deployment of privacy-aware deep learning systems.

Therefore, it is important to build trustworthy DL systems that can satisfy both privacy and performance requirements. Whereas achieving a satisfying balance between privacy and generalization for DP-GD is a challenging problem for general cases, we may start from the following question:
\begin{center}
    \textit{Are there specific learning tasks for which DP-GD can simultaneously provide good privacy guarantees and maintain competitive generalization performance?}
\end{center}

This paper provides an affirmative answer to this question by identifying a concrete binary classification  task in which DP-GD not only protects privacy but also provably yields higher test accuracy than standard GD, showing that there is not always a trade-off. Our choices of learning tasks, models, theorem conditions, experiments are specifically aimed at showing the existence of scenarios where DP enhances accuracy. Specifically, we demonstrate that in the context of training two-layer Huberized ReLU CNNs under certain mild conditions, DP-GD can outperform GD in terms of test accuracy. This result reveals that privacy-preserving training does not always entail a loss in utility, and in some cases, the injected noise may even act as a form of regularization that benefits generalization. %Our findings offer valuable insight into how model selection and the choice of differentially private training algorithms can be aligned to achieve both privacy and performance objectives in specific machine learning scenarios. 
We summarize the main contributions of the paper in the following.

\begin{itemize}
    \item We provide a refined analysis showing that when the signal-to-noise ratio is relatively low,  under mild assumptions about the problem setup, model design, and hyperparameter configuration,  GD  can minimize the training loss to an arbitrarily small value. However, the corresponding test loss and test error remain bounded below by a constant, indicating poor generalization.
    
    \item We show that when the signal-to-noise ratio is not too small, DP-GD can, under similarly mild assumptions, achieve an arbitrarily small training loss. Furthermore, by applying the tool of early stopping, DP-GD is capable of achieving both strong generalization performance and meaningful privacy protection simultaneously.
    
    \item Comparing these two theoretical outcomes reveals that, with appropriate model architecture and careful tuning of hyperparameters, DP-GD can outperform the standard GD in terms of generalization in certain tasks. This highlights the potential of DP training algorithms not only to preserve privacy but also to enhance performance in specific scenarios, offering valuable guidance for model design and hyperparameter selection.
\end{itemize}

\noindent \textbf{Notations.} We use lower case letters, lower case bold face letters, and upper case bold face letters to denote scalars, vectors, and matrices respectively.  For a vector $\vb=(v_1,\cdots,v_d)^{\top}$, we denote by $\|\vb\|_2:=\big(\sum_{j=1}^{d}v_j^2\big)^{1/2}$ its $l_2$ norm. %For two sequence $\{a_k\}$ and $\{b_k\}$, we denote $a_k=O(b_k)$ if $|a_k|\leq C|b_k|$ for some absolute constant $C$, denote $a_k=\Omega(b_k)$ if $b_k=O(a_k)$, and denote $a_k=\Theta(b_k)$ if $a_k=O(b_k)$ and $a_k=\Omega(b_k)$. We also denote $a_k=o(b_k)$ if $\lim|a_k/b_k|=0$. 
We use $\tilde{O}(\cdot)$ and $\tilde{\Omega}(\cdot)$ to omit the logarithmic terms.

\section{Related Work}
\noindent \textbf{Implicit bias of neural networks.} A growing body of research has focused on the aspect of implicit bias, which refers to the intrinsic tendency of learning algorithms to favor solutions with certain underlying structures—often those considered to be "simple" or low-complexity \citep{neyshabur2014search, DBLP:journals/jmlr/SoudryHNGS18, pmlr-v99-ji19a, NEURIPS2022_ab3f6bbe, xie2024implicit, zhangimplicit}. Within the context of neural networks, several studies have explored how this phenomenon manifests for GD. For instance, \citet{lyu2019gradient} and \citet{ji2020directional} showed that when training 
$q$-homogeneous neural networks using GD, the direction of convergence aligns with a KKT point of the maximum $\ell_2$-margin optimization problem. Extending this line of inquiry, \citet{lyu2021gradient} established a stronger convergence result under the assumption of symmetric data.
%and \citet{wang2021implicit} further generalized the analysis to adaptive optimization algorithms.
Additional insights into the implicit bias of deep linear networks have been provided by \citet{ji2019gradient, ji2020directional}, who demonstrated that the weight matrices in each layer eventually converge to rank-one structures. %\citet{li2020towards} further connected the gradient dynamics of two-layer matrix factorization to a heuristic method for minimizing matrix rank. 
On data that is nearly orthogonal, \citet{frei2022implicit} proved that gradient flow in leaky ReLU networks leads to linear decision boundaries, and that the stable rank of the resulting model remains bounded by a constant. \citet{kou2024implicit} extended these results to standard gradient descent under similar data assumptions. 
%Meanwhile, \citet{timor2023implicit} explored rank minimization in nonlinear architectures and presented counterexamples that challenge certain generalizations. 
\citet{cao2022benign} investigated the GD training dynamics of two-layer polynomial ReLU CNNs, identifying specific signal-to-noise ratio thresholds that determine whether the model converges to the underlying signal or fits the noise. Building on this, \citet{kou2023benign} extended the analysis to standard ReLU CNNs. More recently, \citet{zhang2025gradient} demonstrated that training Huberized ReLU CNNs with gradient descent enables the model to robustly learns the intrinsic dimension of the data signals.
Lastly, \citet{vardi2023implicit} compiled a comprehensive survey summarizing key developments and open questions in the study of implicit bias in deep learning. 

\smallskip

\noindent \textbf{Differential privacy.} A large body of work has investigated the privacy and utility guarantees of differentially private gradient-based methods. The gradient perturbation mechanism, first introduced by \citet{song2013stochastic}, forms the foundation for widely studied algorithms such as DP-GD and DP-SGD. In particular, \citet{abadi2016deep} proposed the first algorithm for deep learning with differential privacy. \citet{bassily2014private,bassily2019private,bassily2021differentially,feldman2020private,wang2021differentially,asi2021private} demonstrated that both DP-GD and DP-SGD can achieve optimal utility bounds under different settings. 
Specifically, an excess risk bound of order $ {O}(\frac{\sqrt{d\log(1/\delta)}}{n\epsilon})$ for non-strongly convex problems and $ {O}(\frac{d\log(1/\delta)}{n^2\epsilon^2})$ for convex problems, where $n$ is the sample size, $d$ is input dimension, and $(\epsilon,\delta)$ are the privacy parameters. 
For non-convex problems, \citet{zhang2017efficient} established the utility bound $ {O}(\frac{\sqrt{d\log(1/\delta)}}{n\epsilon})$ for DP-SGD in terms of the squared gradient norm (i.e., first-order optimality) for nonconvex smooth objectives. \citet{wang2019differentially} provided a unified analysis of DP-GD and DP-SVRG for both convex and nonconvex settings, and specifically demonstrated the utility bound $ {O}(\frac{d\log(1/\delta)}{n^2\epsilon^2})$ of DP-GD in terms of the objective gap for objectives satisfying the Polyak-Łojasiewicz (PL) condition. 
Very recently, \citet{bu2023convergence} provided the first convergence analysis of DP-GD for deep learning, using insights on the training dynamics and the neural tangent kernel (NTK).
Yet, their results only showed that DP-GD with global clipping converges monotonically to zero loss without providing convergence rates.

\section{Problem Setting}

We consider a specific binary classification task with the use of two-layer CNNs. We present the data distribution in the following definition, where the input data comprises two types of components: \textit{label-dependent signals} and \textit{label-independent noises}. This simplified setting can already demonstrate the existence of scenarios where DP enhances accuracy. The simplifications are necessary for tractable theoretical analyses. We consider Huberized ReLU activation as it is smooth, which helps our analysis. Notably, similar setups have been considered in recent works \citep{li2019towards,allen2020towards,allen2022feature,cao2022benign, kou2023benign}, making our choice relatively standard.

\begin{definition} \label{datadistribution}
    Let $\bmu \in \RR^d$ be a fixed vector representing the signal contained in each data point. Each data point $(\bm x,y)$ with $\bm x=[\bm x^{(1)\top},\bm x^{(2)\top}]^\top \in \RR^{2d}$ and $y\in\{-1,1\}$ is generated from the following data distribution $\mathcal{D}$:  
    \begin{enumerate}
        \item  The label $y$ is generated as a Rademacher random variable. 
        \item  A noise vector $\bm \xi$ is generated from the Gaussian distribution $\cN(\bm 0, \sigma_p^2 \cdot (\bm I- \bm \mu \bm \mu^\top \cdot \lVert \bm \mu \rVert_2^{-2}))$. 
    %$\cN(\bm 0, \sigma_p^2 \Ib)$. \\
        \item One of $\bm x^{(1)}, \bm x^{(2)}$ is randomly selected and then assigned as $y \cdot \bm \mu$, which represents the signal; the other is then given by $\bxi$, which represents noises.
        \item The signal-to-noise ratio (SNR) is defined as $\text{SNR}=\norm \bmu_2/ \sigma_p \sqrt{d}$.
    \end{enumerate}
\end{definition}

	\noindent \textbf{Two-layer CNNs.} We consider two-layer convolutional neural networks (CNNs)
	\begin{equation*}
		f(\bW,\bx)= F_{+1}(\bW_{+1},\bx) - F_{-1}(\bW_{-1},\bx),
	\end{equation*}
	where $F_{+1}(\bW_{+1},\bx) $ and $ F_{-1}(\bW_{-1},\bx)$ are defined as:
	\begin{align*}
			F_{j}(\bW_{j},\bx) &= \frac{1}{m} \sum_{r=1}^m \left[ \sigma( \langle \bw_{j,r}, \bm x^{(1)} \rangle)  +\sigma\left( \langle \bw_{j,r}, \bm x^{(2)} \rangle\right)  \right] \\
			&= \frac{1}{m} \sum_{r=1}^m \left[ \sigma( \langle \bw_{j,r}, y \cdot \bmu \rangle\right)  +\sigma\left( \langle \bw_{j,r}, \bm \bxi \rangle)  \right]
	\end{align*}
	for $j = \{+1,-1\}$, and $m$ is the number of convolutional filters in $F_{+1}$ and $F_{-1}$. We consider the Huburized ReLU activation function $\sigma(\cdot)$ which is defined as 
    \[
    \sigma(z) = q^{-1} \kappa^{1 - q} z^q \cdot \indicator_{\{ z \in [0,\kappa] \}} + \left(z - \kappa + \frac{\kappa}{q} \right) \cdot \indicator_{\{ z > \kappa \}}
    \]
    where $\kappa$ is the threshold between polynomial and linear functions, and $q \geq 3$, and $\indicator$ is the indicator function. We use $w_{j,r} \in \mathbb{R}^d$ to denote the weight of the $r$-th filter, and $\bW_j$ is the collection of weights associated with $F_j$. We also use $\bW$ to denote the collection of all weights. \\
    
     \noindent \textbf{Training algorithm.} The above CNN model is trained by minimizing the empirical cross-entropy loss function
	\begin{equation*}
		L_S(\bW)= \frac{1}{n} \sum_{i=1}^n \ell\left[y_i \cdot f(\bW,\bx_i) \right] ,
	\end{equation*}
	where $\ell(t)= \log (1+e^{-t})$ is the logistic loss, and $S=\{(\bx_i,y_i)\}_{i=1}^n$ is the training data set. Moreover, the test loss is defined as
	\begin{equation*}
		L_{\mathcal{D}}(\bW) := \mathbb{E}_{(\bx,y)\sim \mathcal{D}} \ell\left[y \cdot f(\bW,\bx) \right] ,
	\end{equation*}	
	and the test error is defined as
	\begin{equation*}
		\mathcal{R}_{\mathcal{D}}(\bW) := \PP_{(\bx,y)\sim \mathcal{D}} \left(y \cdot f(\bW,\bx)<0 \right),
	\end{equation*}	
	
	We consider  DP-GD with Gaussian initialization, where each entry of $\bW_{+1}$ and $\bW_{-1}$ is sampled from a Gaussian distribution $\cN(0,\sigma_0^2)$. The  update rule at step $t$ is given by
    \begin{align}   \label{NGDupdate}
        \bw_{j,r}^{(t+1)} = \bw_{j,r}^{(t)}- \eta \left( \nabla_{\bw_{j,r}} L_S(\bW^{(t)}) + \bb_{j,r,t} \right)  .
        %\nonumber&= \bw_{j,r}^{(t)}- \frac{\eta}{nm} \sum_{i=1}^n \ell_{i}^{\prime(t)} \cdot \sigma^\prime \left( \langle \bw_{j,r}^{(t)}, y_{i} \bmu \rangle\right) \cdot j \bmu  \\
        %& - \frac{\eta}{nm} \sum_{i=1}^n \ell_{i}^{\prime(t)} \cdot \sigma^\prime \left( \langle \bw_{j,r}^{(t)}, \bxi_{i} \rangle\right) \cdot jy_{i} \bxi_{i} - \eta \bb_{j,r,t},
    \end{align}
 where the added Gaussian noises $\bb_{j,r,t} \sim \cN(\zero,\sigma_b^2 \bI)$, and we introduce a shorthand
 notation $\ell_{i}^{\prime(t)} := \ell^{\prime}[y_{i} \cdot f(\bW^{(t)}, \bx_{i})]$.

 This differs from the classical GD in that we add an additional Gaussian noise on the gradient in each iteration. Moreover, the update rule of GD is given by
    \begin{align} \label{GDupdate}
         \bw_{j,r}^{(t+1)} = \bw_{j,r}^{(t)}- \eta \left( \nabla_{\bw_{j,r}} L_S(\bW^{(t)})  \right) . 
    \end{align}
       % \nonumber &= \bw_{j,r}^{(t)}- \frac{\eta}{nm} \sum_{i=1}^n \ell_{i}^{\prime(t)} \cdot \sigma^\prime \left( \langle \bw_{j,r}^{(t)}, y_{i} \bmu \rangle\right) \cdot j \bmu \\
       % & - \frac{\eta}{nm} \sum_{i=1}^n \ell_{i}^{\prime(t)} \cdot \sigma^\prime \left( \langle \bw_{j,r}^{(t)}, \bxi_{i} \rangle\right) \cdot jy_{i} \bxi_{i}.

%\[\sup_{(x,y),(x',y')} \|\nabla_w \ell(y f(W,x)) - \nabla_w \ell(y' f(W,x'))  \|_2\]

\section{Main Results}
In this section, we present our main theoretical findings. In particular, we construct a specific binary classification task using two-layer CNNs, where the SNR satisfies the condition $\tilde \Omega(n^{\frac{1}{q}}) \!\leq\! \text{SNR}^{-1} \!\leq\!  \min\!\big\{ \frac{\sqrt{d}}{Cm^2}, \frac{\sqrt{n}}{C}\big\}$. Under this setting, we show that the training loss of both GD and DP-GD can converge to an arbitrarily small value.  Whereas DP-GD can outperform GD in terms of the generalization performance. We note that our SNR conditions are not intended to guide practice, but are outcomes of theoretical analyses. They define the specific regime where privacy can improve accuracy, and reasonably exclude cases with very low/high SNR, where the task is too difficult/easy and both DP-GD and GD perform similarly poorly/well. This allows us to focus on settings where DP-GD and GD can be distinguished. While the SNR conditions do not directly guide practice, our finding that privacy can sometimes enhance accuracy has practical implications by deepening our understanding of the relationship between privacy and accuracy.

\subsection{Noise Memorization of GD} 
The theoretical analysis of GD is based on the following specific conditions, where we identify an SNR condition $\text{SNR}^{-1} \geq \tilde \Omega(n^{\frac{1}{q}})$, and mild conditions on the choice of hyperparameters in the problem setting and training algorithm. 
%The main results suggest that GD results in noise memorization during the training process, and therefore achieve poor generalization performance.
We consider the learning period $0 \leq t \leq T^*$, where $T^*=\tilde{O}\left( \frac{\kappa^{q-1} mn}{\eta \sigma_0^{q-2} (\sigma_p \sqrt{d})^q} + \frac{m^{3}n}{\eta \epsilon \norm{\bmu}_2^2} \right)$. 

%This comprehensive approach to the theoretical foundations ensures that our findings possess a broad applicability, capable of addressing a diverse array of practical scenarios. By meticulously accounting for the interplay between these critical parameters, we have laid the groundwork for a robust and versatile understanding of the neural network optimization process.
\begin{condition} \label{condition}
    %Denote $\tilde T_1 = \tilde C_1 \frac{\kappa^2}{\eta^2 \sigma_b^2 \| \bmu \|_2^2}$. 
     Suppose there exists a sufficiently large constant $C$, such that the following hold: 
    \begin{enumerate}
        \item The threshold $\kappa$ of the activation function is sufficiently small: 
        $\kappa= O(1)$.
         \item The SNR is sufficiently small: $\text{SNR}^{-1} \geq \tilde \Omega(n^{\frac{1}{q}})$.
        \item  The dimension $d$ is sufficiently large: \\
        $d \geq C m^{\frac{2q}{q-2}} n^{\frac{2q-2}{q-2}}\kappa^{-\frac{2q-2}{q-2}} (\log(\frac{mn^2}{\delta}))^2 (\log(T^*))^2$. 
        \item The training sample size $n$ and the convolutional kernel size $m$ of CNNs is sufficiently large: \\
        $n\geq C\log(\frac{m}{\delta})$, $m\geq C\log \left(\frac{n \tilde T^*}{\delta}\right)$.
        %\item The standard deviation of the noise vector is sufficiently small: \\
        %$\sigma_p \leq C\left(\max\left\{ \frac{mn}{\tilde \epsilon \norm{\bmu}_2}, \frac{\sqrt{mn} d^{\frac{1}{4}}}{\sqrt{\tilde \epsilon} \norm{\bmu}_2} \right\} \right)^{-1}$.
        \item The standard deviation of Gaussian initialization $\sigma_0$ satisfies: 
         $\frac{C n}{\sigma_p d} \sqrt{\log\left(\frac{n^2}{\delta}\right)} \log(T^*) \leq \sigma_0 \leq (C \max\big\{\norm {\bmu}_2  m^{\frac{2}{q-2}} n^{\frac{1}{q-2}} \sqrt{\log(\frac{mn}{\delta})}, \sigma_p \sqrt{d} \big\})^{-1} \kappa^{\frac{q-1}{q-2}} $.
        \item The learning rate $\eta$ is sufficiently small: \\
        $\eta \leq \left(C \max \left\{ \norm \bmu_2^{2}, \sigma_p^2 d \right\}\right)^{-1}$.
        %, \frac{n \norm \bmu_2^{2}}{m}
        %\item The standard deviation of the Gaussian noise $\sigma_b$ satisfies: 
        %$  \frac{C \epsilon \kappa^2 \max\{ n\norm \bmu_2^{2}, \sigma_p^2d\}}{{\eta nm \min\{ \norm \bmu_2^{2}, \sigma_p^2d\}}} \leq \sigma_b^2 \leq \frac{\epsilon n \norm {\bmu}_2^4}{{C \eta m^4 \max^2\{n\norm {\bmu}_2^2 , \sigma_p^2 d \} \log^4 \left( \frac{mn \tilde T^*}{\delta} \right)}}$.
    \end{enumerate}
\end{condition}

The conditions on $d,n,m$ are to ensure that the learning
problem %%cyz
is in a sufficiently over-parameterized setting, and similar conditions have been made in \citet{chatterji2021finite,cao2022benign,frei2022benign}. The condition on the SNR and the lower bound condition on $\sigma_0$ ensure that the memorization of the noises dominates the learning of the signal in GD. The upper bound on $\sigma_0$  ensures that within $T^*$ iterations, the learning of the signal is always small and around the initialization order, even if the training loss converges. The condition imposed on $\eta$ serves as a sufficient requirement to guarantee that GD can effectively minimize the training loss. %However, we note that this bound is not tight and may be further relaxed. We give our main result on the noise memorization of GD in the following theorem.

\begin{theorem}\label{theorem: GD main}
	Under Condition \ref{condition}, for any $\epsilon >0$, denote $T_1= \tilde{\Theta}\left(\frac{\kappa^{q-1} mn}{\eta \sigma_0^{q-2} (\sigma_p \sqrt{d})^q}\right)$, and $T_2= T_1+ \frac{36 nm^2}{\eta \sigma_p^2d}$. Then within $T^*=T_1+ \tilde{O}\left( \frac{m^{3}n}{\eta \epsilon \norm{\bmu}_2^2} \right)$ iterations, with probability at least $1-6\delta$,
	we have
	\begin{enumerate}
		\item The training loss converges: there exists a time $t \leq T^*$ such that $L_{S}(\bW^{(t)}) \leq \epsilon$. 
		\item The test loss is always large: for any $0 \leq t\leq T^*$ we have that $L_{\cD}(\bW^{(t)}) \geq 0.1$.
		\item The test error is always large: suppose that $\sigma_0 \leq \frac{C_3}{m  \norm{\bmu}_2 \sqrt{d}}$  for some small constant $C_3$. For any $T_2 \leq t \leq T^*$, we have that $\mathcal{R}_{\mathcal{D}}(\bW^{(t)})\geq 0.11$. 
	\end{enumerate}
\end{theorem}

Theorem \ref{theorem: GD main} shows that when the SNR is small, the training loss of GD can converge to any accuracy $\epsilon$. However, the test loss of the trained CNN has at least a constant order. Moreover, we provide a sufficient condition on $\sigma_0$ such that the test error of the trained CNN has at least a constant order as well, when the training iteration is not too small. We note that this upper bound condition on $\sigma_0$ is only a sufficient condition. We need this condition due to the technical difficulties, and this condition has the potential to be relaxed. %In the following, we briefly introduce the overview of the proof technique. The main idea is that the memorization of the noises can achieve $\kappa$ and further grow to a constant level, whereas the learning of the signal is always small and around the initialization values. 

\noindent \textbf{Overview of proof technique}
At the core of our analyses is a \textit{signal-noise decomposition} of the filters in the CNN trained by the optimization algorithm. According to the GD update rule \eqref{GDupdate}, it is clear that the gradient descent iterate $\bw_{j,r}^{(t)}$ is a linear combination of its random initialization $\bw_{j,r}^{(0)}$, the signal vector $\bmu$ and the noise vectors in the training data $\bxi_i$, $i\in [n]$. Motivated by this observation, we introduce the following definition. 
\begin{definition}\label{GD:w_decomposition}
	Let $\bw_{j,r}^{(t)}$ for $j\in \{\pm 1\}$, $r \in [m]$ be the CNN convolution filters in the $t$-th iteration of GD \eqref{GDupdate}. Then there exist unique coefficients $\gamma_{j,r}^{(t)} \geq 0$ and $\rho_{j,r,i}^{(t)}$ such that 
	\begin{align*}
		\bw_{j,r}^{(t)} = \bw_{j,r}^{(0)} + \gamma_{j,r}^{(t)} \cdot \| \bmu \|_2^{-2} \cdot j\bmu + \sum_{ i = 1}^n \rho_{j,r,i}^{(t) }\cdot \| \bxi_i \|_2^{-2} \cdot \bxi_{i}.
	\end{align*}
	We further denote $\orho^{(t)} := \rho_{j,r,i}^{(t)}\indicator(\rho_{j,r,i}^{(t)} \geq 0)$, $\urho^{(t)} := \rho_{j,r,i}^{(t)}\indicator(\rho_{j,r,i}^{(t)} \leq 0)$. Then we have
	\begin{align}\label{eqGD:w_decomposition}
		\nonumber  \bw_{j,r}^{(t)} &= \bw_{j,r}^{(0)} + \gamma_{j,r}^{(t)} \cdot \| \bmu \|_2^{-2} \cdot j\bmu + \sum_{ i = 1}^n \orho^{(t) }\cdot \| \bxi_i \|_2^{-2} \cdot \bxi_{i} \\
		&+ \sum_{ i = 1}^n \urho^{(t) }\cdot \| \bxi_i \|_2^{-2} \cdot \bxi_{i} .
	\end{align}
\end{definition}

The normalization factors $\|\bmu\|_{2}^{-2}, \|\bxi_{i}\|_{2}^{-2}$ are to ensure that $\gamma_{j,r}^{(t)} \approx \la \bw_{j,r}^{(t)}, j \bmu \ra$ tracks signal learning and $\rho_{j,r,i}^{(t)} \approx \la \bw_{j,r}^{(t)}, \bxi_{i} \ra$ tracks noise memorization. Using Definition \ref{GD:w_decomposition}, we can reduce the study of the CNN learning process to a careful assessment of the coefficients $\gamma_{j,r}^{(t)}$, $\orho^{(t)}$, $\urho^{(t)}$ throughout training, where $\gamma_{j,r}^{(t)}$ and $\orho^{(t)}$ are increasing monotonically and $\underline{\rho}_{j,r,i}^{(t)}$ is decreasing monotonically. The main idea is that the memorization of noise can achieve $\kappa$ and grow further to a constant level, whereas the learning of the signal is always small and around the initialization values.  %This technique helps us to avoid the complex analysis of the growth of signal learning $\la \bw_{j,r}^{(t)}, j \bmu \ra$ and noise memorization $\la \bw_{j,r}^{(t)}, \bxi_{i} \ra$ directly, which are not monotonically changing and are difficult to analyze. Nevertheless, we can also analyze the growth of $\la \bw_{j,r}^{(t)}, j \bmu \ra$ and $\la \bw_{j,r}^{(t)}, \bxi_{i} \ra$ by utilizing the properties of $\gamma_{j,r}^{(t)}$, $\overline{\rho}_{j,r,i}^{(t)}$, and $\underline{\rho}_{j,r,i}^{(t)}$. \\
%However, \citet{cao2022benign} only characterized the behavior of the leading  neurons by studying $\max_{r}\gamma_{j,r}^{(t)}$, $\max_{r}\zeta_{j,r,i}^{(t)}$. To guarantee that the leading neuron can dominate other neurons after training, they require neurons with different initial weights to have different update speeds, which is guaranteed thanks to the activation function ReLU$^q$ with $q > 2$. But the ReLU function is piece-wise linear, and every activated neuron has the same learning speed $\sigma'(x) = 1$. Therefore dealing with ReLU requires new techniques. 

\noindent \textbf{Stage 1.} We note that the Huberized ReLU activation function is piece-wise continuous, where the part between $[0,\kappa]$ is a polynomial of order $q$, and the part between $[\kappa,\infty)$ is linear, with the threshold $\kappa$. For the first stage, we show that at time $T_1= \tilde \Theta \left( \frac{\kappa^{q-1} mn}{\eta \sigma_0^{q-2} (\sigma_p \sqrt{d})^q} \right)$, there exists a neuron of  noise memorization $\langle\bw_{y_i,r}^{(T_1)}, \bxi_i \rangle $ that can
hit the threshold $\kappa$. However, all neurons of signal learning $\langle\bw_{j,r}^{(t)}, j \mu \rangle$ are around the initialization order.
\begin{lemma} \label{result:GDFirstStage}
	Under Condition \ref{condition}, we can find a time $T_1= \tilde \Theta \left( \frac{\kappa^{q-1} mn}{\eta \sigma_0^{q-2} (\sigma_p \sqrt{d})^q} \right)$, such that
	\begin{itemize}
		\item $\max_{r} \langle\bw_{y_i,r}^{(T_1)}, \bxi_i \rangle \geq \kappa$, $\max_{j,r} \orho^{(T_1)} \geq \kappa$, for all $i \in [n]$.
		\item $\max_{j,r} \langle\bw_{j,r}^{(t)}, j \mu \rangle = \tilde O (\sigma_0 \|\mu\|_2)$, $\max_{j,r} \gam^{(T_1)} = \tilde O (\sigma_0 \|\mu\|_2)$,  for all $0 \leq t \leq T_1$.
		\item $\max_{r,i} \left|\langle\bw_{-y_i,r}^{(t)}, \bxi_i \rangle\right| = \tilde O (\sigma_0 \sigma_p \sqrt{d})$, $\max_{j,r,i} |\urho^{(T_1)}|  = \tilde O (\sigma_0 \sigma_p \sqrt{d})$, for all $0 \leq t \leq T_1$. \\
	\end{itemize}
\end{lemma}

\noindent \textbf{Stage 2.} For the second stage, we demonstrate that the training loss will converge to the accuracy $\epsilon$ at time $T^*=\tilde{O}\left( \frac{\kappa^{q-1} mn}{\eta \sigma_0^{q-2} (\sigma_p \sqrt{d})^q} + \frac{m^{3}n}{\eta \epsilon \norm{\bmu}_2^2} \right)$. Moreover, since $\sum_{s=T_{1}}^{\tilde{T} - 1} \sum_{i=1}^n |\ell_i'^{(t)}| \leq \sum_{s=T_{1}}^{\tilde{T} - 1}L_{S}(\bW^{(s)})$, the convergence of the training error in Stage 2 indicates that the growth of the signal learning $\langle\bw_{j,r}^{(t)}, j \mu \rangle$ is still small and around the initialization order.

\begin{lemma}\label{result:training convergence}
	Under Condition \ref{condition}, let $T^* = T_{1} + \Big\lfloor \frac{\|\bW^{(T_{1})} - \bW^{*}\|_{F}^{2}}{2\eta \epsilon}
	\Big\rfloor  = T_{1} + \tilde{O}\left(\frac{m^{3}n}{\eta \epsilon \norm{\bmu}_2^2} \right)$. Then we have
	\begin{itemize}
		\item $\max_{j,r} \la\bw_{j,r}^{(t)}, j \bmu\ra = \tilde{O} (\sigma_0 \norm{\bmu}_2)$, $\max_{j,r}\gamma_{j,r}^{(t)} = \tilde{O} (\sigma_0 \norm{\bmu}_2)$, for all $T_{1} \leq t \leq T^*$.
		\item $\max_{j,r,i}|\urho^{(t)}| = \tilde{O}(\sigma_{0}\sigma_{p}\sqrt{d})$, for all $T_{1} \leq t \leq T$.
	\end{itemize}
	Besides, for all $T_{1} \leq t \leq T^*$, we have
	\begin{align*}
		&\frac{1}{t - T_{1} + 1}\sum_{s=T_{1}}^{t}L_{S}(\bW^{(s)}) \\
		\leq& \frac{\|\bW^{(T_{1})} - \bW^{*}\|_{F}^{2}}{(2q-1) \eta(t - T_{1} + 1)} + \frac{\epsilon}{(2q-1)}.
	\end{align*}
	Therefore, we can find an iterate with training loss smaller than $\epsilon$ within $T$ iterations.
\end{lemma}

\noindent \textbf{Generalization analysis.} To prove test loss is large, we only need to show that for a new example $(\bx,y)$, the noise term $|\la\bw_{j,r}^{(t)}, \bxi\ra|$ and the signal term $|\la \bw_{y,r}^{(t)}, y\bmu \ra|$ are both small. Notice that $\la\bw_{j,r}^{(t)}, \bxi\ra \sim \cN(0, \sigma_{p}^{2}\|\bw_{j,r}^{(t)}\|_{2}^{2})$. Therefore, we can show that with high probability
\begin{align*} 
	|\la\bw_{j,r}^{(t)}, \bxi\ra| \leq \tilde{O} \left( \sigma_{0}\sigma_{p}\sqrt{d} + \frac{mn}{ \sqrt{d}} \right) = O(\kappa).
\end{align*}
This, together with the first property in Lemma \ref{result:training convergence} demonstrates that $y_i f(\bW^{(t)},\bx)\leq 1$, thus the test loss is large at a constant level. \\

%\noindent \textbf{Test Error.} 
To prove the lower bound of test error, we first show that at time $T_2= T_1+ \frac{36 nm^2}{\eta \sigma_p^2d}$, the neurons have memorized the noises, i.e., $\sum_{r=1}^m \overline{\rho}_{y_i,r,i}^{(t)} \geq m$.

\begin{lemma} \label{result:GDnoisehit2}
	Under Condition \ref{condition}, let $T_2= T_1+ \frac{36 nm^2}{\eta \sigma_p^2d}$. For the time period $T_2 \leq t \leq T^*$ and all $i \in[n]$, we have
	\begin{align*}
		\sum_{r=1}^m \overline{\rho}_{y_i,r,i}^{(t)} \geq m.
	\end{align*}
\end{lemma}

Denote $g(\bxi) = \sum_{r} \sigma(\la \bw_{1, r}^{(t)}, \bxi \ra) - \sum_{r} \sigma(\la \bw_{-1, r}^{(t)}, \bxi \ra) $, and the set 
$\Omega := \bigg\{\bxi \bigg|  g(\bxi) >  \tilde O(m \sigma_0 \norm{\bmu}_2)  \bigg\}. $
Since the signal learning of each neuron is at most $\tilde{O}(\sigma_0 \norm{\bmu}_2)$, we can  give a lower bound of the test error by
\begin{equation*}
	\begin{aligned}
		\mathcal{R}_{\mathcal{D}}(\bW^{(t)}) &= \mathbb{P} \big(y f(\bW^{(t)},\bx) < 0 \big)\\
		& \geq 0.5 \mathbb{P} \bigg(  \sum_{r=1}^m \sigma(\la \bw_{1,r}^{(t)}, \bxi \ra) - \sum_{r=1}^m \sigma(\la \bw_{-1,r}^{(t)}, \bxi \ra) \\
		&> \tilde O(m \sigma_0 \norm{\bmu}_2)  \bigg) \geq 0.5 \mathbb{P}(\Omega).
	\end{aligned}
\end{equation*}

Finally, by utilizing Lemma \ref{result:GDnoisehit2},  the symmetry property of $\bxi$ and the properties of total variance distance, we can derive that $\PP(\Omega) \geq 0.23$, therefore get the desired results on the test error lower bound.

\subsection{Signal Learning of DP-GD}
The theoretical analysis of DP-GD hinges on the following specific conditions, where we identify an SNR condition $\text{SNR}^{-1} \leq \min\left\{ \frac{\sqrt{d}}{Cm^2}, \frac{\sqrt{n}}{C}\right\}$, and mild conditions on the choice of hyperparameters in the problem setting and training algorithm. 
%The main results suggest that GD results in noise memorization during the training process, and therefore achieve poor generalization performance.
We consider the learning period $0 \leq t \leq \tilde T^*$, where $\tilde T^*$ is the maximum number of iterations. 

\begin{condition} \label{condition2}
	%Denote $\tilde T_1 = \tilde C_1 \frac{\kappa^2}{\eta^2 \sigma_b^2 \| \bmu \|_2^2}$. 
	Suppose there exists a sufficiently large constant $C$, such that the following hold: 
	
	\begin{enumerate}
		%\item The threshold $\kappa$ of the activation function satisfies: \\
		%$\kappa^2 \leq \frac{n^2 \norm {\bmu}_2^4 \min\{ \norm {\bmu}_2^2 , \sigma_p^2 d \} }{C m^3  \max^3\{n\norm {\bmu}_2^2 , \sigma_p^2 d \}}$.
		\item The threshold $\kappa$ of the activation function is sufficiently small: 
		$\kappa^2 \leq \frac{ \min\{ \norm \bmu_2^{2}, \sigma_p^2d\}}{C m^2  \max\{ \norm \bmu_2^{2}, \sigma_p^2d\}}$.
		\item The SNR satisfies: $\text{SNR}^{-1} \leq \min\big\{ \frac{\sqrt{d}}{Cm^2}, \frac{\sqrt{n}}{C}\big\}$.
		\item  The dimension $d$ is sufficiently large: \\
		$d\geq C \max\left\{m^4 n^2, \frac{m^2 n^2}{\kappa^2} \log(\frac{n^2}{\delta}) (\log(T^*))^2 \right\}$. 
		\item The training sample size $n$ and the convolutional kernel size $m$ of CNNs is sufficiently large: \\
		$n\geq C\log(\frac{m}{\delta})$, $m\geq C\log \big(\frac{n \tilde T^*}{\delta}\big)$.
		%\item The standard deviation of the noise vector is sufficiently small: \\
		%$\sigma_p \leq C\left(\max\left\{ \frac{mn}{\tilde \epsilon \norm{\bmu}_2}, \frac{\sqrt{mn} d^{\frac{1}{4}}}{\sqrt{\tilde \epsilon} \norm{\bmu}_2} \right\} \right)^{-1}$.
		\item The standard deviation of the Gaussian initialization $\sigma_0$ satisfies: \\
		$\sigma_0 \leq  \min \Big\{\frac{1}{C m \sigma_p \sqrt{d}}, \frac{\kappa}{C \max\{\norm{\bmu}_2,\sigma_p \sqrt{d}\} \sqrt{\log\big(\frac{mn}{\delta} \big)}}  \Big\}$.
		\item The learning rate $\eta$ satisfies: \\
		$\frac{Cm^3 \kappa^2  \max\{ \norm \bmu_2^{2}, \sigma_p^2d\}}{ \norm \bmu_2^{2} \min\{\norm{\bmu}_2^2,\sigma_p^2d\}} \leq \eta \leq  \frac{m}{C  \norm \bmu_2^{2}}. $
		%, \frac{n \norm \bmu_2^{2}}{m}
		\item The standard deviation of the Gaussian noise $\sigma_b$ satisfies: \\
		$  \sigma_b \leq \Big(C \eta \max\big\{\norm{\bmu}_2,\sigma_p \sqrt{d}\big\} \sqrt{\tilde T^*}  \log ^2\big(\frac{mn\tilde T^*}{\delta}\big)\Big)^{-1}$.
	\end{enumerate}
\end{condition}

The condition on the SNR is to ensure that the privacy guarantee and the test loss of DP-GD are good. The upper bounds on $\sigma_0$ and $\eta$ are to ensure that the test loss is good. The lower bound on $\eta$ is to ensure that the value of signal learning can achieve the threshold $\kappa$ at a constant order time. The condition on $\sigma_b$ is to ensure that the influence of the added Gaussian noise is smaller than the learning of the signal.

The following theorem demonstrates that the training loss of DP-GD can achieve an arbitrarily small accuracy $\epsilon$ within $\tilde T^*= \Theta \left(\frac{\kappa^2}{\eta^2 \sigma_b^2 \min\{\norm{\bmu}_2^2,\sigma_p^2d\}}+ \frac{nm^2}{\eta \epsilon \max\{\norm{\bmu}_2^2,\sigma_p^2d\}} \right)$ iterations.

\begin{theorem} \label{result:NGDtraining}
	Under Condition \ref{condition2}, for any $\epsilon >0$, denote $\tilde T_1= \Theta \left(\frac{\kappa^2}{\eta^2 \sigma_b^2 \min\{\norm{\bmu}_2^2,\sigma_p^2d\}} \right)$, if we choose $\tilde T^*= \tilde T_1 + \Theta \left(\frac{nm^2}{\eta \epsilon \max\{\norm{\bmu}_2^2,\sigma_p^2d\}} \right)$,  with probability at least $1-6\delta$, we have $L_S(\bW^{(\tilde T^*)}) \leq \epsilon$. 
\end{theorem}

In the following theorem, we indicate that with a proper choice of $\sigma_b$ to control the level of the injected noise in DP-GD,  and with the tool of early stopping at around $\tilde T_2= \Theta ( \frac{m}{\eta \norm \bmu_2^{2}})$ iterations, we can achieve a good test error and a good privacy guarantee simultaneously for DP-GD.

\begin{theorem} \label{result:NGDtesterror}
	Under Condition \ref{condition2}, by choosing $\sigma_b= \Theta\left( \sqrt{\frac{\norm \bmu_2^{2}}{\eta m^3  \max\{\norm \bmu_2^{2}, \sigma_p^2d\}}}\right)$, and $\tilde{T}_2=  \Theta \left( \frac{m}{\eta \norm \bmu_2^{2}}\right)$.  With probability at least $1-6\delta$, we have $	\mathcal{R}_{\mathcal{D}}(\bW^{(\tilde{T}_2)} )  \leq 0.01$. Moreover, the DP-GD with $\tilde{T}_2$ iterations satisfies $\big( \frac{C_4 m^3  \max^2\{ \norm \bmu_2^{2}, \sigma_p^2d\}}{n^2 \norm \bmu_2^{4}} \log \frac{2}{\delta}, \delta \big)$-DP for some positive constant $C_4$.
\end{theorem}

This theorem highlights the generalization advantage of DP-GD over standard GD. The key insight is that the injected Gaussian noise in DP-GD facilitates the signal learning component in surpassing the threshold $\kappa$ of the Huberized ReLU when $t \geq\tilde T_1$. Once this threshold is crossed, the derivative $\sigma^\prime(z)$ becomes 1, allowing the signal to grow more effectively. In contrast, the signal learning component remains at initialization order, and $\sigma^\prime(z)$ remains very small, limiting the signal learning process in standard GD.

Related to our work, \cite{ding2025understanding} and \cite{zhang2026understanding} established the theoretical framework for analyzing Differentially private Stochastic Gradient Descent (DP-SGD) from a feature learning perspective. Specifically, \cite{ding2025understanding} demonstrated that effective private signal learning requires a higher SNR compared to non-private training and that when noise memorization occurs in the non-private setting, it often persists in private learning, leading to poor generalization. \cite{zhang2026understanding} revealed that DP-SGD can help generalization performance on long-tailed data through gradient clipping and noise injection, which suppress the model’s ability to memorize implicit class-specific features. However, distinct from these works, we identify a specific theoretical regime with moderate SNR where DP-GD effectively enhances generalization compared to standard GD, rather than merely inducing a trade-off between generalization and privacy.

\noindent \textbf{Overview of proof technique}
We utilize the same signal-noise decomposition method as in Definition \ref{GD:w_decomposition} as the core method for our analysis, the different point is that we need the additional added Gaussian noise terms $\eta \sum_{k=0}^{t-1} \bb_{j,r,k}$ in DP-GD.

\begin{definition} \label{def3}
	Let $\wjr^{(t)}$ for $j \in \{+1,-1\}$, $r \in [m]$ be the convolution filters of the CNN at the $t$-th iteration of noisy SGD. Then there exist unique coefficients $\gam^{(t)}$ and $\rho_{j,r,i}^{(t)}$ such that
	\begin{align} \label{decompositionNGD}
		\nonumber\wjr^{(t)}&= \wjr^{(0)} + j \cdot \gam^{(t)} \cdot \norm \bmu_2^{-2} \cdot \bmu \\
		& + \sum_{i=1}^n \rho_{j,r,i}^{(t)} \cdot \norm {\bxi_i}_2^{-2} \cdot \bxi_i  - \eta \sum_{k=0}^{t-1} \bb_{j,r,k}.
	\end{align}
	By further denoting $\orho^{(t)}:= \rho_{j,r,i}^{(t)} \mathds{1}(\rho_{j,r,i}^{(t)} \geq 0)$, and $\urho^{(t)}:= \rho_{j,r,i}^{(t)} \mathds{1}(\rho_{j,r,i}^{(t)} \leq 0)$, we have
	\begin{align} \label{sndecompositionNGD}
		\nonumber  \wjr^{(t)}&= \wjr^{(0)} + j \gam^{(t)} \cdot \norm \bmu_2^{-2} \cdot \bmu + \sum_{i=1}^n \orho^{(t)} \cdot \norm {\bxi_i}_2^{-2} \cdot \bxi_i \\
		& + \sum_{i=1}^n \urho^{(t)} \cdot \norm {\bxi_i}_2^{-2} \cdot \bxi_i - \eta \sum_{k=0}^{t-1} \bb_{j,r,k}.
	\end{align}
\end{definition}

\begin{figure*}[t]
    \centering
    \includegraphics[width=0.93\textwidth]{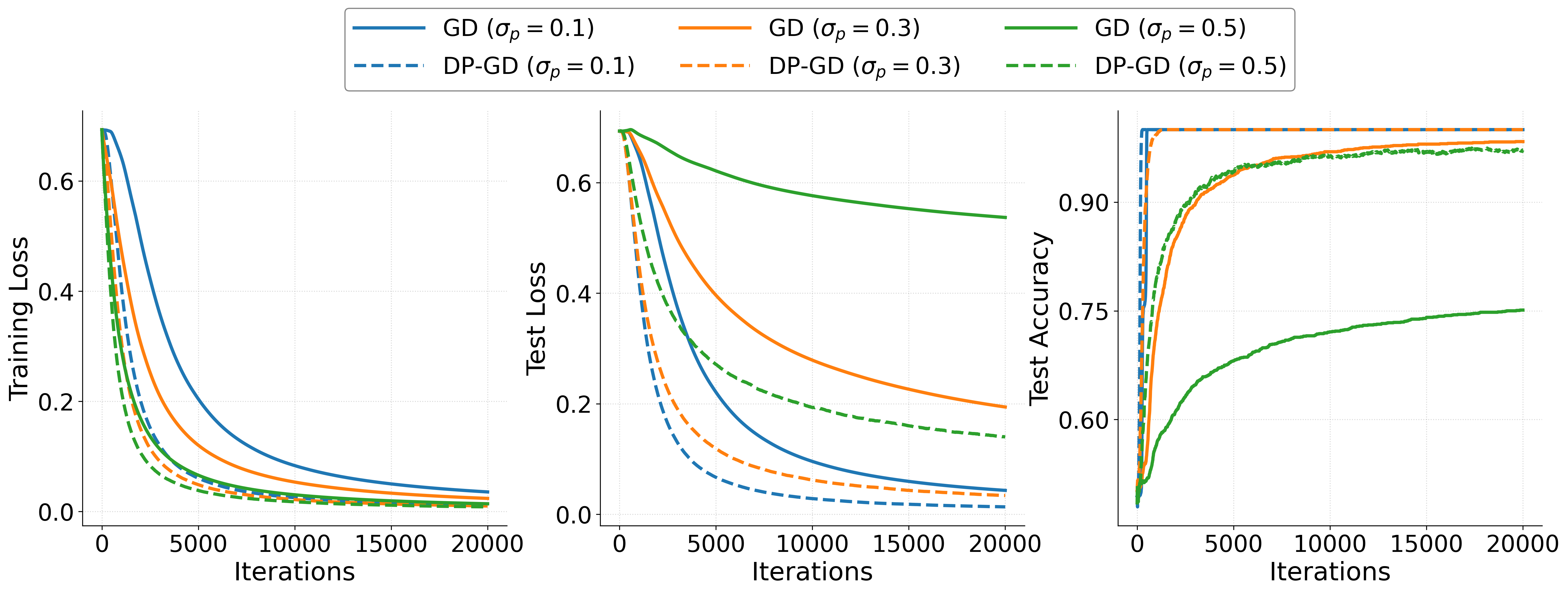}
    
    \vspace{-0.05cm} 

    \noindent
    \begin{minipage}[t]{0.31\textwidth}
        \centering  (a) Training Loss
    \end{minipage}%
    \begin{minipage}[t]{0.31\textwidth}
        \centering  (b) Test Loss
    \end{minipage}%
    \begin{minipage}[t]{0.31\textwidth}
        \centering  (c) Test Accuracy
    \end{minipage}

    \makeatletter
    \refstepcounter{figure} 
    \let\saved@currentlabel\@currentlabel 
    
    \def\@currentlabel{\saved@currentlabel(a)}\label{subfig:training_loss}
    \def\@currentlabel{\saved@currentlabel(b)}\label{subfig:test_loss}
    \def\@currentlabel{\saved@currentlabel(c)}\label{subfig:test_accuracy}
    
    \def\@currentlabel{\saved@currentlabel}\label{fig:synthe_1}
    
    \addtocounter{figure}{-1} 
    \makeatother

    \caption{Training loss, test loss, and test accuracy of two-layer CNNs trained with GD and DP-GD. Results are shown for three noise levels ($\sigma_p \in \{0.1, 0.3, 0.5\}$) with fixed signal strength ($\|\bmu\|_2=1$).}
\end{figure*}

% \begin{figure*}[t]
%     \centering
%     \includegraphics[width=\textwidth，height=0.25\textheight,]{figures/CNN_all_metrics_top_legend.png}
    
%     \vspace{-0.2cm} 

%     \begin{tabular}{p{0.33\textwidth}<{\centering} p{0.33\textwidth}<{\centering} p{0.33\textwidth}<{\centering}}
%         \small (a) Training Loss & \small (b) Test Loss & \small (c) Test Accuracy
%     \end{tabular}

%     \makeatletter
%     \refstepcounter{figure} 
%     \let\saved@currentlabel\@currentlabel 
    
%     \def\@currentlabel{\saved@currentlabel(a)}\label{subfig:training_loss}
%     \def\@currentlabel{\saved@currentlabel(b)}\label{subfig:test_loss}
%     \def\@currentlabel{\saved@currentlabel(c)}\label{subfig:test_accuracy}
    
%     \def\@currentlabel{\saved@currentlabel}\label{fig:synthe_1}
    
%     \addtocounter{figure}{-1} 
%     \makeatother

%     \caption{Training loss, test loss, and test accuracy of two-layer CNNs trained with GD and DP-GD. Results are shown for three noise levels ($\sigma_p =0.1, 0.3, 0.5$) with fixed signal strength ($\|\bmu\|_2=1$).}
% \end{figure*}

\noindent \textbf{Training loss.} We briefly introduce the proof of the convergence of training loss for DP-GD. The key idea is to analyze the growth of the signal and noises together during the training process, which is denoted as
      $\lambda_{i}^{(t)}= \frac{1}{m} \sum_{r=1}^m \left(\gamma_{y_i,r}^{(t)}+\overline{\rho}_{y_i,r,i}^{(t)} \right)$. It is different from the analysis of GD since we cannot identify whether the memorization of noise or the learning of signal dominates during the training process, due to the perturbations caused by the added Gaussian noise in the gradient. However, the added Gaussian noises $\eta \sum_{k=0}^{t-1} \bb_{j,r,k}$ in DP-GD can help at least one neuron in  $\langle\bw_{j,r}^{(t)}, j\bmu\rangle$ or $\langle\bw_{y_i,r}^{(t)}, \bxi_i \rangle$ to achieve the threshold $\kappa$ of the Huberized ReLU activation function, for each $j\in \{\pm 1\}$ and 
      $i\in [n]$, when the iteration $t\geq \tilde{T}_1$ is large enough. This property makes it different from the analysis of GD, since it ensures that at least one neuron in signal learning to hit $\kappa$, and therefore help the signal learning term to grow outside the scope of the initialization order.
      
      Specifically, for the time $t \geq \tilde T_1= \Theta\left(\frac{ \kappa^2}{\eta^2 \sigma_b^2 \min\{\norm{\bmu}_2^2,\sigma_p^2d\}} \right)$, for each $j\in \{\pm 1\}$ and 
     $i\in [n]$, we have
      \begin{align*} 
     	&\max_{r} \sigma'\left( \langle\bw_{j,r}^{(t)}, j\bmu\rangle \right) =1 ,\\
     	&\max_{r} \sigma'\left( \langle\bw_{y_i,r}^{(t)}, \bxi_i \rangle \right)=1.
     \end{align*} 
     Therefore, according to the update rule of $\gamma_{j,r}^{(t)}$, $\orho^{(t)}$ based on Definition \ref{def3}, we have
     \begin{align} 
     	\nonumber \lambda_{i}^{(t+1)} \geq \lambda_{i}^{(t)} + \frac{\eta \max\{\norm{\bmu}_2^2,\sigma_p^2d\}}{12nm^2} e^{-\lambda_{i}^{(t)}},
     \end{align}
     based on this iteration inequality, we can prove that $L_S(\bW^{(\tilde T^*)}) \leq \epsilon$ within $\tilde T^*$ iterations. \\

\iffalse
	according to the tensor power methods, we have a lower bound for each $i \in [n]$ that
	\begin{align*} 
		\lambda_{i}^{(t)} &\geq \log \left(\frac{\eta \max\{\norm{\bmu}_2^2,\sigma_p^2d\}}{12nm^2}  (t- \tilde T_1)\right).
	\end{align*}
	This property gives an upper bound for the training loss of each data point
	\begin{align*}
		\ell \left(y_i f(\bW^{(\tilde T^*)},\bx_i) \right)\leq \frac{48nm^2}{\eta \max\{\norm{\bmu}_2^2,\sigma_p^2d\} (\tilde T^*- \tilde T_1)},
	\end{align*}
	therefore, by the choice of $\tilde T^*= \tilde T_1 + \Theta \left(\frac{nm^2}{\eta \epsilon \max\{\norm{\bmu}_2^2,\sigma_p^2d\}} \right)$, we have $\ell \left(y_i f(\bW^{(\tilde T^*)},\bx_i) \right) \leq \epsilon$ for each $i\in [n]$, hence $L_S(\bW^{(\tilde T^*)}) \leq \epsilon$.  \\
\fi
	
	\noindent \textbf{Test error.}
	To prove the upper bound of the test error, we need the following lemma, which demonstrates that the learning of signal $\frac{1}{m}\sum_{r=1}^m \gam^{(t)}$ can be as large as $\Theta(\frac{1}{m})$ after $\tilde{T}_2$ iterations.
	\begin{lemma} 
		Under Condition \ref{condition2}, denote $c_1= \frac{3\eta \max\{ n\norm \bmu_2^{2}, \sigma_p^2d\}}{nm}$, $\tilde T_1= \Theta \left(\frac{\kappa^2}{\eta^2 \sigma_b^2 \min\{\norm{\bmu}_2^2,\sigma_p^2d\}} \right)$, and $\tilde T_2= \tilde T_1 + \Theta \left(e^{c_1} (\tilde{T}_1 + \frac{1}{c_1}) \right)$. Then, for any $\tilde T_2 \leq t \leq \tilde{T}^*$, we have
		\begin{align}
			\frac{1}{m}\sum_{r=1}^m \gam^{(t)} = \Omega \left(\frac{n \norm \bmu_2^{2}}{e^{c_1} m \max\{n \norm \bmu_2^{2}, \sigma_p^2d\}} \right),
		\end{align}
	\end{lemma}
	
	This lemma is important for the generalization analysis of DP-GD, since it tells us when $t \geq \tilde{T}_2$, the signal learning can escape the initialization order.

	Next, the following lemma shows that if $\sigma_b$ and $t$ are chosen to satisfy a condition, we can derive a high probability upper bound of the test error at time $t$. The condition is chosen to ensure that the signal learning term can dominate the influence of added Gaussian noise to achieve the right classification.
	\begin{lemma}  \label{result:testloss}
		Under Condition \ref{condition2}, let $\tilde{T}_2 \leq t \leq \tilde{T}^*$ and satisfies that
		$$\eta \sigma_b \norm {\bmu}_2 \sqrt{t} \leq \frac{n\norm \bmu_2^{2}}{C m \max\{n \norm \bmu_2^{2}, \sigma_p^2d\}}$$
		for some large constant $C$,
		we have 
		\begin{align*}
			\mathcal{R}_{\mathcal{D}}(\bW^{(t})) & \leq  \exp \bigg( -  C_2 \bigg(\frac{ n \norm \bmu_2^{2}}{e^{c_1} m \max\{n \norm \bmu_2^{2}, \sigma_p^2d\}} \\
			& \cdot  \frac{1}{\sigma_0\sigma_p \sqrt{d} + \frac{\sigma_p m}{\norm \bmu_2} + \frac{mn}{  \sqrt{d}} +  \eta \sigma_b \sigma_p \sqrt{d \tilde{T}_2}}\bigg)^2 \bigg),
		\end{align*}
        where $C_2$ is a positive constant.
	\end{lemma}
	
    \noindent	\textbf{A condition for good test error and DP.}
     The following lemma gives a DP guarantee for DP-GD at time $T$.
    \begin{lemma}[Privacy guarantee]\label{result:dp}
    	The  DP-GD algorithm with $T$ iterations satisfies 
    	%$( \frac{  \lambda m\sum_{t=1}^T\Delta_{\ell_2,t}^2}{ \sigma_b^2}+\frac{\log(2/\delta)}{\lambda-1}, \delta )$-DP for any $\lambda>1$. 
    	$( \frac{T\lambda(2\|\bmu\|_2^2+ 3\sigma_p^2 d)}{\sigma_b^2 n^2 m}+\frac{\log(2/\delta)}{\lambda-1}, \delta )$-DP for any $\lambda>1$.
    \end{lemma}
    Based on Lemma \ref{result:testloss} and Lemma \ref{result:dp}, we can select a condition to ensure that DP-GD can achieve a good test error and a good DP guarantee at the same time. We choose $\eta$ large enough depending on $\sigma_b$ to ensure that
    $$\tilde T_1=  \Theta \left(\frac{\kappa^2}{\eta^2 \sigma_b^2 \min\{\norm{\bmu}_2^2,\sigma_p^2d\}} \right) = O(1).$$
    We also choose $\eta$ to be not that large, so that $c_1=O(1)$. Under these two conditions on $\eta$, we have 
    $$\tilde T_2= \tilde T_1 + \Theta \Big(e^{c_1} (\tilde{T}_1 + \frac{1}{c_1}) \Big)= \Theta \Big( \frac{nm}{\eta \max\{ n\norm \bmu_2^{2}, \sigma_p^2d\}}\Big) .$$
    Furthermore, we choose $\sigma_b$ to be the largest condition such that it satisfies
    \begin{align*}
    	& \eta \sigma_b  \|\bmu\|_2 \sqrt{\tilde T_2} \leq  \frac{n\norm \bmu_2^{2}}{C m \max\{n \norm \bmu_2^{2}, \sigma_p^2d\}}, \\
    	& \eta \sigma_b  \sigma_p \sqrt{d \tilde T_2}\leq  \frac{n\norm \bmu_2^{2}}{C m \max\{n \norm \bmu_2^{2}, \sigma_p^2d\}},
    \end{align*}
    so that the DP guarantee is as good as possible. Such a choice of $\sigma_b$ further provides us with a detailed condition on the lower bound of $\eta$. Furthermore, according to Lemma \ref{result:testloss}, we need to choose $\sigma_0$, $\sigma_p$, $d$ to satisfy the conditions to ensure that
    \begin{align*}
    	\sigma_0\sigma_p \sqrt{d} + \frac{\sigma_p m}{\norm \bmu_2} + \frac{mn}{  \sqrt{d}}  \leq  \frac{n\norm \bmu_2^{2}}{C m \max\{n \norm \bmu_2^{2}, \sigma_p^2d\}}.
    \end{align*} 
    Finally, according to Lemma \ref{result:dp},  DP-GD with $\tilde T_2$ iterations satisfies $\left( \frac{C_4 m^3  \max^2\{ \norm \bmu_2^{2}, \sigma_p^2d\}}{n^2 \norm \bmu_2^{4}} \log \frac{2}{\delta}, \delta \right)$-DP for some positive constant $C_4$. Therefore, we need the SNR condition that $\frac{\sigma_p \sqrt{d}}{\norm \bmu_2} \leq \frac{\sqrt{n}}{C}$ to ensure that the DP guarantee is good.

\section{Numerical Experiments}

In this section, we empirically validate our theoretical analysis by demonstrating that a two-layer CNN trained with DP-GD exhibits stronger noise robustness than one trained with GD. Following Definition~\ref{datadistribution}, we generate three datasets with fixed signal scale $\|\bmu\|_2=1$ and noise levels $\sigma_p\in\{0.1,0.3,0.5\}$. Unless otherwise stated, we set input dimension $d = 2000$, number of convolutional kernels $m = 100$, training sample size $n = 1000$, DP-GD noise parameter $\sigma_b$ = 0.01, and learning rate $\eta = 0.1$. Although these settings are milder than the theoretical requirements in Condition~\ref{condition2}, they still capture the over-parameterized regime and support the generality of our theory. The code is available at \url{https://github.com/ZhongjieSHI/Paper-Codes}.

The results are summarized in Figure~\ref{fig:synthe_1}. Figure~\ref{subfig:training_loss} shows that both GD and DP-GD achieve fast training loss decay across all noise levels, consistent with Theorem~\ref{theorem: GD main} and Theorem~\ref{result:NGDtraining}. However, their generalization performance differs substantially as noise increases. As shown in Figures~\ref{subfig:test_loss} and~\ref{subfig:test_accuracy}: (i) for $\sigma_p=0.1$, both methods achieve near-zero test loss and almost $100\%$ accuracy; (ii) for $\sigma_p=0.3$, DP-GD attains lower test loss (by about $0.2$) and slightly higher accuracy ($100\%$ vs.\ $95\%$); (iii) for $\sigma_p=0.5$, DP-GD maintains strong generalization (test loss $<0.2$, $95\%$ accuracy), while GD largely fails, reaching only $75\%$ accuracy. These results confirm that DP-GD significantly improves generalization performance under high noise, in line with our theoretical predictions.

\section{Conclusion and Future Work}
This paper presents a theoretical study showing that, in certain binary classification tasks using two-layer Huberized ReLU CNNs, under mild conditions on the problem setting, model architecture, and hyperparameters in the training algorithm, the training loss of both GD and DP-GD can converge to an arbitrarily small value. Nonetheless, by using the tool of early stopping, DP-GD can achieve better generalization performance than GD under appropriate signal-to-noise conditions, and ensure good privacy guarantees at the same time, highlighting its potential in achieving a trustworthy deep learning scheme in certain learning tasks. An important future work direction is to derive a more refined analysis of the privacy guarantee by utilizing the detailed properties during the training dynamics. It is also interesting to generalize the results in this paper to other algorithms and learning tasks.

\section*{Acknowledgments} % 3 sentences max
We thank the anonymous reviewers for their helpful comments. Yuan Cao is partially supported by NSFC 12301657 and Hong Kong RGC-ECS 27308624.
The work of Puyu Wang is partially supported by the Alexander von Humboldt Foundation.

\bibliography{ref}

\bibliographystyle{ims}

\appendix
%\onecolumn

\iffalse
\renewcommand{\section}{\@startsection{section}{1}{\z@}%
	{-3.5ex \@plus -1ex \@minus -.2ex}%
	{2.3ex \@plus.2ex}%
	{\normalfont\large\bfseries\raggedright}}
\fi

\newpage

\section*{Appendix}
In the appendix, we prove the main results described in the paper. 

\section{Preliminary Lemmas} \label{appendixB}
In this section, we present some pivotal lemmas that illustrate some important properties of the data and neural network parameters at their random initialization. Let $T^*$ be the maximum number of iterations. The following lemma estimates the norms of the noise vectors $\bxi_i, i \in [n]$, the upper bound of their inner products with each other and with the signal vector $\bmu$. 

\begin{lemma} \label{innerproductxi}
    Suppose that $\delta>0$ and $d= \Omega(\log\left(\frac{6n}{\delta}\right))$. Then with probability at least $1-\delta$,
    \begin{eqnarray*}
        && \frac{\sigma_p^2d}{2} \leq \norm {\bxi_i}_2^2 \leq \frac{3\sigma_p^2d}{2}, \\
        && \left| \la \bxi_i,\bxi_{i'} \ra \right| \leq 2\sigma_p^2 \cdot \sqrt{d \log\left(\frac{6n^2}{\delta}\right)}, \\
        && \left| \la \bxi_i,\bmu \ra \right| \leq \norm {\bmu}_2 \sigma_p \cdot \sqrt{2 \log\left(\frac{6n}{\delta}\right)}, 
        %&& \left| \la \bjrk,\bmu \ra \right| \leq \norm {\bmu}_2 \sigma_b \cdot \sqrt{2 \log(20mT^*/\delta)}, \\
        %&& \left| \la \bjrk,\bxi_i \ra \right| \leq 2 \sigma_p \sigma_b \cdot \sqrt{d \log(20mnT^*/\delta)}.
    \end{eqnarray*}
    for all $i,i' \in [n]$, $j\in \{\pm 1\}$, and $r\in [m]$.
\end{lemma}

\begin{proof} [Proof of Lemma \ref{innerproductxi}]
    Since $\bxi_i \sim \cN(\bm 0, \sigma_p^2 \Ib)$, by Bernstein's inequality, with probability at least $1-\frac{\delta}{3n}$ we have
    \begin{equation*}
        \left| \norm {\bxi_i}_2^2- \sigma_p^2d \right|= O\left(\sigma_p^2 \cdot \sqrt{d \log\left(\frac{6n}{\delta}\right)}\right).
    \end{equation*}
    Therefore, if we set $d= \Omega(\log\left(\frac{6n}{\delta}\right))$, we get
    \begin{equation*}
        \frac{\sigma_p^2d}{2} \leq \norm {\bxi_i}_2^2 \leq \frac{3\sigma_p^2d}{2}.
    \end{equation*}
    Moreover, clearly $\la \bxi_i,\bxi_{i'} \ra$ has mean zero. For any $i,i'$ with $i \neq i'$, by Bernstein's inequality, with probability at least $1-\frac{\delta}{3n^2}$ we have
    \begin{equation*}
        \left| \la \bxi_i,\bxi_{i'} \ra \right| \leq 2\sigma_p^2 \cdot \sqrt{d \log\left(\frac{6n^2}{\delta}\right)},
    \end{equation*}
    Furthermore, note that $\la \bxi_i,\bmu \ra \sim \mathcal{N}(0,\norm {\bmu}_2^2 \sigma_p^2)$. By Gaussian tail bounds, with probability at least $1-\frac{\delta}{3n}$ we have
    \begin{equation*}
        \left| \la \bxi_i,\bmu \ra \right| \leq \norm {\bmu}_2 \sigma_p \cdot \sqrt{2 \log\left(\frac{6n}{\delta}\right)}.
    \end{equation*}
    Applying a union bound completes the proof.
\end{proof}

Then we state Azuma's inequality for light-tailed random variables \cite[Theorem 3]{attia2024note}.
\begin{lemma} \label{azumainequality}
    Let $Z_1, \dots, Z_n$ be a martingale difference sequence (MDS) and suppose there are constants $\gamma, \sigma, c>0$ such that, deterministically,
    \begin{equation*}
        \mathbb{P}\left(|Z_i| \geq t| Z_1, \dots, Z_{i-1}\right) \leq c \cdot \exp\left(- \left(\frac{t}{\sigma}\right)^{\gamma} \right), \quad \forall \ 0\leq i<n, \ t\geq 0.
    \end{equation*}
    Then with probability at least $1-2\delta$, we have
    \begin{equation*}
        \sum_{i=1}^n Z_i \leq \sigma \sqrt{32n\log \left(\frac{1}{\delta}\right)} \left(\max \left\{\log \left(\frac{2cn}{\delta}\right), \frac{2}{\gamma} \right\} \right).
    \end{equation*}
\end{lemma}

By utilizing Azuma's inequality stated above, we have the following  lemma that estimate the upper bound of the sum of the inner products of $\bjrk$ with $\bmu$, and the sum of the inner products of  $\bjrk$ with $\bxi_i$.
\begin{lemma} \label{sumofnoisebound}
    Suppose that $\delta>0$. Then with probability at least $1-\delta$,
    \begin{align*}
        & \sum_{k=0}^{t-1} \la \bjrk,\bmu \ra \leq  8 \sigma_b \norm {\bmu}_2 \sqrt{t}  \log ^2\left(\frac{32mt}{\delta}\right), \\
        & \sum_{k=0}^{t-1} \la \bb_{j,r,k},\bxi_{i} \ra \leq  16\sigma_b \sigma_p \sqrt{dt}  \log ^2\left(\frac{32mnt}{\delta}\right),
    \end{align*}
    for all $i \in [n]$, $j\in \{\pm 1\}$, $r\in [m]$, and $t \in \NN$.
\end{lemma}

\begin{proof} [Proof of Lemma \ref{sumofnoisebound}]
    Denote $Z_k= \la \bjrk,\bmu \ra \sim \mathcal{N}(0,\norm {\bmu}_2^2 \sigma_b^2)$. Then by Gaussian tail bounds,
    \begin{equation*}
        \mathbb{P}\left(Z_k \geq t| Z_1,\dots, Z_{k-1}\right) \leq 2 \exp\left(-\left(\frac{t}{\sqrt{2} \norm {\bmu}_2 \sigma_b}\right)^2\right),
    \end{equation*}
    then by Lemma \ref{azumainequality}, for each $j \in \{+1,-1\}$ and $r \in [m]$, we have with probability at least $1-\frac{\delta}{4m}$,
    \begin{equation*}
        \sum_{k=0}^{t-1} \la \bjrk,\bmu \ra \leq  8\norm {\bmu}_2 \sigma_b \sqrt{t}  \log ^2\left(\frac{32mt}{\delta}\right).
    \end{equation*}
    
    Next, note that $\bb_{j,r,t} \sim \cN(\zero,\sigma_b^2 \bI_d)$ and $\bxi_i \sim \cN(\bm 0, \sigma_p^2 \bI_d)$, it follows that $Z_{k}^i :=\la \bb_{j,r,k},\bxi_{i} \ra$ has mean zero, by Bernstein's inequality, we have
    \begin{equation*}
        \mathbb{P}\left(Z_k^i \geq t| Z_1^i,\dots, Z_{k-1}^i\right) \leq 2 \exp\left(-\left(\frac{t}{2\sigma_b \sigma_p \sqrt{d}}\right)^2\right),
    \end{equation*}
    then by Lemma \ref{azumainequality}, for each $j \in \{+1,-1\}$, $r \in [m]$ and $i \in [n]$, we have with probability at least $1-\frac{\delta}{4mn}$,
    \begin{equation*}
        \sum_{k=0}^{t-1} \la \bb_{j,r,k},\bxi_{i} \ra \leq  16\sigma_b \sigma_p \sqrt{dt}  \log ^2\left(\frac{32mnt}{\delta}\right).
    \end{equation*}
    Applying a union bound completes the proof.
\end{proof}

In the following Lemma, we provide a high probability upper bound on the norm of the sum of added Gaussian noises $\bjrk$.
\begin{lemma} \label{sumofnoisenorm}
    Suppose that $t \in \NN$ and $\delta>0$. Then with probability at least $1-\delta$,
    \begin{align*}
       \left\|\sum_{k=0}^{t-1} \bjrk\right\|_2 \leq 2\sigma_b \sqrt{2dt \log \left(\frac{2m}{\delta}\right)},
    \end{align*}
    for all $j\in \{\pm 1\}$, $r\in [m]$.
\end{lemma}

\begin{proof} [Proof of Lemma \ref{sumofnoisebound}]
    According to the Gaussian concentration inequality, with probability at least $1- \frac{\delta}{2m}$, 
    \begin{align*}
        \left\|\sum_{k=0}^{t-1} \bjrk\right\|_2 \leq 2\sigma_b \sqrt{2dt \log \left(\frac{2m}{\delta}\right)},
    \end{align*}
    Applying a union bound completes the proof.
\end{proof}

The following lemma estimates the norms of the noise vectors $ \wjr^{(0)}$, the upper bound of its inner products with noises $\bxi_i$ and with the signal vector $\bmu$, and the lower bound of $\max_{r\in[m]}  \la \bw_{j,r}^{(0)}, j\bmu \ra$ and $\max_{r\in[m]} \la \bw_{j,r}^{(0)}, \bxi_i \ra$.
\begin{lemma} \label{innerproductw0}
    Suppose that $d=\Omega(\log(\frac{m}{\delta}))$, $ m = \Omega(\log(1 / \delta))$. Then with probability at least $1-\delta$,
    \begin{eqnarray*}
        && \frac{\sigma_0^2 d}{2} \leq \| \wjr^{(0)}\|_2^2 \leq \frac{3\sigma_0^2d}{2}, \\
        && |\la \bw_{j,r}^{(0)}, j \bmu \ra| \leq \sigma_0 \| \bmu \|_2 \cdot \sqrt{2\log \left(\frac{12m}{\delta}\right)}, \\
        && |\la \wjr^{(0)}, \bxi_i \ra| \leq 2\sigma_0 \sigma_p \sqrt{d} \cdot \sqrt{\log\left(\frac{12mn}{\delta}\right)} ,
    \end{eqnarray*}
    for all $r\in [m]$, $j\in \{\pm 1\}$, and $i\in [n]$. Moreover,
    \begin{align*}
    &\frac{\sigma_0 \| \bmu \|_2}{2}  \leq \max_{r\in[m]}  \la \bw_{j,r}^{(0)}, j\bmu \ra \leq \sigma_0 \| \bmu \|_2 \cdot \sqrt{2\log \left(\frac{12m}{\delta}\right)},\\
    &\frac{\sigma_0 \sigma_p \sqrt{d}}{4}  \leq \max_{r\in[m]}  \la \bw_{j,r}^{(0)}, \bxi_i \ra \leq 2\sigma_0 \sigma_p \sqrt{d} \cdot \sqrt{\log\left(\frac{12mn}{\delta}\right)},
\end{align*}
for all $j\in \{\pm 1\}$ and $i\in [n]$.
\end{lemma}

\begin{proof} [Proof of Lemma \ref{innerproductw0}]
    Since $\bw_{j,r}^{(0)} \sim \cN (0, \sigma_0^2 \Ib)$, 
    by Bernstein's inequality, with probability at least $1 - \frac{\delta}{12m}$ we have
    \begin{align*}
        \left| \| \bw_{j,r}^{(0)} \|_2^2 - \sigma_0^2 d \right| = O\left(\sigma_0^2 \cdot \sqrt{d \log \left(\frac{24m}{\delta} \right)}\right).
    \end{align*}
    Therefore, if we set appropriately $d = \Omega( \log(\frac{m}{\delta}) )$, we get
    \begin{align*}
         \frac{\sigma_0^2 d}{2}  \leq \| \bw_{j,r}^{(0)} \|_2^2 \leq \frac{3\sigma_0^2 d}{2}.
    \end{align*}
    
    Moreover, since $ \la \bw_{j,r}^{(0)}, j\bmu \ra \sim \cN(0, \sigma_0^2 \| \bmu \|_2^2)$. By Gaussian tail bounds, with probability at least $1 - \frac{\delta}{6m}$, we have
    \begin{align*}
       |\la \bw_{j,r}^{(0)}, j\bmu \ra| \leq \sigma_0 \| \bmu \|_2 \cdot \sqrt{2\log \left(\frac{12m}{\delta}\right)} .
    \end{align*}
    
    Similarly, since $\la \wjr^{(0)}, \bxi_i \ra$ has mean zero,  by Bernstein's inequality, with probability at least $1 - \frac{\delta}{6mn}$ we have
    \begin{align*}
        |\la \wjr^{(0)}, \bxi_i \ra| \leq 2\sigma_0 \sigma_p \sqrt{d} \cdot \sqrt{\log\left(\frac{12mn}{\delta}\right)} .
    \end{align*}
    
    Finally, notice that $\PP\left( \frac{\sigma_0 \| \bmu \|_2}{2} >  \la \bw_{j,r}^{(0)}, j\bmu \ra \right)$ is an absolute constant, by utilizing the condition on $m$, we have
    \begin{align*}
        \PP\left( \frac{\sigma_0 \| \bmu \|_2}{2}  \leq \max_{r\in[m]}  \la \bw_{j,r}^{(0)}, j\bmu \ra \right) &= 1 - \PP\left( \frac{\sigma_0 \| \bmu \|_2}{2} > \max_{r\in[m]}  \la \bw_{j,r}^{(0)}, j\bmu \ra \right) \\
        &= 1 - \PP\left( \frac{\sigma_0 \| \bmu \|_2}{2} >  \la \bw_{j,r}^{(0)}, j\bmu \ra \right)^{m} \\
        &\geq 1 - \frac{\delta}{12}.
    \end{align*}
    % Since $\bxi_i$, $i\in[n]$ are i.i.d. Gaussian random vectors, 
    The result for $\max_{r\in[m]} \la \bw_{j,r}^{(0)}, \bxi_i\ra$ follows the same proof as $\max_{r\in[m]} \la \bw_{j,r}^{(0)}, j\bmu \ra$. 
    Therefore, by applying a union bound we complete the proof. 
\end{proof}

\iffalse
\begin{lemma} \label{innerproductbjrt}
    Suppose that $ m = \Omega(\log(1 / \delta))$. Then with probability at least $1-\delta$,
    \begin{align*}
       \max_{r\in[m]} \la \bb_{j,r,0}, -j \bmu \ra \geq \frac{\sigma_b \norm{\bmu}_2}{2},
    \end{align*}
    for all $j\in \{\pm 1\}$, $i\in [n]$, and $t \in \NN$.
\end{lemma}
\fi

Similarly, we can derive the following lemma that provides a lower bound on $\max_{r\in[m]} \sum_{k=1}^t \la \bb_{j,r,k}, j \bmu \ra $ and $ \max_{r\in[m]} \sum_{k=1}^t  \la \bb_{j,r,k}, \bxi_i \ra $. The proof is omitted since it has the same idea as the proof of Lemma \ref{innerproductw0}.
\begin{lemma} \label{maxinnerproductbjrt}
    Suppose that $ m = \Omega\left(\log\frac{\tilde T^*}{\delta}\right)$. Then with probability at least $1-\delta$,
    \begin{align*}
    &\frac{\sigma_b \| \bmu \|_2 \sqrt{t}}{2}  \leq \max_{r\in[m]} \sum_{k=1}^t \la \bb_{j,r,k}, j \bmu \ra \leq 8 \sigma_b \norm {\bmu}_2 \sqrt{t}  \log ^2\left(\frac{16t}{\delta}\right),\\
    &\frac{\sigma_b \sigma_p \sqrt{dt}}{4}  \leq \max_{r\in[m]} \sum_{k=1}^t  \la \bb_{j,r,k}, \bxi_i \ra \leq  16\sigma_b \sigma_p \sqrt{dt}  \log ^2\left(\frac{16nt}{\delta}\right),
    \end{align*}
for all $j\in \{\pm 1\}$, $i\in [n]$, and $t \leq \tilde T^*$.
\end{lemma}

Denote $\Gamma_j=\{i|y_i=j\}$, with $j\in \{\pm 1\}$, the following lemma states the bounds of its size with high probability. 

\begin{lemma}\label{Sjsize}
Suppose that $\delta>0$ and $n\geq 8\log\left(\frac{4}{\delta}\right)$. Then with probability at least $1-\delta$, 
\begin{equation*}
    |\Gamma_{j}| \in \left[\frac{n}{4}, \frac{3n}{4} \right], \quad \text{for} \ j\in \{\pm 1\}.
\end{equation*}
\end{lemma}

\begin{proof}[Proof of Lemma \ref{Sjsize}]
According to the data distribution $\cD$ defined in Definition \ref{datadistribution}, for $(\bx, y) \sim \cD$, we have 
\begin{equation*}
    \PP(y=1) =\PP(y=-1) = \frac{1}{2} .
\end{equation*}
Since $|\Gamma_{1}|=\sum_{i=1}^{n} \indicator[y_i=1]$, $|\Gamma_{-1}|=\sum_{i=1}^{n} \indicator[y_i=-1]$, we have $\mathbb{E}[|\Gamma_{1}|]=\mathbb{E}[|\Gamma_{-1}|]=\frac{n}{2}$. By Hoeffding’s inequality, for arbitrary $t>0$ the following holds: 
\begin{align*}
    &\mathbb{P}\big(\big||\Gamma_1|-\mathbb{E}[|\Gamma_1|]\big|\geq t\big)\leq2\exp\Big(-\frac{2t^2}{n}\Big),\\
    &\mathbb{P}\big(\big||\Gamma_{-1}|-\mathbb{E}[|\Gamma_{-1}|]\big|\geq t\big)\leq2\exp\Big(-\frac{2t^2}{n}\Big).
\end{align*}
Setting $t=\sqrt{\frac{n}{2}\log\left(\frac{4}{\delta}\right)}$, and taking a union bound, we have with probability at least $1-\delta$, 
\begin{align*}
    &\Big||\Gamma_{1}|-\frac{n}{2}\Big|\leq\sqrt{\frac{n}{2}\log\Big(\frac{4}{\delta}\Big)},\\
    &\Big||\Gamma_{-1}|-\frac{n}{2}\Big|\leq\sqrt{\frac{n}{2}\log\Big(\frac{4}{\delta}\Big)}.
\end{align*}
Therefore, as long as $n\geq8\log(4/\delta)$, we have $\sqrt{\frac{n}{2}\log\left(\frac{4}{\delta}\right)} \leq \frac{n}{4}$, and hence $\frac{n}{4}\leq |\Gamma_{1}|,|\Gamma_{-1}|\leq \frac{n}{4}$.
\end{proof}

We also need the tensor power method bounds stated in our previous work \cite[Lemma E.8, Lemma E.9,]{zhang2025gradient}. The first one compares the growth speed of two sequences of updates, the second one estimates the growth speed of a sequence of updates. 
\begin{lemma} \label{tensorpower}
Suppose a positive sequence \(\{x_t\}_{t=0}^\infty\) satisfies the following iterative rules:
\[
\begin{aligned}
x_{t+1} &\geq x_t + \eta \cdot C_1 \cdot x_t^{q-1}, \\
x_{t+1} &\leq x_t + \eta \cdot C_2 \cdot x_t^{q-1},
\end{aligned}
\]
with \(C_2 \geq C_1 > 0\). For any \(v > x_0\), let \(T_v\) be the first time such that \(x_t \geq v\). Then for any constant \(\zeta > 0\), we have
\[
T_v \leq \frac{1 + \zeta}{\eta C_1 x_0^{q-2}} + \frac{(1 + \zeta)^{q-1} C_2 \log\left(\frac{v}{x_0}\right)}{C_1},
\]
and
\[
T_v \geq \frac{1}{(1 + \zeta)^{q-1} \eta C_2 x_0^{q-2}} - \frac{\log\left(\frac{v}{x_0}\right)}{(1 + \zeta)^{q-2}}. 
\]
\end{lemma}

\begin{lemma} \label{tensorpower2}
    Suppose that a positive sequence $x_t$, $t \geq 0$ follows the iterative formula
    \begin{equation*}
    x_{t+1} = x_t + c_1 e^{-c_2 x_t}
    \end{equation*}
    for some $c_1, c_2 > 0$. Then it holds that
    \begin{equation*}
    \frac{1}{c_2} \log(c_1 c_2 t + e^{c_2 x_0}) \leq x_t \leq c_1 e^{-c_2 x_0} + \frac{1}{c_2} \log(c_1 c_2 t + e^{c_2 x_0})
    \end{equation*}
    for all $t \geq 0$.
\end{lemma}

\section{Noise Memorization of GD}
In this section, we first consider the noise memorization case by utilizing the GD training algorithm under Condition \ref{condition}. These results are based on the conclusions in Appendix \ref{appendixB}. We use $\cE_{\text{prelim}}$ to denote the event that all the results in Appendix \ref{appendixB} hold (for a given $\delta$, we have $\PP(\cE_{\text{prelim}}) \geq 1-6\delta$ by a union bound). For simplicity and clarity, we state all the results in this and the following sections conditional on $\cE_{\text{prelim}}$. 

\subsection{Signal-noise Decomposition and the properties} \label{subsecB1}
We begin by analyzing the coefficients in the signal-noise decomposition and establishing a serise of properties, which holds with high probability. The first lemma presents an iterative expression for the change of coefficients.

\begin{lemma} \label{lemma:GDcoefficient}
    The coefficients $\gam^{(t)},\orho^{(t)},\urho^{(t)}$ defined in Definition \ref{GD:w_decomposition} satisfy the following iterative equations:
    \begin{eqnarray*}
        &&\gam^{(0)},\orho^{(0)},\urho^{(0)} = 0, \\
        &&\gam^{(t+1)}= \gam^{(t)} - \frac{\eta}{nm} \sum_{i=1}^n \ell_{i}^{\prime(t)} \cdot \sigma^\prime \left( \langle \bw_{j,r}^{(t)}, y_{i} \bmu \rangle\right) \cdot \norm \bmu_2^{2}, \\
        && \orho^{(t+1)}= \orho^{(t)} -\frac{\eta}{nm} \ell_{i}^{\prime(t)} \cdot \sigma^\prime \left( \langle \bw_{j,r}^{(t)}, \bxi_{i} \rangle\right)  \cdot \norm {\bxi_i}_2^{2} \cdot \mathds{1} (y_{i}=j) ,  \\
        && \urho^{(t+1)}= \urho^{(t)} +\frac{\eta}{nm} \ell_{i}^{\prime(t)} \cdot \sigma^\prime \left( \langle \bw_{j,r}^{(t)}, \bxi_{i} \rangle\right)  \cdot \norm {\bxi_i}_2^{2} \cdot \mathds{1} (y_{i}=-j) .
    \end{eqnarray*}
\end{lemma}

\begin{proof} [Proof of Lemma \ref{lemma:GDcoefficient}]
    Note that the vectors are linearly independent with probability 1, thus the decomposition \eqref{eqGD:w_decomposition} is unique. Now consider $\tilde \gamma_{j,r}^{(0)}= \tilde \rho_{j,r,i}^{(0)}= 0$, and
    \begin{eqnarray*}
        &&\tilde \gamma_{j,r}^{(t+1)}= \tilde \gamma_{j,r}^{(t)} - \frac{\eta}{nm} \sum_{i=1}^n \ell_{i}^{\prime(t)} \cdot \sigma^\prime \left( \langle \bw_{j,r}^{(t)}, y_{i} \bmu \rangle\right) \cdot \norm \bmu_2^{2}, \\
        && \tilde \rho_{j,r,i}^{(t+1)}= \tilde \rho_{j,r,i}^{(t)} -\frac{\eta}{nm} \ell_{i}^{\prime(t)} \cdot \sigma^\prime \left( \langle \bw_{j,r}^{(t)}, \bxi_{i} \rangle\right)  \cdot \norm {\bxi_i}_2^{2} \cdot j y_i .
    \end{eqnarray*}
    Then it is easy to check by the gradient update rule (\ref{GDupdate}) that
    \begin{equation*}
        \wjr^{(t)}= \wjr^{(0)} + j \cdot \tilde \gamma_{j,r}^{(t)} \cdot \norm \bmu_2^{-2} \cdot \bmu + \sum_{i=1}^n \tilde \rho_{j,r,i}^{(t)} \cdot \norm {\bxi_i}_2^{-2} \cdot \bxi_i.
    \end{equation*}
    Hence by the uniqueness of the decomposition, we have $\gamma_{j,r}^{(t)}= \tilde \gamma_{j,r}^{(t)}$, and $\rho_{j,r,i}^{(t)}=\tilde \rho_{j,r,i}^{(t)}$. Therefore, we have that
    \begin{eqnarray}
        %&& \gam^{(t)}= - \sum_{s=0}^{t-1} \frac{\eta}{nm} \ell_{i}^{\prime(s)} \cdot \sigma^\prime \left( \langle \bw_{j,r}^{(s)}, y_{i_s} \bmu \rangle\right) \cdot \norm \bmu_2^{2}, \label{B1} \\
         \nonumber \rho_{j,r,i}^{(t)}= -\frac{\eta}{nm} \sum_{s=0}^{t-1} \ell_{i}^{\prime(s)} \cdot \sigma^\prime \left( \langle \bw_{j,r}^{(s)}, \bxi_{i} \rangle\right)  \cdot \norm {\bxi_i}_2^{2} \cdot j y_i, 
    \end{eqnarray}
    Moreover, notice that $\ell_i^{\prime(t)}<0$ for the cross-entropy loss, we have
    \begin{eqnarray}
        && \orho^{(t+1)}= -\frac{\eta}{nm} \sum_{s=0}^{t-1} \ell_{i}^{\prime(s)} \cdot \sigma^\prime \left( \langle \bw_{j,r}^{(s)}, \bxi_{i} \rangle\right)  \cdot \norm {\bxi_i}_2^{2} \cdot \mathds{1} (y_{i}=j), \label{B2} \\
        && \urho^{(t+1)}= \frac{\eta}{nm} \sum_{s=0}^{t-1} \ell_{i}^{\prime(s)} \cdot \sigma^\prime \left( \langle \bw_{j,r}^{(s)}, \bxi_{i} \rangle\right)  \cdot \norm {\bxi_i}_2^{2} \cdot \mathds{1} (y_{i}=-j) \label{B3}.
    \end{eqnarray}
    Writing out the iterative formulations of \eqref{B2} and \eqref{B3} completes the proof.
\end{proof}

Next, we will show that the coefficients in the signal-noise decomposition will stay within a reasonable range for a considerable amount of time. Consider the training period $0\leq t \leq T^*$, where $T^*=\tilde{O}\left( \frac{\kappa^{q-1} mn}{\eta \sigma_0^{q-2} (\sigma_p \sqrt{d})^q} + \frac{m^{3}n}{\eta \epsilon \norm{\bmu}_2^2} \right)$ is the maximum admissible iterations. Denote 
\begin{eqnarray*}
    && \alpha := 4m \log(T^*), \\
    && \beta := 2 \max_{j,r,i} \left\{ |\la \bw_{j,r}^{(0)}, \bmu \ra|, |\la \bw_{j,r}^{(0)}, \bxi_i \ra | \right\}, \\
    && \zeta := 8 n \sqrt{\frac{\log(6n^2/\delta)}{d}} \alpha.
%    && \snr := \frac{\norm {\bmu}_2}{\sigma_p \sqrt{d}}, \\
    %&& \beta:= \frac{1}{n \cdot {\snr}^2}.
\end{eqnarray*}
Recall that the conditions on $\sigma_0$ and $d$ specified in Condition \ref{condition} satisfy
\begin{align*}
    & d \geq C\frac{m^2 n^2}{\kappa^2} \log(\frac{n^2}{\delta}) (\log(T^*))^2, \\
    &  \sigma_0 \leq \left(C \max\left\{\norm {\bmu}_2 , \sigma_p \sqrt{d} \right\} \sqrt{\log(\frac{mn}{\delta})} \right)^{-1} \kappa,
\end{align*}
by Lemma \ref{innerproductw0}, we have
\begin{align*}
&\beta \leq 2 \max \left\{ \sigma_0 \| \bmu \|_2 \cdot \sqrt{2\log \left(\frac{12m}{\delta}\right)}, 2\sigma_0 \sigma_p \sqrt{d} \cdot \sqrt{\log\left(\frac{12mn}{\delta}\right)} \right\} \leq 0.1 \kappa, \\
&\zeta = 8 n \sqrt{\frac{\log(6n^2/\delta)}{d}} \alpha \leq 0.1 \kappa.
\end{align*}
In the next proposition, we demonstrate the bounds on the growth of $\gam^{(t)}$, $\orho^{(t)}$ and $\urho^{(t)}$.

\begin{proposition} \label{proposition:GDgrowthbound} 
%[Part ial restatement of Proposition \shi{to complete}]
    Under Condition \ref{condition}, for $0\leq t \leq T^*$, we have that
    \begin{eqnarray}
        %&& \gam^{(0)},\orho^{(0)},\urho^{(0)} = 0, \label{eqB7} \\
        && 0 \leq \gam^{(t)}, \orho^{(t)} \leq \alpha, \label{eqB3} \\
        && 0 \geq  \urho^{(t)} \geq -2\beta - 2\zeta \geq -\alpha, \label{eqB4}
    \end{eqnarray}
    for all $r \in [m]$, $j\in \{\pm 1\}$, and $i\in [n]$. Besides, $\gam^{(t)}$ and $\orho^{(t)}$ are non-decreasing for $0 \leq t \leq T^*$.
\end{proposition}

We will use induction to prove Proposition \ref{proposition:GDgrowthbound}. For any time $\tilde t \leq T^*$, we suppose that the results in Proposition \ref{proposition:GDgrowthbound} hold for all the time $0 \leq t \leq \tilde t -1$, then we can derive the following properties, which will be used for the inductive proof of Proposition \ref{proposition:GDgrowthbound}.
\begin{lemma} \label{innerprodectwjrxi}
    Under Condition \ref{condition}, suppose  \eqref{eqB3} and \eqref{eqB4} hold at iteration $t \leq \tilde t -1$. Then, for all $r\in[m]$, $j\in \{\pm 1\}$ and $i\in[n]$, we have
    \begin{eqnarray}
        && \left| \la \wjr^{(t)}, \bmu \ra - j \cdot \gam^{(t)} \right| \leq \beta \leq \kappa , \label{eqB5} \\
        && \left| \la \wjr^{(t)}, \bxi_i \ra -  \orho^{(t)} \right| \leq \beta + \zeta \leq \kappa, \quad \text{if} \ j=y_i, \label{eqB6} \\
        && \left| \la \wjr^{(t)}, \bxi_i \ra -  \urho^{(t)} \right| \leq \beta + \zeta \leq \kappa, \quad \text{if} \ j \neq y_i. \label{eqB7}
    \end{eqnarray}
\end{lemma}

\begin{proof} [Proof of Lemma \ref{innerprodectwjrxi}]
    For any $0 \leq t \leq T^*$, we have from the signal-noise decomposition (\ref{eqGD:w_decomposition}) that
    \begin{align*}
        \la \wjr^{(t)}, \bmu \ra &= j \cdot \gam^{(t)} + \la \wjr^{(0)}, \bmu \ra + \sum_{i=1}^n \orho^{(t)} \cdot \norm {\bxi_i}_2^{-2} \cdot \la \bxi_i,\bmu \ra + \sum_{i=1}^n \urho^{(t)} \cdot \norm {\bxi_i}_2^{-2} \cdot \la \bxi_i,\bmu \ra \\
        &= j \cdot \gam^{(t)} + \la \wjr^{(0)}, \bmu \ra,
    \end{align*}
    it follows that
    \begin{equation*}
        \left| \la \wjr^{(t)}, \bmu \ra - j \cdot \gam^{(t)} \right| \leq \beta \leq \kappa.
    \end{equation*}

    For $j = y_i$ and any $0 \leq t \leq T^*$, we have $\urho^{(t)}=0$, therefore 
    \begin{equation*}
        \begin{aligned}
            \la \wjr^{(t)}, \bxi_i \ra &= \la \wjr^{(0)}, \bxi_i \ra  + \sum_{i'=1}^n \overline{\rho}_{j,r,i'}^{(t)} \cdot \norm {\bxi_{i'}}_2^{-2} \cdot \la \bxi_{i'},\bxi_i \ra  + \sum_{i'=1}^n \underline{\rho}_{j,r,i'}^{(t)} \cdot \norm {\bxi_{i'}}_2^{-2} \cdot \la \bxi_{i'},\bxi_i \ra  \\
            &= \orho^{(t)} + \la \wjr^{(0)}, \bxi_i \ra +  \sum_{i'\neq i} \overline{\rho}_{j,r,i'}^{(t)} \cdot \norm {\bxi_{i'}}_2^{-2} \cdot \la \bxi_{i'},\bxi_i \ra  + \sum_{i'\neq i} \underline{\rho}_{j,r,i'}^{(t)} \cdot \norm {\bxi_{i'}}_2^{-2} \cdot \la \bxi_{i'},\bxi_i \ra ,
        \end{aligned}
    \end{equation*}
    it follows that
    \begin{equation*}
        \begin{aligned}
             \left| \la \wjr^{(t)}, \bxi_i \ra -  \orho^{(t)} \right| 
             &\leq \beta+ \sum_{i'\neq i} \left|\overline{\rho}_{j,r,i'}^{(t)}\right| \cdot \norm {\bxi_{i'}}_2^{-2} \cdot \la \bxi_{i'},\bxi_i \ra  + \sum_{i'\neq i} \left|\underline{\rho}_{j,r,i'}^{(t)}\right| \cdot \norm {\bxi_{i'}}_2^{-2} \cdot \la \bxi_{i'},\bxi_i \ra  \\
             & \leq \beta + 4 \sqrt{\frac{\log(6n^2/\delta)}{d}} \left(  \sum_{i=1}^n \left| \orho^{(t)} \right| +  \sum_{i=1}^n \left| \urho^{(t)} \right| \right)  \\
            & \leq \beta + \zeta \leq \kappa,
        \end{aligned}
    \end{equation*}
     where the first inequality is by triangle inequality, the second inequality is by Lemma \ref{innerproductxi}, the third inequality is by  \eqref{eqB3} and \eqref{eqB4}.
    
    For $j \neq y_i$ and any $0 \leq t \leq T^*$, we have $\orho^{(t)}=0$, hence
    \begin{equation*}
        \begin{aligned}
            \la \wjr^{(t)}, \bxi_i \ra &= \la \wjr^{(0)}, \bxi_i \ra  + \sum_{i'=1}^n \overline{\rho}_{j,r,i'}^{(t)} \cdot \norm {\bxi_{i'}}_2^{-2} \cdot \la \bxi_{i'},\bxi_i \ra + \sum_{i'=1}^n \underline{\rho}_{j,r,i'}^{(t)} \cdot \norm {\bxi_{i'}}_2^{-2} \cdot \la \bxi_{i'},\bxi_i \ra \\
            &= \urho^{(t)} + \la \wjr^{(0)}, \bxi_i \ra + \sum_{i'\neq i} \overline{\rho}_{j,r,i'}^{(t)} \cdot \norm {\bxi_{i'}}_2^{-2} \cdot \la \bxi_{i'},\bxi_i \ra  + \sum_{i'\neq i} \underline{\rho}_{j,r,i'}^{(t)} \cdot \norm {\bxi_{i'}}_2^{-2} \cdot \la \bxi_{i'},\bxi_i \ra  ,
        \end{aligned}
    \end{equation*}
    it follows that
    \begin{equation*}
        \begin{aligned}
             \left| \la \wjr^{(t)}, \bxi_i \ra -  \urho^{(t)} \right| 
             &\leq \beta+ \sum_{i'\neq i} \left|\overline{\rho}_{j,r,i'}^{(t)}\right| \cdot \norm {\bxi_{i'}}_2^{-2} \cdot \la \bxi_{i'},\bxi_i \ra  + \sum_{i'\neq i} \left|\underline{\rho}_{j,r,i'}^{(t)}\right| \cdot \norm {\bxi_{i'}}_2^{-2} \cdot \la \bxi_{i'},\bxi_i \ra  \\
             & \leq \beta + 4 \sqrt{\frac{\log(6n^2/\delta)}{d}} \left(  \sum_{i=1}^n \left| \orho^{(t)} \right| +  \sum_{i=1}^n \left| \urho^{(t)} \right| \right)  \\
            & \leq \beta + \zeta \leq \kappa,
        \end{aligned}
    \end{equation*}
     where the first inequality is by triangle inequality, the second inequality is by Lemma \ref{innerproductxi}, the third inequality is by  \eqref{eqB3} and \eqref{eqB4}.
    Thus we complete the proof.
\end{proof}

\begin{lemma} \label{Fminusyibound}
    Under Condition \ref{condition}, suppose  \eqref{eqB3} and \eqref{eqB4} hold for any iteration $t \leq \tilde t -1$. For all $i\in[n]$, if $j \neq y_i$, we have
    \begin{equation}
        F_{j} (\bW_j^{(t)},\bx_i) \leq \frac{2\kappa^{1-q}}{q} (\beta+ \zeta)^q \leq \kappa.
    \end{equation}
\end{lemma}

\begin{proof} [Proof of Lemma \ref{Fminusyibound}]
    By writing out the expression of $F_{j} (\bW_j^{(t)},\bx_i)$, we have
    \begin{equation*}
    \begin{aligned}
        F_{j} (\bW_j^{(t)},\bx_i) &= \frac{1}{m} \sum_{r=1}^m \left[\sigma\left(\la \wjr^{(t)}, y_i \cdot \bmu \ra\right) + \sigma\left(\la \wjr^{(t)},  \bxi_i \ra\right)\right] \\
         &= \frac{1}{m} \sum_{r=1}^m \left[\sigma\left(\la \wjr^{(t)}, y_i \cdot \bmu \ra+ \gam^{(t)}- \gam^{(t)} \right) + \sigma\left(\la \wjr^{(t)},  \bxi_i \ra -\urho^{(t)}+ \urho^{(t)} \right)\right] \\
        & \leq \frac{1}{m} \sum_{r=1}^m \left[ \sigma\left(\left| \la \wjr^{(t)}, y_i \cdot \bmu \ra+ \gam^{(t)} \right|\right) + \sigma\left(\left|\la \wjr^{(t)},  \bxi_i \ra- \urho^{(t)} \right| \right) \right] \\
        & \leq  \frac{\kappa^{1-q}}{q} \beta^q + \frac{\kappa^{1-q}}{q} (\beta+ \zeta)^q \leq \frac{2\kappa^{1-q}}{q} (\beta+ \zeta)^q \leq \kappa,
    \end{aligned}
    \end{equation*}
    where the first inequality is by $\gam^{(t)} \geq 0$, $\urho^{(t)} \leq 0$, and the non-decreasing property of Huberized ReLU activation function $\sigma$, the second inequality is by \eqref{eqB5} and \eqref{eqB7}, and the property that $\beta + \zeta \leq \kappa$.
\end{proof}

We are now ready to prove Proposition \ref{proposition:GDgrowthbound}.
\begin{proof} [Proof of Proposition \ref{proposition:GDgrowthbound}]
    We prove the results by induction. The results in Proposition \ref{proposition:GDgrowthbound} explicitly hold for $t=0$. Then we suppose that for $\tilde t \leq T^*$, the results in Proposition \ref{proposition:GDgrowthbound} hold for all the time $0 \leq t \leq \tilde t-1$.  Recall the update rule in Lemma \ref{lemma:GDcoefficient} that
    \begin{eqnarray*}
        &&\gam^{(t+1)}= \gam^{(t)} - \frac{\eta}{nm} \sum_{i=1}^n \ell_{i}^{\prime(t)} \cdot \sigma^\prime \left( \langle \bw_{j,r}^{(t)}, y_{i} \bmu \rangle\right) \cdot \norm \bmu_2^{2}, \\
        && \orho^{(t+1)}= \orho^{(t)} -\frac{\eta}{nm} \ell_{i}^{\prime(t)} \cdot \sigma^\prime \left( \langle \bw_{j,r}^{(t)}, \bxi_{i} \rangle\right)  \cdot \norm {\bxi_i}_2^{2} \cdot \mathds{1} (y_{i}=j) ,  \\
        && \urho^{(t+1)}= \urho^{(t)} +\frac{\eta}{nm} \ell_{i}^{\prime(t)} \cdot \sigma^\prime \left( \langle \bw_{j,r}^{(t)}, \bxi_{i} \rangle\right)  \cdot \norm {\bxi_i}_2^{2} \cdot \mathds{1} (y_{i}=-j) .
    \end{eqnarray*}
    
    We first prove \eqref{eqB3} hold at time $\tilde t$, denote $t_{j,r}$ as the last time $t \leq T^*$ such that $ \gam^{(t)}\leq 0.5\alpha$. Then we have 
    \begin{equation} \label{eqB9}
        \begin{aligned}
            \gam^{(\tilde t)} &= \gam^{(t_{j,r})} -  \frac{\eta}{nm}  \sum_{i=1}^n \ell_{i}^{\prime(t_{j,r})} \cdot \sigma^\prime \left( \langle \bw_{j,r}^{(t_{j,r})}, y_{i} \bmu \rangle\right) \cdot \norm \bmu_2^{2} \\
            &- \sum_{t_{j,r}<t< \tilde t}  \frac{\eta}{nm} \sum_{i=1}^n \ell_{i}^{\prime(t)} \cdot \sigma^\prime \left( \langle \bw_{j,r}^{(t)}, y_{i} \bmu \rangle\right) \cdot \norm \bmu_2^{2}.
        \end{aligned}
    \end{equation}
    The second term in \eqref{eqB9} can be bounded by
    \begin{equation} \label{eqB10}
        -  \frac{\eta}{nm}  \sum_{i=1}^n \ell_{i}^{\prime(t_{j,r})} \cdot \sigma^\prime \left( \langle \bw_{j,r}^{(t_{j,r})}, y_{i} \bmu \rangle\right) \cdot \norm \bmu_2^{2} \leq \frac{\eta}{m} \norm {\bmu}_2^{2} \leq 0.25 \alpha,
    \end{equation}
    where the first inequality is by $-\ell^{\prime (t)}_{i} < 1$ and the property of Huberized ReLU that $\sigma^\prime (z) \leq 1$,  the last inequality is by $\alpha\geq 1$ and the condition on $\eta$ specified in Condition \ref{condition}.
    
    Then we bound the third term in \eqref{eqB9}. For $t_{j,r} < t< \tilde t$ and any index $i$ satisfying $y_i=-j$, we have
    \[
    \langle \bw_{j,r}^{(t)}, y_{i} \bmu \rangle= \langle \bw_{j,r}^{(t)}, -j \bmu \rangle + \gam^{(t)} - \gam^{(t)} \leq \kappa - \gam^{(t)} \leq 0,
    \]
    where the first inequality is by Lemma \ref{innerprodectwjrxi}, the second inequality is by $\gam^{(t)} \geq 0.5 \alpha \geq \kappa$. Hence,
    \[
    \sigma^\prime \left(\langle \bw_{j,r}^{(t)}, y_{i} \bmu \rangle \right) =0.
    \]
    For $t_{j,r} < t< \tilde t$ and any index $i$ satisfying $y_i=j$, we have 
    \[
    \langle \bw_{j,r}^{(t)}, y_{i} \bmu \rangle= \langle \bw_{j,r}^{(t)}, j \bmu \rangle - \gam^{(t)} + \gam^{(t)}  \geq  \gam^{(t)} - \kappa \geq \kappa,
    \]
    where the first inequality is by Lemma \ref{innerprodectwjrxi}, the second inequality is by $\gam^{(t)} \geq 0.5 \alpha \geq 2\kappa$. Hence,
    \[
    \sigma^\prime \left(\langle \bw_{j,r}^{(t)}, y_{i} \bmu \rangle \right) =1.
    \]
    Moreover, notice that
    \begin{equation*}
    \begin{aligned}
        y_i f(\bW^{(t)},\bx_i) &= j F_{j}(\bW_{j}^{(t)},\bx_i) - j F_{-j}(\bW_{-j}^{(t)},\bx_i)  \\
        &\geq j F_{j}(\bW_{j}^{(t)},\bx_i) - \kappa \\
        &= \frac{1}{m} \sum_{r'=1}^m \left[ \sigma\left( \langle \bw_{j,r}^{(t)}, j \bmu \rangle\right)  +\sigma\left( \langle \bw_{j,r'}^{(t)}, \bm \bxi_i \rangle\right)  \right] - \kappa \\
        & \geq \frac{1}{m}  \sigma\left( \langle \bw_{j,r}^{(t)},  j \bmu \rangle - \gamma_{j,r}^{(t)} +\gamma_{j,r}^{(t)} \right)  -\kappa \\
        & \geq \frac{1}{m} \sigma\left(\gam^{(t)}- \kappa \right) - \kappa \\
        & \geq \frac{1}{m}  \gam^{(t)} - 2\kappa.
    \end{aligned}
    \end{equation*}
    where the first inequality is by Lemma \ref{Fminusyibound}, the second inequality is by $\sigma(t) \geq 0$, the third inequality is by Lemma \ref{innerprodectwjrxi} and the non-decreasing property of $\sigma$, and the last inequality is by $\sigma(t) \geq t$ for $t= \frac{1}{m} \gam^{(t)}- \kappa \geq 0.5 \frac{\alpha}{m} -\kappa \geq \kappa$. It follows that
    \begin{equation} \label{litupperbound}
        \begin{aligned}
            -\ell^{\prime (t)}_{i} &= -\ell^{\prime} \left(y_{i} f(\bW^{(t)},\bx_{i})\right) \leq \exp\left(-y_{i} f(\bW^{(t)},\bx_{i})\right) \\
            & \leq \exp\left(-\frac{1}{m}   \gam^{(t)} + 2\kappa \right) \\
            & \leq \exp\left(-\frac{\alpha}{2m} +\frac{\alpha}{4m}\right) \\
            &= \exp\left(-\frac{\alpha}{4m}\right),
        \end{aligned}
    \end{equation}
    where the first inequality is by $-\ell^{\prime}(t)= \frac{e^{-t}}{1+e^{-t}} \leq e^{-t}$, the second inequality is by $\gam^{(t)} \geq 0.5\alpha$, and the third inequality is by $\kappa \leq 1 \leq \frac{\alpha}{8m}$.
    Therefore,
    \begin{equation} \label{eqB11}
        \begin{aligned}
            & - \sum_{t_{j,r}<t< \tilde t}  \frac{\eta}{nm} \sum_{i=1}^n \ell_{i}^{\prime(t)} \cdot \sigma^\prime \left( \langle \bw_{j,r}^{(t)}, y_{i} \bmu \rangle\right) \cdot \norm \bmu_2^{2} \\
             \leq & \frac{\eta(\tilde t- t_{j,r}-1)}{m} \cdot \exp\left(-\frac{\alpha}{4m}\right)  \cdot  \norm \bmu_2^{2} \\
             \leq & \frac{\eta T^*}{m} \cdot \exp\left(-\log(T^*)\right)  \cdot  \norm \bmu_2^{2} \leq 0.25 \alpha,
        \end{aligned}
    \end{equation}
    where the first inequality is by \eqref{litupperbound}, the second inequality is by the definition of $\alpha$, and the last inequality is by the condition that $\eta=O(\frac{m\log(T^*)}{\norm{\bmu}_2})$ specified in Condition \ref{condition}. By utilizing the bounds \eqref{eqB10} and \eqref{eqB11} in \eqref{eqB9}, we have
    \begin{equation*}
        \gam^{(\tilde t)} \leq 0.5\alpha+ 0.25\alpha+ 0.25\alpha = \alpha.
    \end{equation*}
    Similarly, we can prove that $\urho^{(t)} \leq \alpha$ using the property that $\norm {\bxi_i}_2^2 \leq \frac{3\sigma_p^2d}{2}$ specified in Lemma \ref{innerproductxi}, and the condition that $\eta= O(\frac{n m^2 \log(T^*)}{\sigma_p^2 d})$ specified in Condition \ref{condition}.

    Finally, we prove \eqref{eqB4} hold at time $t=\tilde t$. We only need to consider the case that $j=-y_i$, since $\urho^{(t)}= 0$ for $j=y_i$. By the induction hypothesis, we have $\urho^{(\tilde t-1)} \geq -2\beta - 2\zeta$. Then we consider the two cases. For the first case, if $\urho^{(\tilde t-1)} < -\beta - \zeta$, according to Lemma \ref{innerprodectwjrxi}, we have
    \begin{equation*}
        \la \wjr^{(\tilde t-1)}, \bxi_i \ra \leq  \urho^{(\tilde t-1)} +\beta + \zeta < 0,
    \end{equation*}
    then by the update rule of $\urho^{(\tilde t)}$, we have
    \begin{align*}
        \urho^{(\tilde t)} &= \urho^{(\tilde t-1)} +\frac{\eta}{nm} \ell_{i}^{\prime(\tilde t-1)} \cdot \sigma^\prime \left( \langle \bw_{j,r}^{(\tilde t-1)}, \bxi_{i} \rangle\right)  \cdot \norm {\bxi_i}_2^{2} \cdot \mathds{1} ( y_{i}=-j) \\
        &= \urho^{(\tilde t-1)} \geq -2\beta -2\zeta .
    \end{align*}
    For the second case, if $\urho^{(\tilde t-1)} \geq -\beta - \zeta$, we have
    \begin{align*}
        \urho^{(\tilde t)} &= \urho^{(\tilde t-1)} +\frac{\eta}{nm} \ell_{i}^{\prime(\tilde t-1)} \cdot \sigma^\prime \left( \langle \bw_{j,r}^{(\tilde t-1)}, \bxi_{i} \rangle\right)  \cdot \norm {\bxi_i}_2^{2} \cdot \mathds{1} (y_{i}=-j) \\
        &\geq -\beta - \zeta - \frac{3\kappa^{1-q} \eta \sigma_p^2 d}{2nm} (\beta + \zeta)^{q-1} \\
        & \geq -\beta - \zeta - \beta - \zeta = -2\beta - 2\zeta,
    \end{align*}
    where the first inequality is by $|\ell_{i}^{\prime(\tilde t-1)}| < 1$, $\langle \bw_{j,r}^{(\tilde t-1)}, \bxi_{i} \rangle \leq \beta+ \zeta \leq \kappa$, and Lemma \ref{innerproductxi}, the second inequality is by the condition on $\eta$ stated in Condition \ref{condition}. Combining these two cases, we have $\urho^{(\tilde t)} \geq -2\beta - 2\zeta \geq - \alpha$.
\end{proof}

\subsection{First Stage}
\begin{lemma} \label{GDFirstStage}
    Under Condition \ref{condition}, we can find a time $T_1= \tilde \Theta \left( \frac{\kappa^{q-1} mn}{\eta \sigma_0^{q-2} (\sigma_p \sqrt{d})^q} \right)$, such that
    \begin{itemize}
        \item $\max_{r} \langle\bw_{y_i,r}^{(T_1)}, \bxi_i \rangle \geq \kappa$, $\max_{j,r} \orho^{(T_1)} \geq \kappa$, for all $i \in [n]$.
        \item $\max_{j,r} \langle\bw_{j,r}^{(t)}, j \mu \rangle = \tilde O (\sigma_0 \|\mu\|_2)$, $\max_{j,r} \gam^{(T_1)} = \tilde O (\sigma_0 \|\mu\|_2)$,  for all $0 \leq t \leq T_1$.
        \item $\max_{r,i} \left|\langle\bw_{-y_i,r}^{(t)}, \bxi_i \rangle\right| = \tilde O (\sigma_0 \sigma_p \sqrt{d})$, $\max_{j,r,i} |\urho^{(T_1)}|  = \tilde O (\sigma_0 \sigma_p \sqrt{d})$, for all $0 \leq t \leq T_1$.
    \end{itemize}
\end{lemma}

\begin{proof} [Proof of Lemma \ref{GDFirstStage}]
    We first prove the second result. Denote $\tilde \beta= \sigma_0 \| \bmu \|_2 \cdot \sqrt{2\log \left(\frac{12m}{\delta}\right)}$. According to the iteration formulation \eqref{GDupdate}, we have
    \[
    \begin{aligned}
        \langle\bw_{j,r}^{(t+1)}, j\bmu\rangle &= \langle \bw_{j,r}^{(t)}, j\bmu\rangle- \frac{\eta}{nm} \sum_{i=1}^n \ell_{i}^{\prime(t)} \cdot \sigma^\prime \left( \langle \bw_{j,r}^{(t)}, y_{i} \bmu \rangle\right) \cdot \|\bmu\|_2^2 \\
        & \leq \langle \bw_{j,r}^{(t)}, j\bmu\rangle + \frac{\eta}{nm} \sum_{i=1}^n  \sigma^\prime \left( \max \left\{ \langle \bw_{j,r}^{(t)}, j \bmu \rangle, \langle \bw_{j,r}^{(t)}, -j \bmu \rangle \right\} \right) \cdot \|\bmu\|_2^2 \\
        & \leq \langle \bw_{j,r}^{(t)}, j\bmu\rangle + \frac{\eta}{nm} \sum_{i=1}^n  \sigma^\prime \left( \langle \bw_{j,r}^{(t)}, j \bmu \rangle + \tilde \beta  \right) \cdot \|\bmu\|_2^2
    \end{aligned}
    \]
    where the first inequality is by $-\ell_{i}^{\prime(t)} \leq 1$, the second inequality utilizes the property $\langle \bw_{j,r}^{(t)}, -j\bmu\rangle \leq \langle \bw_{j,r}^{(0)}, -j\bmu\rangle \leq \tilde \beta$ by Lemma \ref{innerproductw0}. 
    
    Denote $A^{(t)}= \max_{j,r} \left\{\langle \bw_{j,r}^{(t)}, j \bmu \rangle + \tilde \beta  \right\}$. We use induction to prove that $A^{(t)} = \tilde O (\sigma_0 \|\mu\|_2)$ holds for $0 \leq t \leq T_1$. Clearly, it holds for $t=0$ by Lemma \ref{innerproductw0}. Suppose that it holds for $0 \leq t \leq \tilde T-1$ with $\tilde T \leq T_1$. Since $A^{(t)} = \tilde O (\sigma_0 \|\mu\|_2) \leq \kappa$, we can rewrite the above iteration as
    \[
    A^{(t+1)} \leq A^{(t)} + \frac{\eta \kappa^{1-q} \| \bmu \|_2^2}{m} \left[A^{(t)}\right]^{q-1}.
    \]
     by taking a telescoping sum we have
    \begin{align*}
        A^{(\tilde{T})} &\leq A^{(0)} + \sum_{s=0}^{\tilde{T}} \frac{\eta \kappa^{1-q} \|\bmu\|_{2}^{2}}{m} \left[A^{(s)}\right]^{q-1}\\
        &\leq  A^{(0)} + \tilde O \left( \frac{\eta \kappa^{1-q}  \sigma_0^{q-1} \| \bmu \|_2^{q+1} T_{1}}{m} \right) \\ 
        & = \tilde O (\sigma_0 \|\mu\|_2) + \tilde O\left( \frac{\sigma_0 \|\mu\|_2 n \|\mu\|_2^q}{(\sigma_p \sqrt{d})^q}\right) \\
        & = \tilde O (\sigma_0 \|\mu\|_2),
    \end{align*}
    where the first equation is by the choice of $T_1= \tilde \Theta \left( \frac{\kappa^{q-1} mn}{\eta \sigma_0^{q-2} (\sigma_p \sqrt{d})^q} \right)$ and the second equation is by the SNR condition $\frac{\sigma_p^q d^{q/2}}{n \|\mu\|_2^q} \geq \tilde \Omega(1)$ specified in Condition \ref{condition}.
    
    Next, we prove the first result by contradiction. We consider the case $j=y_i$. Denote $T_{1,i}$ be the the last time in the period $0 < t \leq T_{1,i}$ satisfying that $\max_{j,r} \langle\bw_{j,r}^{(t)}, \bxi_i \rangle \leq \kappa$. Moreover, for $0 \leq t \leq T_{1,i}$, notice that $\la \bw_{y_i,r}^{(t)}, y_i \cdot \bmu \ra = \tilde O(\sigma_0 \|\mu\|_2)=O(1) $ and $\la \bw_{y_i,r}^{(t)},  \bxi_i \ra \leq \kappa = O(1)$,
    %\begin{align*}
    %    F_{-y_i} (\bW_{-y_i}^{(t)},\bx_i) &= \frac{1}{m} \sum_{r=1}^m \left[\sigma\left(\la \bw_{-y_i,r}^{(t)}, y_i \cdot \bmu \ra\right) + \sigma\left(\la \bw_{-y_i,r}^{(t)},  \bxi_i \ra\right)\right] \\
    %     &= \frac{\kappa^{1-q}}{q} \left[ \tilde O \left((\sigma_0 \|\mu\|_2)^q \right) + (\sigma_0 \sigma_p \sqrt{d})^q \right] \leq 1.
    %\end{align*}
    therefore,
    \begin{align*}
        -\ell_i^{'(t)} &= \frac{1}{1+e^{F_{y_i}(\bW_{y_i}^{(t)},\bx_i)- F_{-y_i} (\bW_{-y_i}^{(t)},\bx_i)}}  \geq \frac{1}{2e^{F_{y_i}(\bW_{y_i}^{(t)},\bx_i)}} \\
        & = \frac{1}{2} \exp\left(-\frac{1}{m} \sum_{r=1}^m \left[\sigma\left(\la \bw_{y_i,r}^{(t)}, y_i \cdot \bmu \ra\right) + \sigma\left(\la \bw_{y_i,r}^{(t)},  \bxi_i \ra\right)\right]\right) \geq C_{\ell},
    \end{align*}
    for some positive constant $C_\ell$. Recall that
    \begin{align*}
        \orho^{(t+1)}= \orho^{(t)} -\frac{\eta}{nm} \ell_{i}^{\prime(t)} \cdot \sigma^\prime \left( \langle \bw_{j,r}^{(t)}, \bxi_{i} \rangle\right)  \cdot \norm {\bxi_i}_2^{2}.
    \end{align*}
    Moreover, by Lemma \ref{innerprodectwjrxi}, we have
    \begin{align*}
        \kappa \geq \langle \bw_{j,r}^{(t)}, \bxi_{i} \rangle \geq \orho^{(t)}+ \langle \bw_{j,r}^{(0)}, \bxi_{i} \rangle - \zeta.
    \end{align*}
    By denoting $B_{i}^{(t)} = \max_{j,r}\left\{\orho^{(t)} + \la \bW_{j,r}^{(0)}, \bxi_{i}\ra - \zeta \right\}$. Notice that,
    $$B_{i}^{(0)} = \max_{j,r}\left\{\la \bW_{j,r}^{(0)}, \bxi_{i}\ra - \zeta \right\} \geq \frac{\sigma_0 \sigma_p \sqrt{d}}{4} - \zeta \geq \frac{\sigma_0 \sigma_p \sqrt{d}}{8},$$
    where the first inequality is by Lemma \ref{innerproductw0}, the second inequality is by the condition that $\sigma_0 \geq \frac{Cn}{\sigma_p d} \sqrt{\log\left(\frac{n^2}{\delta}\right)} \log(T^*)$ specified in Condition \ref{condition}.
    Therefore, we have the following inequality on the growth of $B_{i}^{(t)}$,
    \begin{align*}
    B_{i}^{(t+1)} &\geq B_{i}^{(t)} + \frac{C_{\ell} \kappa^{1-q} \eta \sigma_{p}^{2}d}{2nm}\left[B_{i}^{(t)}\right]^{q-1},
    \end{align*}
    by utilizing Lemma \ref{tensorpower} with $C_1= \frac{C_{\ell} \kappa^{1-q} \sigma_{p}^{2}d}{2nm}$, we conclude that $B_i^{(t)}$ will reach $2$ at the time $T_{1,i}= \Theta \left( \frac{\kappa^{q-1} mn}{\eta \sigma_0^{q-2} (\sigma_p \sqrt{d})^q} \right)$.
     It follows that
     $$\max_{j, r} \langle\bw_{j,r}^{(t)}, \bxi_i \rangle \geq \max_{j, r}\orho^{(t)}- \kappa \geq B_{i}^{(t)} - \kappa - \beta - \zeta \geq 2- 2\kappa \geq \kappa,$$ 
     where the first inequality is by Lemma \ref{innerprodectwjrxi}, the second inequality is by the definition of $\beta$. This contradicts with the assumption that $\max_{j, r} \langle\bw_{j,r}^{(T_{1,i})}, \bxi_i \rangle \leq \kappa$. Notice that $T_{1}= \tilde \Theta \left( \frac{\kappa^{q-1} mn}{\eta \sigma_0^{q-2} (\sigma_p \sqrt{d})^q} \right)$, we conclude that $\max_{j, r} \langle\bw_{j,r}^{(T_1)}, \bxi_i \rangle \geq \kappa$ for all $i\in [n]$.

    The last result follows from the property that
    \begin{align*}
        \left|\urho^{(t)} \right| \leq 2\beta + 2\zeta \leq \tilde O (\sigma_0 \sigma_p \sqrt{d})+ 2\zeta = \tilde O (\sigma_0 \sigma_p \sqrt{d}),
    \end{align*}
    where the first inequality is by Proposition \ref{proposition:GDgrowthbound}, the second inequality is by Lemma \ref{innerproductw0}, and the last inequality is by the condition that $\sigma_0 \geq \frac{C nm}{\sigma_p d} \sqrt{\log\left(\frac{n^2}{\delta}\right)} \log(T^*)$ specified in Condition \ref{condition}. Thus we complete the proof.
\end{proof}

\subsection{Second Stage}
Note that Lemma \ref{GDFirstStage} indicates that $\max_{r} \langle\bw_{y_i,r}^{(t)}, \bxi_i \rangle \geq \kappa$ for all $t\geq T_{1}$. Now we choose $\bW^{*}$ as follows
\begin{align*}
\bw^{*}_{j,r} = \bw_{j,r}^{(0)} + 4qm \log\frac{2q}{\epsilon}\bigg[\sum_{ i = 1}^n \frac{\bxi_{i}}{\|\bxi_{i}\|_{2}} \cdot \indicator(j = y_{i}) \bigg].
\end{align*}
Based on the definition of $\bW^{*}$, we have the following lemma.
\begin{lemma}\label{lm:distance2}
Under Condition \ref{condition}, we have that $\|\bW^{(T_{1})} - \bW^{*}\|_{F} =  \tilde O \left(\frac{m^{\frac{3}{2}} n^{\frac{1}{2}}}{\|\bmu\|_{2}} \right)$.
\end{lemma}
\begin{proof}[Proof of Lemma~\ref{lm:distance2}] 
By the signal-noise decompositon, at the end of the first stage, we have
\begin{align*}
\bw_{j,r}^{(T_1)} = \bw_{j,r}^{(0)} + \gamma_{j,r}^{(T_1)} \cdot \| \bmu \|_2^{-2} \cdot j\bmu + \sum_{ i = 1}^n \orho^{(T_1) }\cdot \| \bxi_i \|_2^{-2} \cdot \bxi_{i} + \sum_{ i = 1}^n \urho^{(T_1) }\cdot \| \bxi_i \|_2^{-2} \cdot \bxi_{i},
\end{align*}
it follows that
\begin{align*}
\|\bW^{(T_{1})} - \bW^{*}\|_{F} &\leq \|\bW^{(T_{1})} - \bW^{(0)}\|_{F} + \|\bW^{(0)} - \bW^{*}\|_{F} \\
&\leq \sum_{j,r}\gamma_{j,r}^{(T_{1})}\|\bmu\|_{2}^{-1} + O(\sqrt{m})\max_{j,r}\bigg\|\sum_{ i = 1}^n \orho^{(T_{1})} \cdot \frac{\bxi_{i}}{\|\bxi_{i}\|_{2}^{2}} + \sum_{ i = 1}^n \urho^{(T_{1})} \cdot \frac{\bxi_{i}}{\|\bxi_{i}\|_{2}^{2}}\bigg\|_{2}\\
&\qquad + O \left(\frac{m^{\frac{1}{2}} n^{\frac{1}{2}}}{\sigma_p \sqrt{d}} \log\frac{1}{\epsilon} \right)\\
&= \tilde O \left(\frac{m^{\frac{3}{2}}}{\|\bmu\|_{2}} \right) + \tilde O \left(\frac{m^{\frac{3}{2}} n^{\frac{1}{2}}}{\sigma_p \sqrt{d}} \right) + \tilde{O}\left( \frac{m^{\frac{1}{2}} n^{\frac{1}{2}}}{\sigma_p \sqrt{d}}\right) = \tilde O \left(\frac{m^{\frac{3}{2}} n^{\frac{1}{2}}}{\|\bmu\|_{2}} \right),
\end{align*}
where the first inequality is by triangle inequality, the second inequality is by our decomposition of $\bW^{(T_{1})}, \bW^{*}$ and  Lemma \ref{innerproductxi}, the first equality is by Proposition~\ref{proposition:GDgrowthbound}, the last equality is by the SNR condition $\frac{\sigma_p^q d^{q/2}}{n \|\mu\|_2^q} \geq \tilde \Omega(1)$ specified in Condition \ref{condition}.
\end{proof}

\begin{lemma}\label{lm: Gradient Stable2}
Under Condition \ref{condition}, we have that
\begin{align*}
y_{i}\la \nabla f(\bW^{(t)}, \bx_{i}), \bW^{*} \ra \geq q\log \frac{2q}{\epsilon}  
\end{align*}
for all $ T_{1} \leq t\leq T^{*}$.
\end{lemma}
\begin{proof}[Proof of Lemma~\ref{lm: Gradient Stable2}]
Recall that 
$$f(\bW^{(t)}, \bx_{i}) = \frac{1}{m} {\sum_{j,r}}j\cdot \big[\sigma(\la\bw_{j,r},y_{i}\cdot\bmu\ra) + \sigma(\la\bw_{j,r}, \bxi_{i}\ra)\big]. $$ 
Therefore, we have 
\begin{align*}
&y_{i}\la \nabla f(\bW^{(t)}, \bx_{i}), \bW^{*} \ra \notag\\
&= \frac{1}{m}\sum_{j,r}\sigma'(\la \bw_{j,r}^{(t)}, y_{i}\bmu\ra)\la \bmu,  j\bw_{j,r}^{*}\ra + \frac{1}{m}\sum_{j,r}\sigma'(\la \bw_{j,r}^{(t)}, \bxi_{i}\ra)\la y_{i}\bxi_{i},  j\bw_{j,r}^{*}\ra \notag\\
&= \frac{1}{m}\sum_{j,r}\sum_{i'=1}^{n}\sigma'(\la \bw_{j,r}^{(t)}, \bxi_{i}\ra) 4qm\log \frac{2q}{\epsilon} \cdot \frac{\la\bxi_{i'}, \bxi_{i} \ra}{\|\bxi_{i'}\|_{2}} \cdot  \indicator(j = y_{i'}) \notag\\
&\qquad+ \frac{1}{m}\sum_{j,r}\sigma'(\la \bw_{j,r}^{(t)}, y_{i}\bmu\ra)\la \bmu,  j\bw_{j,r}^{(0)}\ra  + \frac{1}{m}\sum_{j,r}\sigma'(\la \bw_{j,r}^{(t)}, \bxi_{i}\ra)\la y_{i}\bxi_{i},  j\bw_{j,r}^{(0)}\ra \notag\\
&\geq \frac{1}{m}\sum_{j=y_{i},r}\sigma'(\la \bw_{j,r}^{(t)}, \bxi_{i}\ra)4qm \log \frac{2q}{\epsilon}  - \frac{1}{m}\sum_{j,r}\sum_{i'\not= i}\sigma'(\la \bw_{j,r}^{(t)}, \bxi_{i}\ra)4qm\log \frac{2q}{\epsilon} \cdot \frac{|\la\bxi_{i'}, \bxi_{i} \ra|}{\|\bxi_{i'}\|_{2}} \notag \\
&\qquad -  \frac{1}{m}\sum_{j,r}\sigma'(\la \bw_{j,r}^{(t)}, y_{i}\bmu\ra) \left|\la \bmu,  j\bw_{j,r}^{(0)}\ra\right|  - \frac{1}{m}\sum_{j,r}\sigma'(\la \bw_{j,r}^{(t)}, \bxi_{i}\ra) \left| \la y_{i}\bxi_{i},  j\bw_{j,r}^{(0)}\ra \right| \notag\\
&\geq \frac{1}{m}\sum_{j=y_{i},r}\sigma'(\la \bw_{j,r}^{(t)}, \bxi_{i}\ra) 4qm\log \frac{2q}{\epsilon} - \frac{1}{m}\sum_{j,r}\sigma'(\la \bw_{j,r}^{(t)}, \bxi_{i}\ra)\tilde{O}\left(\frac{mn}{\sqrt{d}}\right) \notag \\
&\qquad -  \frac{1}{m}\sum_{j,r}\sigma'(\la \bw_{j,r}^{(t)}, y_{i}\bmu\ra)\tilde{O}(\sigma_{0}\|\bmu\|_{2})  - \frac{1}{m}\sum_{j,r}\sigma'(\la \bw_{j,r}^{(t)}, \bxi_{i}\ra)\tilde{O}(\sigma_{0}\sigma_{p}\sqrt{d}) \\
&\geq 4q\log \frac{2q}{\epsilon}  - \tilde{O}\left(\frac{mn}{\sqrt{d}}\right)   -\tilde{O}(\sigma_{0}\|\bmu\|_{2}) - \tilde{O}(\sigma_{0}\sigma_{p}\sqrt{d}) \geq \log q\frac{2q}{\epsilon},
\end{align*}
where the second inequality is by Lemma~\ref{innerproductxi} and Lemma \ref{innerproductw0}, the third inequality is by the first property in Lemma \ref{GDFirstStage} and the property that $\sigma^\prime (z) \leq 1$,  the last inequality is by $d\geq \tilde{\Omega}(m^{2}n^{2})$ and $\sigma_{0} \leq O(\min\{(\sigma_{p}\sqrt{d})^{-1}, \|\bmu\|_{2}^{-1}\})$ specified in Condition~\ref{condition}.
\end{proof}

\begin{lemma}\label{lm: gradient upbound}
Under Condition~\ref{condition}, for $0 \leq t\leq T^{*}$, the following result holds.
\begin{align*}
\|\nabla L_{S}(\bW^{(t)})\|_{F}^{2} \leq  O(\max\{\|\bmu\|_{2}^{2}, \sigma_{p}^{2}d\}) L_{S}(\bW^{(t)}).   
\end{align*}
\end{lemma}
\begin{proof}[Proof of Lemma~\ref{lm: gradient upbound}]
Note that
\begin{align} \label{eqB14}
\|\nabla f(\bW^{(t)}, \bx_{i}) \|_{F}&\leq \frac{1}{m}\sum_{j,r}\bigg\|  \big[\sigma'(\la\bw_{j,r}^{(t)}, y_i\bmu\ra)y_i \bmu + \sigma'(\la\bw_{j,r}^{(t)}, \bxi_i\ra)\bxi_i\big] \bigg\|_{2} \notag\\
&\leq \frac{1}{m}\sum_{j,r}\sigma'(\la\bw_{j,r}^{(t)}, y_i\bmu\ra)\|\bmu\|_{2} +   \frac{1}{m}\sum_{j,r}\sigma'(\la\bw_{j,r}^{(t)}, \bxi_{i}\ra)\|\bxi_{i}\|_{2} \notag \\
& = O\left(\max\{\|\bmu\|_{2}, \sigma_{p}\sqrt{d}\} \right),
\end{align}
where the first and second inequalities are by triangle inequality, the last inequality is by Lemma~\ref{innerproductxi} and the property that $\sigma'(z) \leq 1$.

Then we can upper bound the gradient norm $\|\nabla L_{S}(\bW^{(t)})\|_{F}$ as follows,
\begin{align*}
\|\nabla L_{S}(\bW^{(t)})\|_{F}^{2} &\leq \bigg[\frac{1}{n}\sum_{i=1}^{n}\ell'\big(y_{i}f(\bW^{(t)},\bx_{i})\big) \|\nabla f(\bW^{(t)}, \bx_{i})\|_{F}\bigg]^{2}\\
&\leq \bigg[\frac{1}{n}\sum_{i=1}^{n}O\left(\max\{\|\bmu\|_{2}, \sigma_{p} \sqrt{d}\}\right) \cdot \left|\ell'\big(y_{i}f(\bW^{(t)},\bx_{i})\big)\right| \bigg]^{2}\\
&\leq O(\max\{\|\bmu\|_{2}^{2}, \sigma_{p}^{2}d\}) \cdot \frac{1}{n}\sum_{i=1}^{n} \left|\ell'\big(y_{i}f(\bW^{(t)},\bx_{i})\big) \right|\\
&\leq O(\max\{\|\bmu\|_{2}^{2}, \sigma_{p}^{2}d\}) L_{S}(\bW^{(t)}),
\end{align*}
where the first inequality is by triangle inequality, the second inequality is by \eqref{eqB14}, the third inequality is by Cauchy-Schwartz inequality and the property that $|\ell'(z)| \leq 1$, and the last inequality is due to the property of the cross entropy loss $|\ell'(z)| \leq \ell(z)$.
\end{proof}

\begin{lemma}\label{lemma:noise_stage2_homogeneity}
Under Condition \ref{condition}, we have that  
\begin{align*}
\|\bW^{(t)} - \bW^{*}\|_{F}^{2} - \|\bW^{(t+1)} - \bW^{*}\|_{F}^{2} \geq (2q-1)\eta L_{S}(\bW^{(t)}) - \eta\epsilon
\end{align*}
for all $ T_{1} \leq t \leq T^{*}$.
\end{lemma}

\begin{proof}[Proof of Lemma~\ref{lemma:noise_stage2_homogeneity}]
Note that $\sigma$ has the following property. When $z \in [0,\kappa]$, we have $\sigma'(z) z= q \sigma(z)$; when $z \geq \kappa$, we have $\sigma'(z) z= z = \sigma(z) - \frac{\kappa}{q} + \kappa \leq q \sigma(z)$. It follows that
\begin{align*}
&\|\bW^{(t)} - \bW^{*}\|_{F}^{2} - \|\bW^{(t+1)} - \bW^{*}\|_{F}^{2}\\
 =&  2\eta \la \nabla L_{S}(\bW^{(t)}), \bW^{(t)} - \bW^{*}\ra - \eta^{2}\|\nabla L_{S}(\bW^{(t)})\|_{F}^{2}\\
 =& \frac{2\eta}{n}\sum_{i=1}^{n}\ell'^{(t)}_{i} \left[ y_{i} \la \nabla f(\bW^{(t)}, \bx_{i}), \bW^{(t)} \ra - y_i \la \nabla f(\bW^{(t)}, \bx_{i}), \bW^{*} \ra \right] - \eta^{2}\|\nabla L_{S}(\bW^{(t)})\|_{F}^{2}\\
\geq&  \frac{2\eta}{n}\sum_{i=1}^{n}\ell'^{(t)}_{i} \left[qy_{i}f(\bW^{(t)}, \bx_{i}) - q \log \frac{2q}{\epsilon} \right] - \eta^{2}\|\nabla L_{S}(\bW^{(t)})\|_{F}^{2}\\
\geq & \frac{2q\eta}{n}\sum_{i=1}^{n} \left[\ell\big(y_{i}f(\bW^{(t)},\bx_{i})\big) - \frac{\epsilon}{2q} \right] - \eta^{2}\|\nabla L_{S}(\bW^{(t)})\|_{F}^{2}\\
\geq & (2q-1)\eta L_{S}(\bW^{(t)}) - \eta\epsilon,
\end{align*}
where the first inequality is by Lemma~\ref{lm: Gradient Stable2} and the property that $\sigma^\prime(z) \cdot z \geq q \sigma(z)$, the second inequality is due to the convexity of the cross entropy function, the last inequality is due to Lemma~\ref{lm: gradient upbound} and the condition that $\eta= O(\frac{1}{(\max\{\|\bmu\|_{2}^{2}, \sigma_{p}^{2}d\}})$ specified in Condition \ref{condition}.
\end{proof}

% [Restatement of Lemma~\ref{lemma:noise_proof_sketch}]
\begin{lemma}\label{thm:noise_proof}
Under Condition \ref{condition}, let $T = T_{1} + \Big\lfloor \frac{\|\bW^{(T_{1})} - \bW^{*}\|_{F}^{2}}{2\eta \epsilon}
\Big\rfloor  = T_{1} + \tilde{O}\left(\frac{m^{3}n}{\eta \epsilon \norm{\bmu}_2^2} \right)$. Then we have
\begin{itemize}
    \item $\max_{j,r} \la\bw_{j,r}^{(t)}, j \bmu\ra = \tilde{O} (\sigma_0 \norm{\bmu}_2)$, $\max_{j,r}\gamma_{j,r}^{(t)} = \tilde{O} (\sigma_0 \norm{\bmu}_2)$, for all $T_{1} \leq t \leq T$.
    \item $\max_{j,r,i}|\urho^{(t)}| = \tilde{O}(\sigma_{0}\sigma_{p}\sqrt{d})$, for all $T_{1} \leq t \leq T$.
\end{itemize}
Besides, for all $T_{1} \leq t \leq T$, we have
\begin{align*}
\frac{1}{t - T_{1} + 1}\sum_{s=T_{1}}^{t}L_{S}(\bW^{(s)}) \leq \frac{\|\bW^{(T_{1})} - \bW^{*}\|_{F}^{2}}{(2q-1) \eta(t - T_{1} + 1)} + \frac{\epsilon}{(2q-1)}.
\end{align*}
Therefore, we can find an iterate with training loss smaller than $\epsilon$ within $T$ iterations.
\end{lemma}

\begin{proof}[Proof of Lemma~\ref{thm:noise_proof}]
By Lemma~\ref{lemma:noise_stage2_homogeneity}, for any $T_{1} \leq t \leq T$, we obtain that  
\begin{align}
\|\bW^{(s)} - \bW^{*}\|_{F}^{2} - \|\bW^{(s+1)} - \bW^{*}\|_{F}^{2} \geq (2q-1)\eta L_{S}(\bW^{(s)}) - \eta\epsilon \label{eq: bounddistance}
\end{align}
holds for $T_1 \leq s \leq t$. Taking a summation, we have that 
\begin{align}
\sum_{s=T_{1}}^{t}L_{S}(\bW^{(s)}) &\leq  \frac{\|\bW^{(T_{1})} - \bW^{*}\|_{F}^{2}}{(2q-1) \eta} + \frac{\epsilon (t - T_{1} + 1)}{2q-1} \notag\\
&\leq \frac{2\|\bW^{(T_{1})} - \bW^{*}\|_{F}^{2} }{(2q-1) \eta} \notag\\
&= \tilde{O}\left(\frac{m^{3}n}{\eta \norm{\bmu}_2^2} \right)\label{eq: sum2},
\end{align}
where the second inequality is by $t\leq T$ and the definition of $T$, the equality is by Lemma~\ref{lm:distance2} .

Then we can use induction to prove that $\max_{j,r} \la\bw_{j,r}^{(t)}, j\bmu\ra = \tilde{O} (\sigma_0 \norm{\bmu}_2)$ for all $t \in [T_1,T]$. By the second property in Lemma \ref{GDFirstStage}, it holds for $t=T_1$. Suppose that it holds for $t \in [T_1,\tilde T -1]$,  by the update rule of GD \eqref{GDupdate}, for any $r\in [m]$, we have
% the results hold for time $t < t'$,
\begin{align*}
    \la\bw_{j,r}^{(\tilde{T})}, j \bmu\ra &= \la\bw_{j,r}^{(T_{1})}, j \bmu\ra - \frac{\eta}{nm} \sum_{s=T_{1}}^{\tilde{T} - 1}\sum_{i=1}^n \ell_i'^{(t)} \cdot \sigma'(\la\bw_{j,r}^{(t)}, j \bmu\ra) \cdot \|\bmu\|_{2}^{2},\\
    &\leq \tilde{O} (\sigma_0 \norm{\bmu}_2)  + \tilde{O} \left( \frac{\kappa^{1-q}\eta \sigma_0^{q-1} \|\bmu\|_{2}^{q+1}}{nm} \right) \sum_{s=T_{1}}^{\tilde{T} - 1}\sum_{i=1}^n |\ell_i'^{(t)}|\\
    &\leq \tilde{O} (\sigma_0 \norm{\bmu}_2)  + \tilde{O} \left( \frac{\kappa^{1-q}\eta \sigma_0^{q-1} \|\bmu\|_{2}^{q+1}}{m} \right) \sum_{s=T_{1}}^{\tilde{T} - 1}L_{S}(\bW^{(s)})\\
    &\leq \tilde{O} (\sigma_0 \norm{\bmu}_2)  + \tilde{O} \left( \kappa^{1-q} m^2 n \sigma_0^{q-1} \|\bmu\|_{2}^{q-1} \right) \\
    &= \tilde{O} (\sigma_0 \norm{\bmu}_2)
\end{align*}
where the first inequality is by the induction hypothesis that $\max_{j,r} \la\bw_{j,r}^{(t)}, j \bmu\ra = \tilde{O} (\sigma_0 \norm{\bmu}_2)$ for all $t \in [T_1,T]$, the second inequlity is by $|\ell'(z)| \leq \ell(z)$, the third inequality is by \eqref{eq: sum2}, the last equality is by the condition that $\sigma_0= O \left( \frac{\kappa^{\frac{q-1}{q-2}}}{m^{\frac{2}{q-2}} n^{\frac{1}{q-2}} \norm{\bmu}_2} \right)$ specified in Condition~\ref{condition}. Therefore, we complete the induction proof.
\end{proof}

\subsection{Generalization Error Analysis}

We first provide a lower bound on the test loss.
\begin{lemma}\label{lemma: noise generalization}
Under Condition \ref{condition}, within $T^*=\tilde{O}\left( \frac{\kappa^{q-1} mn}{\eta \sigma_0^{q-2} (\sigma_p \sqrt{d})^q} + \frac{m^{3}n}{\eta \epsilon \norm{\bmu}_2^2} \right)$ iterations, we can find a time $t$ such that $L_{S}(\bW^{(t)}) \leq \epsilon$. Besides, for any $0 \leq t\leq T$ we have that $L_{\cD}(\bW^{(t)}) \geq 0.1$.
\end{lemma}

\begin{proof}[Proof of Lemma~\ref{lemma: noise generalization}]
Note that for $j \in \{+1, -1\}$, $r \in [m]$,
\begin{align*}
\|\bw_{j,r}^{(t)}\|_{2} &= \bigg\|\bw_{j,r}^{(0)} + j \cdot \gam^{(t)} \cdot \frac{\bmu}{\|\bmu\|_{2}^{2}} + \sum_{ i = 1}^n \orho^{(t)} \cdot \frac{\bxi_{i}}{\|\bxi_{i}\|_{2}^{2}} + \sum_{ i = 1}^n \urho^{(t)} \cdot \frac{\bxi_{i}}{\|\bxi_{i}\|_{2}^{2}}\bigg\|_{2}\\
& \leq \|\bw_{j,r}^{(0)}\|_{2} +  \frac{\gamma_{j,r}^{(t)}}{\|\bmu\|_{2}} + \sum_{ i = 1}^n  \frac{\orho^{(t)}}{\|\bxi_{i}\|_{2}} + \sum_{ i = 1}^n  \frac{|\urho^{(t)}|}{\|\bxi_{i}\|_{2}}\\
& = O(\sigma_{0}\sqrt{d})  + \tilde{O} \left( \frac{mn}{\sigma_p \sqrt{d}} \right),
\end{align*}
where the first inequality is by triangle inequality, the last equality is by Lemma \ref{innerproductxi}, Lemma \ref{innerproductw0}, $\max_{j,r}\gamma_{j,r}^{(t)} = \tilde{O}(\sigma_{0}\|\bmu\|_{2})$ in Lemma~\ref{thm:noise_proof}, and $\max_{j,r,i}|\rho_{j,r,i}| \leq 4m\log(T^{*})$ in Proposition~\ref{proposition:GDgrowthbound}.

Given a new example $(\bx,y)$, we have that $\la\bw_{j,r}^{(t)}, \bxi\ra \sim \cN(0, \sigma_{p}^{2}\|\bw_{j,r}^{(t)}\|_{2}^{2})$. Therefore, with probability at least $1 - \frac{1}{4m}$, we have
\begin{align} \label{eqB17}
|\la\bw_{j,r}^{(t)}, \bxi\ra| \leq \tilde{O} \left( \sigma_{0}\sigma_{p}\sqrt{d} + \frac{mn}{ \sqrt{d}} \right).
\end{align}
By union bound, with probability at least $1-0.5$, we have that  
\begin{align*}
F_{y}(\bW_{y}^{(t)},\bx) &= \frac{1}{m}\sum_{r=1}^{m}\sigma(\la \bw_{y,r}^{(t)}, y\bmu \ra)
+ \frac{1}{m}\sum_{r=1}^{m}\sigma(\la \bw_{y,r}^{(t)}, \bxi \ra)\\
&\leq \tilde{O} \left(\kappa^{1-q} (\sigma_{0}\|\bmu\|_{2})^{q} \right) + \tilde{O}\left( \kappa^{1-q} (\sigma_{0}\sigma_{p}\sqrt{d})^{q}\right) + \tilde{O} \left( \kappa^{1-q} \left(\frac{mn}{ \sqrt{d}} \right)^{q} \right) \\
&\leq 1,
\end{align*}
where the first inequality is by $\max_{j,r} \la\bw_{j,r}^{(t)}, j \bmu\ra = \tilde{O} (\sigma_0 \norm{\bmu}_2)$ in Lemma~\ref{thm:noise_proof} and \eqref{eqB17}, the last inequality is by $\sigma_{0} \leq \tilde{O}\left(\min\{(\sigma_{p}\sqrt{d})^{-1}, \|\bmu\|_{2}^{-1}\} \kappa^{\frac{q-1}{q}} \right)$ and $d \geq \tilde{\Omega}\left(m^{2}n^{2} \kappa^{-\frac{2q-2}{q}} \right)$ specified in Condition~\ref{condition}. 

Therefore, with probability at least $1 - 0.5$, we have that 
\begin{align*}
\ell\big(y \cdot f(\bW^{(t)},\bx)\big) \geq \log(1 + e^{-1}).   
\end{align*}
Thus $L_{\cD}(\bW^{(t)}) \geq \log(1 + e^{-1})\cdot 0.5\geq 0.1$. Thus we complete the proof.
\end{proof}

Next, we try to provide a lower bound on the test error.
\begin{lemma} \label{GDnoisehit2}
	Under Condition \ref{condition}, let $T_2= T_1+ \frac{36 nm^2}{\eta \sigma_p^2d}$. For the time period $T_2 \leq t \leq T^*$, we have
	\begin{align*}
		\sum_{r=1}^m \overline{\rho}_{y_i,r,i}^{(t)} \geq m,
	\end{align*}
	for all $i \in[n]$.
\end{lemma}

\begin{proof}
	Denote $\lambda_{i}^{(t)}= \frac{1}{m} \sum_{r=1}^m \overline{\rho}_{y_i,r,i}^{(t)}$, notice that
	\begin{align*}
		y_i f(\bW^{(t)},\bx_i) &=  F_{y_i} (\bW_{y_i}^{(t)},\bx_i) -  F_{-y_i} (\bW_{-y_i}^{(t)},\bx_i) \\
		&\leq \frac{1}{m} \sum_{r=1}^m \sigma \left(\langle\bw_{y_i,r}^{(t)}, y_i \bmu\rangle -\gamma_{y_i,r}^{(t)}+ \gamma_{y_i,r}^{(t)} \right) + \frac{1}{m} \sum_{r=1}^m \sigma \left(\langle\bw_{y_i,r}^{(t)}, \bxi_i \ra -\overline{\rho}_{y_i,r,i}^{(t)} +\overline{\rho}_{y_i,r,i}^{(t)} \rangle \right) \\
		& \leq \frac{1}{m} \sum_{r=1}^m \sigma \left( \gamma_{y_i,r}^{(t)}+ \beta \right) + \frac{1}{m} \sum_{r=1}^m \sigma \left( \overline{\rho}_{y_i,r,i}^{(t)}+ \beta+\zeta \right)  \\
		& \leq \lambda_i^{(t)}+ 0.3\kappa +\tilde O(\sigma_0 \norm{\bmu}_2) \leq \lambda_i^{(t)}+ \log 2,
	\end{align*}
	where the first inequality is by $F_{-y_i} (\bW_{-y_i}^{(t)},\bx_i)\geq 0$, the second inequality is by Lemma \ref{innerprodectwjrxi}, the third inequality is by the property $\sigma(z) \leq z$ for $z \geq 0$ and Lemma \ref{thm:noise_proof}, and the last inequality is by the condition on $\sigma_0$ specified in Condition \ref{condition}.
	It follows that
	\begin{align} 
		-\ell_i^{\prime(t)} = \frac{1}{1+ e^{y_i f(\bW^{(t)},\bx_i)}} \geq \frac{1}{3} e^{-y_i f(\bW^{(t)},\bx_i)} \geq \frac{1}{6} e^{-\lambda_i^{(t)}}
	\end{align}
	where we use $e^{y_i f(\bW^{(t)},\bx_i)} \geq e^{-  F_{-y_i} (\bW_{-y_i}^{(t)},\bx_i)} \geq e^{-\kappa} \geq \frac{1}{2}$ according to Lemma \ref{Fminusyibound}.
	
	Therefore, according to the update rule in Lemma \ref{lemma:GDcoefficient}, for each $i\in [n]$ and $t \geq T_1$, since $\max_{j,r}  \langle \bw_{j,r}^{(t)}, \bxi_{i} \rangle \geq \kappa$ by Lemma \ref{GDFirstStage}, we have
	\begin{align*}
		\lambda_{i}^{(t+1)} &=\lambda_{i}^{(t)}  -\frac{\eta}{nm^2} \ell_{i}^{\prime(t)} \sum_{r=1}^m \sigma^\prime \left( \langle \bw_{j,r}^{(t)}, \bxi_{i} \rangle\right)  \cdot \norm {\bxi_i}_2^{2} \\
		& \geq \lambda_{i}^{(t)} + \frac{\eta \sigma_p^2d}{12nm^2} e^{-\lambda_{i}^{(t)}},
	\end{align*}
	by Lemma \ref{tensorpower2}, we have
	\begin{align}
		\lambda_{i}^{(t)} &\geq \log \left(\frac{\eta \sigma_p^2d}{12nm^2}  (t- T_1)+ e^{\frac{\kappa}{m}}\right).
	\end{align}
	Therefore, by choosing $T_2= T_1+ \frac{36 nm^2}{\eta \sigma_p^2d}$, we have $\lambda_{i}^{(T_2)} \geq 1$. Thus we complete the proof.
\end{proof}

Then we present an important Lemma, which bounds the Total Variation (TV) distance between two Gaussian with the same covariance matrix.

\begin{lemma}[Proposition 2.1 in \citet{devroye2018total}]\label{lm: TV}
The TV distance between $\cN(0, \sigma_{p}^{2}\Ib_{d})$ and $\cN(\bv, \sigma_{p}^{2}\Ib_{d})$ is smaller than $\|\bv\|_{2}/2\sigma_p$.
\end{lemma}

Based on this lemma, we are able to derive a lower bound for the test error.

\begin{lemma} \label{gd test error lower bound}
Under Condition \ref{condition}, further suppose that $\sigma_0 \leq \frac{C_3}{m  \norm{\bmu}_2 \sqrt{d}}$  for some small constant $C_3$. For the time period $T_2 \leq t \leq T^*$, we have that $\mathcal{R}_{\mathcal{D}}(\bW^{(t)})\geq 0.11$.
\end{lemma}

\begin{proof} [Proof of Lemma \ref{gd test error lower bound}]
For the time period $T_2 \leq t \leq T^*$, consdier a new data sample $(\bx, y) \sim \cD$ where $\bx=(y \cdot \bmu, \xi)$, we have
\begin{equation}
    \begin{aligned}
        &\mathcal{R}_{\mathcal{D}}(\bW^{(t)})= \mathbb{P} \big(y f(\bW^{(t)},\bx) < 0 \big)\\
        &= \mathbb{P} \left(  \sum_{r=1}^m \left[\sigma(\la \bw_{y,r}^{(t)}, y\cdot \bmu \ra) + \sigma(\la \bw_{y,r}^{(t)}, \bxi \ra) \right] 
        < \sum_{r=1}^m \left[\sigma(\la \bw_{-y,r}^{(t)}, y\cdot \bmu \ra) + \sigma(\la \bw_{-y,r}^{(t)}, \bxi \ra) \right] \right)\\
        & \geq 0.5 \mathbb{P} \left(  \sum_{r=1}^m \sigma(\la \bw_{1,r}^{(t)}, \bxi \ra) - \sum_{r=1}^m \sigma(\la \bw_{-1,r}^{(t)}, \bxi \ra) > \sum_{r=1}^m \sigma(\la \bw_{1,r}^{(t)}, \bmu \ra) \right) \\
        & \geq 0.5 \mathbb{P} \left(  \sum_{r=1}^m \sigma(\la \bw_{1,r}^{(t)}, \bxi \ra) - \sum_{r=1}^m \sigma(\la \bw_{-1,r}^{(t)}, \bxi \ra) > \tilde O(m \sigma_0 \norm{\bmu}_2)  \right)
    \end{aligned}\label{eq11}
\end{equation}
where the first inequality is because $y$ is generated as a Rademacher random variable, and the second inequality is because 
$$\sum_{r=1}^m \sigma(\la \bw_{1,r}^{(t)}, \bmu \ra) \leq \sum_{r=1}^m \sigma\left(\tilde O( \sigma_0 \norm{\bmu}_2)\right) \leq \tilde O(m \sigma_0 \norm{\bmu}_2) ,$$
according to Lemma \ref{thm:noise_proof}, the property $\sigma(z) \leq z$ for $z\geq 0$.

Let $g(\bxi) = \sum_{r} \sigma(\la \bw_{1, r}^{(t)}, \bxi \ra) - \sum_{r} \sigma(\la \bw_{-1, r}^{(t)}, \bxi \ra) $. Denote the set 
\begin{align*}
\Omega := \bigg\{\bxi \bigg|  g(\bxi) >  \tilde O(m \sigma_0 \norm{\bmu}_2)  \bigg\},    
\end{align*}
we can write the test error as
\begin{align*}
    \mathcal{R}_{\mathcal{D}}(\bW^{(t)}) \geq 0.5 \mathbb{P}(\Omega).
\end{align*}

\iffalse
Similarly, we have
\begin{align*}
    g(-\bxi+\bv)- g(-\bxi) 
          &\geq  \sum_{r} \sigma\left(\la \bw_{1, r}^{(t)}, -\bxi \ra+\tau \sum_{i:y_i=1} (\overline{\rho}_{1,r,i}^{(t)}- 2\beta) \right)- \sum_{r} \sigma(\la \bw_{1, r}^{(t)}, -\bxi \ra) -2mn\tau \beta.
\end{align*}
Denote $\xi_r = \la \bw_{1, r}^{(t)}, \bxi \ra$ and $a_r= \tau \sum_{i:y_i=1} (\overline{\rho}_{1,r,i}^{(t)}- 2\beta)$. By Lemma \ref{GDnoisehit2} and Lemma \ref{Sjsize}, we have
$$\sum_r a_r= \tau \sum_r \sum_{i:y_i=1} (\overline{\rho}_{1,r,i}^{(t)}- 2\beta) \geq 0.2\tau m n .$$ 
Since one of $x_r$ and $-x_r$ must be larger than or equal to zero, we have
\fi

Notice that for each $i\in [n]$, we have $\sum_{r=1}^m \overline{\rho}_{y_i,r,i}^{(t)} \geq m$ by Lemma \ref{GDnoisehit2}, denote $r_i= \arg \max_r \overline{\rho}_{y_i,r,i}^{(t)}$, we have $\overline{\rho}_{y_i,r_i,i}^{(t)} \geq 1$. Further denote $i^*= \arg \max_{i:y_i=1} \overline{\rho}_{y_i,r_i,i}^{(t)}$, and $r^*=r_{i^*}$. Additionally, we consider the following set $$\mathcal{E}:= \left\{\bxi\bigg| \la \bw_{1, r^*}^{(t)}, \bxi \ra \geq 0\right\}.$$
 Denote $\bv=  \tau \bxi_{i^*}$, where $\tau =\tilde \Omega(m \sigma_0 \norm{\bmu}_2)$,
notice that
\begin{align*}
	g(\bxi+\bv)- g(\bxi) &=
	\sum_{r} \sigma(\la \bw_{1, r}^{(t)}, \bxi+\bv \ra)- \sum_{r} \sigma(\la \bw_{1, r}^{(t)}, \bxi \ra) \\
	& + \sum_{r} \sigma(\la \bw_{-1, r}^{(t)}, \bxi \ra) - \sum_{r} \sigma(\la \bw_{-1, r}^{(t)}, \bxi+\bv \ra) \\ 
	&\geq  \sum_{r} \sigma\left(\la \bw_{1, r}^{(t)}, \bxi \ra+\tau  (\overline{\rho}_{1,r,i^*}^{(t)}- \beta-\zeta) \right)- \sum_{r} \sigma(\la \bw_{1, r}^{(t)}, \bxi \ra) \\
	&+ \sum_{r} \sigma(\la \bw_{-1, r}^{(t)}, \bxi \ra) - \sum_{r} \sigma\left(\la \bw_{-1, r}^{(t)}, \bxi \ra+ \tau (\beta+\zeta)\right) \\
	&\geq  \sum_{r} \sigma\left(\la \bw_{1, r}^{(t)}, \bxi \ra+ \tau  (\overline{\rho}_{1,r,i^*}^{(t)}- \beta-\zeta) \right)- \sum_{r} \sigma(\la \bw_{1, r}^{(t)}, \bxi \ra)-  m\tau (\beta+\zeta),
\end{align*}
where the first inequality is by
\begin{align*}
	&\la \bw_{1, r}^{(t)}, \bv \ra= \tau  \la \bw_{1, r}^{(t)}, \bxi_{i^*} \ra \geq  \tau  (\overline{\rho}_{1,r,i^*}^{(t)}- \beta-\zeta), \\
	&\la \bw_{-1, r}^{(t)}, \bv \ra= \tau  \la \bw_{-1, r}^{(t)}, \bxi_{i^*} \ra \leq  \tau  (\underline{\rho}_{-1,r,i^*}^{(t)}+ \beta+\zeta) \leq   \tau(\beta+\zeta)
\end{align*}
according to Lemma \ref{innerprodectwjrxi}. It follows that
we have
\begin{align*}
    g(\bxi+\bv)- g(\bxi) & \geq     \sigma\left(\la \bw_{1, r^*}^{(t)}, \bxi \ra+ \tau  (\overline{\rho}_{1,r^*,i^*}^{(t)}- \beta-\zeta) \right)- \sigma(\la \bw_{1, r^*}^{(t)}, \bxi \ra) \\
    & + \sum_{r \neq r^*} \sigma\left(\la \bw_{1, r}^{(t)}, \bxi \ra+ \tau  (\overline{\rho}_{1,r,i^*}^{(t)}- \beta-\zeta) \right)- \sum_{r} \sigma(\la \bw_{1, r}^{(t)}, \bxi \ra)-  m\tau (\beta+\zeta), \\
    & \geq \tau- 2m\tau (\beta+\zeta) \geq \frac{\tau}{2} \geq \tilde \Omega(m \sigma_0 \norm{\bmu}_2)
\end{align*}
where the second inequality is by $\la \bw_{1, r^*}^{(t)}, \bxi \ra \geq 0$ and $\overline{\rho}_{1,r^*,i^*}^{(t)}- \beta-\zeta \geq 1- \beta-\zeta \geq 0$, the third inequality is by the condition on $\sigma_{0}$ and $d$ specified in Condition \ref{condition}.

Therefore, by pigeon's hole principle, there must exist one of $g(\bxi+\bv), g(-\bxi)$ larger than $\tilde \Omega(m \sigma_0 \norm{\bmu}_2)$, hence one of $\bxi + \bv$, $-\bxi$ belongs $\Omega$. It follows that $ (-\Omega \cap \mathcal E) \cup ((\Omega - \{\bv\}) \cap \mathcal E)  =\mathcal E$. Therefore, we have 
$$\PP(-\Omega \cap \mathcal E) + \PP((\Omega - \{\bv\}) \cap \mathcal E) \geq \PP(\mathcal{E})=0.5.$$
Hence, one of $\PP(-\Omega \cap \mathcal E)$ and $\PP((\Omega - \{\bv\}) \cap \mathcal E)$ is larger than 0.25.

 Furthermore, Notice that $\mathbb{P}(-\Omega) = \mathbb{P}(\Omega)$ and 
\begin{align*}
\left|\mathbb{P}(\Omega \cap \mathcal{E}) - \mathbb{P}((\Omega - \{\bv\}) \cap (\mathcal{E}- \{\bv\}))\right| &= \left|\mathbb{P}_{\bxi \sim \cN(0, \sigma_p^{2} \Ib_{d})}(\bxi \in \Omega \cap \mathcal{E}) - \mathbb{P}_{\bxi \sim \cN(\bv, \sigma_{p}^{2}\Ib_{d})}(\bxi \in \Omega \cap \mathcal{E})\right|\\
&\leq \text{TV}\left(\cN(0, \sigma_p^{2} \Ib_{d}), \cN(\bv, \sigma_p^{2} \Ib_{d})\right)\\
&\leq \frac{\|\bv\|_{2}}{2\sigma_{p}} 
\leq \tilde{\Omega}(m \sigma_0 \sqrt{d} \norm{\bmu}_2) \leq 0.01,
\end{align*}
where the first inequality is by the definition of the TV distance, the second inequality is by 
Lemma~\ref{lm: TV}, and the last inequality is by the condition $\sigma_0 \leq O \left( \frac{1}{m  \norm{\bmu}_2 \sqrt{d}} \right)$.  Similarly, we have
\begin{align*}
	\left|\mathbb{P}( \mathcal{E}) - \mathbb{P}(\mathcal{E}- \{\bv\})\right| \leq 0.01,
\end{align*}
moreover, since $\la \bw_{1, r^*}^{(t)}, \bv \ra \geq \tau(\overline{\rho}_{1,r^*,i^*}^{(t)}- \beta-\zeta)>0$, we have $\mathcal{E}- \{\bv\} \subset \mathcal{E} $. Therefore, we have $(\Omega - \{\bv\}) \cap (\mathcal{E}- \{\bv\})\subset (\Omega - \{\bv\}) \cap \mathcal{E}$. Hence, we have
\begin{align*}
	|\PP((\Omega - \{\bv\}) \cap (\mathcal{E}- \{\bv\}))-\PP((\Omega - \{\bv\}) \cap \mathcal E)| \leq 0.01,
\end{align*}
it follows that
\begin{align*}
	&\left|\mathbb{P}(\Omega \cap \mathcal{E}) -\PP((\Omega - \{\bv\}) \cap \mathcal E) \right| \\
	&\leq 
	\left|\mathbb{P}(\Omega \cap \mathcal{E}) - \mathbb{P}((\Omega - \{\bv\}) \cap (\mathcal{E}- \{\bv\}))\right| + |\PP((\Omega - \{\bv\}) \cap (\mathcal{E}- \{\bv\}))-\PP((\Omega - \{\bv\}) \cap \mathcal E) | \\
	& \leq 0.02.
\end{align*}
Notice that 
\begin{align*}
	& \PP(\Omega) =\PP(-\Omega) \geq \PP(-\Omega \cap \mathcal E), \\
	& \PP(\Omega) \geq \mathbb{P}(\Omega \cap \mathcal{E}) \geq \PP((\Omega - \{\bv\}) \cap \mathcal E) -0.02,
\end{align*}
we conclude that $\PP(\Omega) \geq 0.23$. Thus we complete the proof.
\end{proof}

\section{Signal Learning of DP-GD}
In this section, we consider the signal case by utilizing the DP-GD training algorithm under Condition \ref{condition2}. These results are based on the conclusions in Appendix \ref{appendixB}, which hold with probability at least $1-6\delta$.

\subsection{Properties of the learning of signal and noise}

Based on the update rule of DP-GD \eqref{NGDupdate}, we have the following iterative equations,
    \begin{eqnarray}
        &&\langle\bw_{j,r}^{(t+1)}, j\bmu\rangle = \langle \bw_{j,r}^{(t)}, j\bmu\rangle- \frac{\eta}{nm} \sum_{i=1}^n \ell_{i}^{\prime(t)} \cdot \sigma^\prime \left( \langle \bw_{j,r}^{(t)}, y_{i} \bmu \rangle\right) \cdot \|\bmu\|_2^2 - \eta \langle \bb_{j,r,t}, j\bmu\rangle,  \label{C1} \\
        && \langle\bw_{j,r}^{(t+1)}, -j\bmu\rangle = \langle \bw_{j,r}^{(t)}, -j\bmu\rangle + \frac{\eta}{nm} \sum_{i=1}^n \ell_{i}^{\prime(t)} \cdot \sigma^\prime \left( \langle \bw_{j,r}^{(t)}, y_{i} \bmu \rangle\right) \cdot \|\bmu\|_2^2 + \eta \langle \bb_{j,r,t}, j\bmu\rangle, \label{C2} \\
        && \nonumber \langle\bw_{j,r}^{(t+1)}, \bxi_i \rangle = \langle \bw_{j,r}^{(t)}, \bxi_i \rangle- \frac{\eta}{nm} \ell_{i}^{\prime(t)} \cdot \sigma^\prime \left( \langle \bw_{j,r}^{(t)}, \bxi_{i} \rangle\right) \cdot \|\bxi_i\|_2^2 \\
        &&  \qquad \qquad \qquad - \frac{\eta}{nm} \sum_{i'\neq i} \ell_{i'}^{\prime(t)} \cdot \sigma^\prime \left( \langle \bw_{j,r}^{(t)}, \bxi_{i'} \rangle\right) \cdot \langle \bxi_i, \bxi_{i'} \rangle - \eta \langle \bb_{j,r,t}, \bxi_i \rangle, \quad (\hbox{for} \ j=y_i), \label{C3} \\
        && \nonumber \langle\bw_{j,r}^{(t+1)}, \bxi_i \rangle = \langle \bw_{j,r}^{(t)}, \bxi_i \rangle+ \frac{\eta}{nm} \ell_{i}^{\prime(t)} \cdot \sigma^\prime \left( \langle \bw_{j,r}^{(t)}, \bxi_{i} \rangle\right) \cdot \|\bxi_i\|_2^2 \\
        &&  \qquad \qquad \qquad + \frac{\eta}{nm} \sum_{i'\neq i} \ell_{i'}^{\prime(t)} \cdot \sigma^\prime \left( \langle \bw_{j,r}^{(t)}, \bxi_{i'} \rangle\right) \cdot \langle \bxi_i, \bxi_{i'} \rangle + \eta \langle \bb_{j,r,t}, \bxi_i \rangle, \quad (\hbox{for} \ j \neq y_i) \label{C4}.
    \end{eqnarray}

By utilizing a signal-noise decomposition expression in Definition \eqref{def3}, similar as Lemma \ref{lemma:GDcoefficient}, we can derive the following lemma, which presents an iterative expression for the change of coefficients.

\begin{lemma} \label{lemma18}
     The coefficients $\gam^{(t)},\orho^{(t)},\urho^{(t)}$ defined in Definition \ref{def3} satisfy the following iterative equations:
    \begin{eqnarray*}
        &&\gam^{(0)},\orho^{(0)},\urho^{(0)} = 0, \\
        &&\gam^{(t+1)}= \gam^{(t)} - \frac{\eta}{nm} \sum_{i=1}^n \ell_{i}^{\prime(t)} \cdot \sigma^\prime \left( \langle \bw_{j,r}^{(t)}, y_{i} \bmu \rangle\right) \cdot \norm \bmu_2^{2}, \\
        && \orho^{(t+1)}= \orho^{(t)} -\frac{\eta}{nm} \ell_{i}^{\prime(t)} \cdot \sigma^\prime \left( \langle \bw_{j,r}^{(t)}, \bxi_{i} \rangle\right)  \cdot \norm {\bxi_i}_2^{2} \cdot \mathds{1} (y_{i}=j) ,  \\
        && \urho^{(t+1)}= \urho^{(t)} +\frac{\eta}{nm} \ell_{i}^{\prime(t)} \cdot \sigma^\prime \left( \langle \bw_{j,r}^{(t)}, \bxi_{i} \rangle\right)  \cdot \norm {\bxi_i}_2^{2} \cdot \mathds{1} (t \in y_{i}=-j) .
    \end{eqnarray*}
\end{lemma}
    
Next, we aim to analyze the coefficients in the signal-noise decomposition in Definition \ref{def3}. Specifically, we will show that the learning of the signal and noise will stay within a reasonable range for a considerable amount of time. Consider the training period $0\leq t \leq \tilde T^*$, where $\tilde T^*= \Theta \left(\frac{\kappa^2}{\eta^2 \sigma_b^2 \min\{\norm{\bmu}_2^2,\sigma_p^2d\}} + \frac{nm^2}{\eta \epsilon \max\{\norm{\bmu}_2^2,\sigma_p^2d\}} \right)$ is the maximum admissible iterations. Denote 
\begin{eqnarray*}
    && \tilde \alpha := 4 m \log \left( \tilde T^*\right), \\
    && \beta := 2 \max_{j,r,i} \left\{ |\la \bw_{j,r}^{(0)}, \bmu \ra|, |\la \bw_{j,r}^{(0)}, \bxi_i \ra | \right\}, \\
    && \zeta := 8 n \sqrt{\frac{\log(6n^2/\delta)}{d}} \tilde \alpha, \\
    && \theta := 16\eta\sigma_b \max\left\{\norm{\bmu}_2,\sigma_p \sqrt{d}\right\} \sqrt{\tilde T^*}  \log ^2\left(\frac{32mn\tilde T^*}{\delta}\right),
%    && \snr := \frac{\norm {\bmu}_2}{\sigma_p \sqrt{d}}, \\
    %&& \beta:= \frac{1}{n \cdot {\snr}^2}
\end{eqnarray*}
Recall the condition on $\sigma_0$ and $d$ specified in Condition \ref{condition2},
\begin{align*}
	& d \geq C\frac{m^2 n^2}{\kappa^2} \log(\frac{n^2}{\delta}) (\log(T^*))^2, \\
	&  \sigma_0 \leq \left(C \max\left\{\norm {\bmu}_2 , \sigma_p \sqrt{d} \right\} \sqrt{\log(\frac{mn}{\delta})} \right)^{-1} \kappa,
\end{align*}
the same as is described in Subsection \ref{subsecB1}, we have $\beta, \zeta \leq 0.1 \kappa$. Moreover, according to Condition \ref{condition2}
\begin{align*}
    \sigma_b \leq \frac{1}{C \eta \max\left\{\norm{\bmu}_2,\sigma_p \sqrt{d}\right\} \sqrt{\tilde T^*}  \log ^2\left(\frac{32mn\tilde T^*}{\delta}\right)},
    	%    && \snr := \frac{\norm {\bmu}_2}{\sigma_p \sqrt{d}}}
\end{align*}
for some large constant $C$,
we have $\theta = O(1)$.

In the next proposition, we demonstrate a coarse bound on the growth of $ \gam^{(t)}$, $\orho^{(t)}$, and $\urho^{(t)}$. We omit the proof of this proposition since it is the same as the proof of Proposition \ref{proposition:GDgrowthbound}, by utilizing the bound of the additional term $\theta$.

\begin{proposition} \label{proposition:NGDgrowthbound} 
%[Part ial restatement of Proposition \shi{to complete}]
    Under Condition \ref{condition2}, for $0\leq t \leq \tilde T^*$, we have that
    \begin{eqnarray}
        %&& \gam^{(0)},\orho^{(0)},\urho^{(0)} = 0, \label{eqB7} \\
        && \sum_{r=1}^m \gam^{(t)}\leq \tilde \alpha, \quad \sum_{r=1}^m \sigma\left(\la \bw_{j,r}^{(t)},j\bmu \ra\right)\leq \tilde \alpha, \label{eqCt5}\\
        && 0 \geq \sum_{r=1}^m \urho^{(t)} \geq -\tilde \alpha, \label{eqCt7}
    \end{eqnarray}
    for all $r \in [m]$, $j\in \{\pm 1\}$, and $i\in [n]$.
\end{proposition}

 We suppose that the results in Proposition \ref{proposition:NGDgrowthbound} hold for the time $ t \leq \tilde T^*$, then we can derive the following properties.

\begin{lemma} \label{innerprodectwjrxiNGD}
    Under Condition \ref{condition2}, suppose  \eqref{eqCt5} and \eqref{eqCt7} hold for any iteration $t \leq \tilde T^*$. Then, for all $r\in[m]$, $j\in \{\pm 1\}$ and $i\in[n]$, we have
    \begin{eqnarray}
        && \left| \la \wjr^{(t)}, \bmu \ra - j \cdot \gam^{(t)} \right| \leq \beta +\theta \leq 0.1\kappa + \theta , \label{eqC8} \\
        && \left| \la \wjr^{(t)}, \bxi_i \ra -  \orho^{(t)} \right| \leq \beta +\zeta+\theta \leq 0.2\kappa + \theta, \quad \text{if} \ j=y_i, \label{eqC9} \\
        && \left| \la \wjr^{(t)}, \bxi_i \ra -  \urho^{(t)} \right| \leq \beta +\zeta+\theta \leq 0.2\kappa + \theta, \quad \text{if} \ j \neq y_i. \label{eqC10}
    \end{eqnarray}
\end{lemma}

\begin{proof} [Proof of Lemma \ref{innerprodectwjrxiNGD}]
    Firstly, for any time $0 \leq t \leq \tilde t-1$, we have from the signal-noise decomposition (\ref{sndecompositionNGD}) that
    \begin{equation*}
        \la \wjr^{(t)}, \bmu \ra = \la \wjr^{(0)}, \bmu \ra  +j \cdot \gam^{(t)}  - \eta \sum_{k=0}^{t-1} \la \bb_{j,r,k}, \bmu \ra.
    \end{equation*}
    Notice that we have
    \begin{equation*}
        \begin{aligned}
            \left| \la \wjr^{(t)}, \bmu \ra - j \cdot \gam^{(t)} \right| &\leq \left| \la \wjr^{(0)}, \bmu \ra\right| + \eta \left|\sum_{k=0}^{t-1} \la \bb_{j,r,k}, \bmu \ra \right| \\
            \leq & \beta + 8 \eta \sigma_b \norm {\bmu}_2 \sqrt{t}  \log ^2\left(\frac{16mt}{\delta}\right)  \leq \beta + \theta,
        \end{aligned}
    \end{equation*}
    where the first inequality is by triangle inequality, the second inequality is by Lemma \ref{sumofnoisebound}.

    Secondly, for $j = y_i$, we have $\urho^{(t)}=0$, and 
    \begin{equation*}
        \begin{aligned}
            \la \wjr^{(t)}, \bxi_i \ra &= \la \wjr^{(0)}, \bxi_i \ra  + \sum_{i'=1}^n \overline{\rho}_{j,r,i'}^{(t)} \cdot \norm {\bxi_{i'}}_2^{-2} \cdot \la \bxi_{i'},\bxi_i \ra 
             - \eta \sum_{k=0}^{t-1} \la \bb_{j,r,k}, \bxi_i \ra \\
            &= \orho^{(t)} + \la \wjr^{(0)}, \bxi_i \ra +  \sum_{i'\neq i} \overline{\rho}_{j,r,i'}^{(t)} \cdot \norm {\bxi_{i'}}_2^{-2} \cdot \la \bxi_{i'},\bxi_i \ra  - \eta \sum_{k=0}^{t-1} \la \bb_{j,r,k}, \bxi_i \ra ,
        \end{aligned}
    \end{equation*}
     it follows that
    \begin{equation*}
        \begin{aligned}
             \left| \la \wjr^{(t)}, \bxi_i \ra -  \orho^{(t)} \right| 
             &\leq \left|\la \wjr^{(0)}, \bxi_i \ra\right|+ \sum_{i'\neq i} \overline{\rho}_{j,r,i'}^{(t)} \cdot \norm {\bxi_{i'}}_2^{-2} \cdot \left|\la \bxi_{i'},\bxi_i \ra\right|  + \eta \left|\sum_{k=0}^{t-1} \la \bb_{j,r,k}, \bxi_i \ra \right|   \\
             & \leq \beta + 4 \sqrt{\frac{\log(6n^2/\delta)}{d}}   \sum_{i=1}^n  \orho^{(t)}   +  16\eta\sigma_b \sigma_p \sqrt{dt}  \log ^2\left(\frac{32mnt}{\delta}\right)  \\
            & \leq \beta + \zeta +\theta,
        \end{aligned}
    \end{equation*}
     where the first inequality is by triangle inequality, the second inequality is by Lemma \ref{innerproductxi} and Lemma \ref{sumofnoisebound}, the third inequality is by  \eqref{eqCt5}.
    
    Finally, for $j \neq y_i$, we have $\orho^{(t)}=0$, and
   \begin{equation*}
        \begin{aligned}
            \la \wjr^{(t)}, \bxi_i \ra &= \la \wjr^{(0)}, \bxi_i \ra  + \sum_{i'=1}^n \underline{\rho}_{j,r,i'}^{(t)} \cdot \norm {\bxi_{i'}}_2^{-2} \cdot \la \bxi_{i'},\bxi_i \ra 
             - \eta \sum_{k=0}^{t-1} \la \bb_{j,r,k}, \bxi_i \ra \\
            &= \urho^{(t)} + \la \wjr^{(0)}, \bxi_i \ra +  \sum_{i'\neq i} \underline{\rho}_{j,r,i'}^{(t)} \cdot \norm {\bxi_{i'}}_2^{-2} \cdot \la \bxi_{i'},\bxi_i \ra  - \eta \sum_{k=0}^{t-1} \la \bb_{j,r,k}, \bxi_i \ra ,
        \end{aligned}
    \end{equation*}
     it follows that
    \begin{equation*}
        \begin{aligned}
             \left| \la \wjr^{(t)}, \bxi_i \ra -  \urho^{(t)} \right| 
             &\leq \left|\la \wjr^{(0)}, \bxi_i \ra\right|+ \sum_{i'\neq i} \left|\underline{\rho}_{j,r,i'}^{(t)}\right| \cdot \norm {\bxi_{i'}}_2^{-2} \cdot \left|\la \bxi_{i'},\bxi_i \ra\right|  + \eta \left|\sum_{k=0}^{t-1} \la \bb_{j,r,k}, \bxi_i \ra \right|   \\
             & \leq \beta + 4 \sqrt{\frac{\log(6n^2/\delta)}{d}}   \sum_{i=1}^n  \left|\underline{\rho}_{j,r,i}^{(t)}\right|    +  16\eta\sigma_b \sigma_p \sqrt{dt}  \log ^2\left(\frac{32mnt}{\delta}\right)  \\
            & \leq \beta + \zeta +\theta,
        \end{aligned}
    \end{equation*}
     where the first inequality is by triangle inequality, the second inequality is by Lemma \ref{innerproductxi} and Lemma \ref{sumofnoisebound}, the third inequality is by  \eqref{eqCt7}.
    Thus we complete the proof.
\end{proof}

\begin{lemma} \label{FminusyiboundNGD}
    Under Condition \ref{condition2}, suppose  \eqref{eqCt5} and \eqref{eqCt7} hold for any iteration $t \leq \tilde T^*$. For all $i\in[n]$ and $j \neq y_i$, we have
    \begin{equation}
        F_{j} (\bW_{j}^{(t)},\bx_i) \leq 2\beta+\zeta+2\theta \leq 0.3\kappa +2\theta.
    \end{equation}
\end{lemma}

\begin{proof} [Proof of Lemma \ref{Fminusyibound}]
    % the condition $\sigma_b \leq \frac{1}{C \eta \sqrt{\tilde T^*} \max\{\norm{\bmu}_2, \sigma_p \sqrt{d}\} \log^2 \frac{mn \tilde T^*}{\delta}}$, and $\tilde C_1= 0.1\kappa + \frac{16}{C}$ is very small by a big choice of $C$ in Condition \ref{condition}. 
    By writing out the expression of $F_{j} (\bW_j^{(t)},\bx_i)$, we have
    \begin{equation*}
    \begin{aligned}
        F_{j} (\bW_j^{(t)},\bx_i) &= \frac{1}{m} \sum_{r=1}^m \left[\sigma\left(\la \wjr^{(t)}, y_i \cdot \bmu \ra\right) + \sigma\left(\la \wjr^{(t)},  \bxi_i \ra\right)\right] \\
         &= \frac{1}{m} \sum_{r=1}^m \left[\sigma\left(\la \wjr^{(t)}, y_i \cdot \bmu \ra+ \gam^{(t)}- \gam^{(t)} \right) + \sigma\left(\la \wjr^{(t)},  \bxi_i \ra -\urho^{(t)}+ \urho^{(t)} \right)\right] \\
        & \leq \frac{1}{m} \sum_{r=1}^m \left[ \sigma\left(\left| \la \wjr^{(t)}, y_i \cdot \bmu \ra+ \gam^{(t)} \right|\right) + \sigma\left(\left|\la \wjr^{(t)},  \bxi_i \ra- \urho^{(t)} \right| \right) \right] \\
        & \leq  \sigma(\beta+ \theta) + \sigma(\beta+ \zeta+\theta) \leq 2\beta+\zeta+2\theta \leq 0.3\kappa +2\theta,
    \end{aligned}
    \end{equation*}
    where the first inequality is by $\gam^{(t)} \geq 0$, $\urho^{(t)} \leq 0$, the second inequality is by \eqref{eqC8} and \eqref{eqC10}, and the last inequality is by the property of the Huberized ReLU activation function $\sigma$,.
\end{proof}

\subsection{Training Loss Analysis}
The following lemma shows the convergence of the training loss for DP-GD.
\begin{lemma} \label{lemma29}
    Under Condition \ref{condition2}, for any $\epsilon >0$, denote $\tilde T_1= \Theta \left(\frac{\kappa^2}{\eta^2 \sigma_b^2 \min\{\norm{\bmu}_2^2,\sigma_p^2d\}} \right)$ and $\tilde T^*= \tilde T_1 + \Theta \left(\frac{nm^2}{\eta \epsilon \max\{\norm{\bmu}_2^2,\sigma_p^2d\}} \right)$, we have $L_S(\bW^{(\tilde T^*)}) \leq \epsilon$. 
\end{lemma}

\begin{proof} [Proof of Lemma \ref{lemma29}]
    For each $t\geq \tilde T_1= \Theta \left(\frac{\kappa^2}{\eta^2 \sigma_b^2 \min\{\norm{\bmu}_2^2,\sigma_p^2d\}} \right)$, it follows from \eqref{C1} that
    \begin{align*}
         \langle\bw_{j,r}^{(t+1)}, j\bmu\rangle \geq \langle \bw_{j,r}^{(t)}, j\bmu\rangle - \eta \langle \bb_{j,r,t}, j\bmu\rangle,
    \end{align*}
    by taking the telescoping sum, we have
    \begin{align*}
        \langle\bw_{j,r}^{(t)}, j\bmu\rangle \geq \langle \bw_{j,r}^{(0)}, j\bmu\rangle - \eta \sum_{k=0}^{t-1} \langle \bb_{j,r,k}, j\bmu\rangle,
    \end{align*}
    it follows that
    \begin{align*}
        \max_{r} \langle\bw_{j,r}^{(t)}, j\bmu\rangle &\geq \eta \max_{r \in [m]} \sum_{k=0}^{t-1} \langle \bb_{j,r,k}, -j\bmu\rangle - \left| \langle \bw_{j,r}^{(0)}, j\bmu\rangle \right| \\
        & \geq \frac{\eta \sigma_b \| \bmu \|_2 \sqrt{\tilde T_1}}{2}- \sigma_0 \| \bmu \|_2 \cdot \sqrt{2\log \left(\frac{12m}{\delta}\right)} \\
        & \geq \frac{\sqrt{\tilde C_1}}{2} \kappa - \sigma_0 \| \bmu \|_2 \cdot \sqrt{2\log \left(\frac{12m}{\delta}\right)} \geq \kappa,
    \end{align*}
    where the second inequality is by Lemma \ref{innerproductw0} and Lemma \ref{maxinnerproductbjrt}, the third inequality is by the defintion of $\tilde T_1= \Theta\left(\frac{ \kappa^2}{\eta^2 \sigma_b^2 \min\{\norm{\bmu}_2^2,\sigma_p^2d\}} \right)$, and the last inequality is by the condition $\sigma_0= O\left(\frac{\kappa}{\max\{\norm{\bmu}_2,\sigma_p \sqrt{d}\} \sqrt{\log\left(\frac{mn}{\delta} \right)}} \right)$ specified in Condition \ref{condition2}. It follows that 
    \begin{align} \label{signalhitk}
        \max_{r} \sigma'\left( \langle\bw_{j,r}^{(t)}, j\bmu\rangle \right) =1 
    \end{align}
    Similarly, for each $i\in [n]$ and $t \geq \tilde T_1$, we have
    \begin{align} \label{noisehitk}
        \max_{r} \sigma'\left( \langle\bw_{y_i,r}^{(t)}, \bxi_i \rangle \right)=1.
    \end{align}
    
    Denote $\lambda_{i}^{(t)}= \frac{1}{m} \sum_{r=1}^m \left(\gamma_{y_i,r}^{(t)}+\overline{\rho}_{y_i,r,i}^{(t)} \right)$, notice that
    \begin{align*}
        y_i f(\bW^{(t)},\bx_i) &=  F_{y_i} (\bW_{y_i}^{(t)},\bx_i) -  F_{-y_i} (\bW_{-y_i}^{(t)},\bx_i) \\
        &\leq \frac{1}{m} \sum_{r=1}^m \sigma \left(\langle\bw_{y_i,r}^{(t)}, y_i \bmu\rangle -\gamma_{y_i,r}^{(t)}+ \gamma_{y_i,r}^{(t)} \right) + \frac{1}{m} \sum_{r=1}^m \sigma \left(\langle\bw_{y_i,r}^{(t)}, \bxi_i \ra -\overline{\rho}_{y_i,r,i}^{(t)} +\overline{\rho}_{y_i,r,i}^{(t)} \rangle \right) \\
        & \leq \frac{1}{m} \sum_{r=1}^m \sigma \left( \gamma_{y_i,r}^{(t)}+ 0.1\kappa +\theta \right) + \frac{1}{m} \sum_{r=1}^m \sigma \left( \overline{\rho}_{y_i,r,i}^{(t)}+ 0.2\kappa +\theta \right)  \\
        & \leq \lambda_i^{(t)}+ 0.3\kappa + 2\theta \leq \lambda_i^{(t)}+ \log 2,
    \end{align*}
    where the first inequality is by $F_{-y_i} (\bW_{-y_i}^{(t)},\bx_i)\geq 0$, the second inequality is by Lemma \ref{innerprodectwjrxiNGD}, and the third inequality is by the property $\sigma(z) \leq z$ for $z \geq 0$.
    It follows that
    \begin{align} \label{ell'lower}
        -\ell_i^{\prime(t)} = \frac{1}{1+ e^{y_i f(\bW^{(t)},\bx_i)}} \geq \frac{1}{3} e^{-y_i f(\bW^{(t)},\bx_i)} \geq \frac{1}{6} e^{-\lambda_i^{(t)}}
    \end{align}
    where we use $e^{y_i f(\bW^{(t)},\bx_i)} \geq e^{-  F_{-y_i} (\bW_{-y_i}^{(t)},\bx_i)} \geq e^{-0.3\kappa -2\theta} \geq \frac{1}{2}$ according to Lemma \ref{FminusyiboundNGD}.
    
    Therefore, according to the update rule in Lemma \ref{lemma18}, for each $i\in [n]$ and $t \geq \tilde T_1$, we have
    \begin{align*}
        \lambda_{i}^{(t+1)} &= \lambda_{i}^{(t)} - \frac{\eta}{nm^2} \sum_{i=1}^n \ell_{i}^{\prime(t)} \sum_{r=1}^m \sigma^\prime \left( \langle \bw_{j,r}^{(t)}, y_{i} \bmu \rangle\right) \cdot \norm \bmu_2^{2} \\
        &-\frac{\eta}{nm^2} \ell_{i}^{\prime(t)} \sum_{r=1}^m \sigma^\prime \left( \langle \bw_{j,r}^{(t)}, \bxi_{i} \rangle\right)  \cdot \norm {\bxi_i}_2^{2},
    \end{align*}
    by \eqref{signalhitk} and \eqref{noisehitk}, we have
    \begin{align} \label{lambdalowerupdate}
        \nonumber \lambda_{i}^{(t+1)} & \geq \lambda_{i}^{(t)} + \frac{\eta \max\{\norm{\bmu}_2^2,\sigma_p^2d\}}{2nm^2} \ell_{i}^{\prime(t)} \left(\sum_{r=1}^m \sigma^\prime \left( \langle \bw_{j,r}^{(t)}, y_{i} \bmu \rangle\right) + \sum_{r=1}^m \sigma^\prime \left( \langle \bw_{j,r}^{(t)}, \bxi_{i} \rangle\right) \right) \\
        & \geq \lambda_{i}^{(t)} + \frac{\eta \max\{\norm{\bmu}_2^2,\sigma_p^2d\}}{12nm^2} e^{-\lambda_{i}^{(t)}},
    \end{align}
    by Lemma \ref{tensorpower2}, we have
    \begin{align} \label{lambdalowerbound}
        \lambda_{i}^{(t)} &\geq \log \left(\frac{\eta \max\{\norm{\bmu}_2^2,\sigma_p^2d\}}{12nm^2}  (t- \tilde T_1)\right).
    \end{align}
    Notice that
    \begin{equation} \label{yfWlower}
    \begin{aligned}
        y_i f(\bW^{(t)},\bx_i) &=  F_{y_i} (\bW_{y_i}^{(t)},\bx_i) -  F_{-y_i} (\bW_{-y_i}^{(t)},\bx_i) \\
        &\geq \frac{1}{m} \sum_{r=1}^m \sigma \left(\langle\bw_{y_i,r}^{(t)}, y_i \bmu\rangle -\gamma_{y_i,r}^{(t)}+ \gamma_{y_i,r}^{(t)} \right) \\
        &+ \frac{1}{m} \sum_{r=1}^m \sigma \left(\langle\bw_{y_i,r}^{(t)}, \bxi_i \ra -\overline{\rho}_{y_i,r,i}^{(t)} +\overline{\rho}_{y_i,r,i}^{(t)} \rangle \right) - 0.3 \kappa - 2\theta \\
        & \geq \frac{1}{m} \sum_{r=1}^m \sigma \left( \gamma_{y_i,r}^{(t)}- 0.1\kappa -\theta \right) + \frac{1}{m} \sum_{r=1}^m \sigma \left( \overline{\rho}_{y_i,r,i}^{(t)}- 0.2\kappa -\theta \right) - 0.3 \kappa - 2\theta \\
        & \geq \lambda_i^{(t)}- 2.6\kappa - 4\theta \geq \lambda_i^{(t)}- \log2,      
    \end{aligned}
        \end{equation}
    where the first inequality is by Lemma \ref{FminusyiboundNGD}, the second inequality is by Lemma \ref{innerprodectwjrxiNGD}, and the last inequality is by the property $\sigma(z-b) \geq z-b-\kappa$ for $z \geq 0$. We have
    \begin{align*}
        \ell \left(y_i f(\bW^{(\tilde T^*)},\bx_i) \right)= \log \left(1+ e^{-y_i f(\bW^{(\tilde T^*)},\bx_i)}\right) \leq e^{- \lambda_i^{(\tilde T^*)}+ \log2} \leq \frac{48nm^2}{\eta \max\{\norm{\bmu}_2^2,\sigma_p^2d\} (\tilde T^*- \tilde T_1)},
    \end{align*}
    therefore, by the choice of $\tilde T^*= \tilde T_1 + \Theta \left(\frac{nm^2}{\eta \epsilon \max\{\norm{\bmu}_2^2,\sigma_p^2d\}} \right)$, we have $\ell \left(y_i f(\bW^{(\tilde T^*)},\bx_i) \right) \leq \epsilon$ for each $i\in [n]$, hence $L_S(\bW^{(\tilde T^*)}) \leq \epsilon$. 
\end{proof}

\subsection{Generalization Error Analysis}
To prove the test error result, we need the following lemma, which shows that when the iteration $\tilde T_2 \leq t \leq \tilde{T}^*$ is large enough, DP-GD can learn the signal as large as $\Theta\left(\frac{1}{m}\right)$.
\begin{lemma} \label{signallowerbound}
    Under Condition \ref{condition2}, denote $c_1= \frac{3\eta \max\{ n\norm \bmu_2^{2}, \sigma_p^2d\}}{nm}$, $\tilde T_1= \Theta \left(\frac{\kappa^2}{\eta^2 \sigma_b^2 \min\{\norm{\bmu}_2^2,\sigma_p^2d\}} \right)$, and $\tilde T_2= \tilde T_1 + \Theta \left(e^{c_1} (\tilde{T}_1 + \frac{1}{c_1}) \right)$. Then, for any $\tilde T_2 \leq t \leq \tilde{T}^*$, we have
    \begin{align}
        \frac{1}{m}\sum_{r=1}^m \gam^{(t)} = \Omega \left(\frac{n \norm \bmu_2^{2}}{e^{c_1} m \max\{n \norm \bmu_2^{2}, \sigma_p^2d\}} \right).
    \end{align}
\end{lemma}

\begin{proof} [Proof of Lemma \ref{signallowerbound}]
     Denote $\lambda_{i}^{(t)}= \frac{1}{m} \sum_{r=1}^m \left(\gamma_{y_i,r}^{(t)}+\overline{\rho}_{y_i,r,i}^{(t)} \right)$, and $\tilde \lambda^{(t)} = \min_{i} \lambda_{i}^{(t)}$. Recall that $y_if(\bW^{(t)},\bx_i) \geq \lambda_i^{(t)}- \log2$ by \eqref{yfWlower}, we have $-\ell_i^{\prime (t)} \leq 2 e^{-\lambda_i^{(t)}}$. According to the update rule in Lemma \ref{lemma18}, for $0 \leq t \leq \tilde T^*$ and each $i\in [n]$, we have
    \begin{align*}
        \lambda_{i}^{(t+1)} &= \lambda_{i}^{(t)} - \frac{\eta}{nm^2} \sum_{i=1}^n \ell_{i}^{\prime(t)} \sum_{r=1}^m \sigma^\prime \left( \langle \bw_{j,r}^{(t)}, y_{i} \bmu \rangle\right) \cdot \norm \bmu_2^{2} \\
        &-\frac{\eta}{nm^2} \ell_{i}^{\prime(t)} \sum_{r=1}^m \sigma^\prime \left( \langle \bw_{j,r}^{(t)}, \bxi_{i} \rangle\right)  \cdot \norm {\bxi_i}_2^{2} \\
        & \leq \lambda_i^{(t)} + \frac{2\eta \norm \bmu_2^{2}}{nm} \sum_{i'=1}^n e^{-\lambda_{i'}^{(t)}} + \frac{3\eta \sigma_p^2 d}{nm} e^{-\lambda_{i}^{(t)}} \\
        & \leq \lambda_i^{(t)} + \frac{3\eta \max\{ n\norm \bmu_2^{2}, \sigma_p^2d\}}{nm} e^{- \tilde \lambda^{(t)}}.
    \end{align*}
    By denoting $c_1= \frac{3\eta \max\{ n\norm \bmu_2^{2}, \sigma_p^2d\}}{nm}$, $c_2= \frac{\eta \max\{\norm{\bmu}_2^2,\sigma_p^2d\}}{12nm^2}$, recall \eqref{lambdalowerupdate}, we have the following inequalities for $t \geq \tilde T_1$,
    \begin{align}
        \lambda_i^{(t+1)} &\leq \lambda_i^{(t)} + c_1 e^{-\tilde \lambda^{(t)}}, \label{C23}  \\ 
        \lambda_i^{(t+1)} &\geq \lambda_i^{(t)} + c_2 e^{-\lambda_i^{(t)}}, \label{C25},
    \end{align}
    utilizing Lemma \ref{tensorpower2} for \eqref{C25}, denote $z_0 = \tilde \lambda^{(\tilde T_1)}$, for each $i\in [n]$, we have
    \begin{align*}
        \lambda_i^{(t)} \geq \log (c_2 (t-\tilde T_1)+ e^{\lambda_i^{(\tilde T_1)}}) \geq \log (c_2 (t-\tilde T_1)+e^{z_0}),
    \end{align*}
    therefore,
    \begin{align*}
        \tilde \lambda^{(t)} \geq \log (c_2 (t-\tilde T_1)+e^{z_0}),
    \end{align*}
    subsituting the above inequality into \eqref{C23}, we get
    \begin{align*}
        \lambda_i^{(t+1)} \leq \lambda_i^{(t)} + \frac{c_1}{c_2} \cdot \frac{c_2}{c_2 (t-\tilde T_1)+e^{z_0}} ,
    \end{align*}
    by taking an integration inequality of the sum, we get
    \begin{align} \label{lambdaupperbound}
        \nonumber \lambda_i^{(t)} &\leq \frac{c_1}{c_2} \log(c_2(t-\tilde T_1) +e^{z_0})+c_1e^{-z_0} +(1-\frac{c_1}{c_2})z_0 \\
        & = \tilde \gamma \log \left(c_2(t-\tilde T_1) + e^{z_0}\right) + c_1e^{-z_0} +(1-\tilde \gamma)z_0,
    \end{align}
    where $\tilde \gamma= \frac{c_1}{c_2}= \frac{36m \max\{ n\norm \bmu_2^{2}, \sigma_p^2d\}}{\max\{\norm \bmu_2^{2}, \sigma_p^2d\}}$.
    %satisfies $\Omega(m) \leq \tilde \gamma \leq O (n m)$. 
    Therefore, by \eqref{ell'lower}, for each $i\in [n]$, we have
    \begin{equation} \label{maxell'}
    \begin{aligned} 
         \left| \ell_i^{\prime (t)} \right| \geq \frac{e^ { (\tilde \gamma -1)z_0 -c_1e^{-z_0}}}{6 \left(c_2(t-\tilde T_1) + e^{z_0}\right)^{\tilde \gamma}},
    \end{aligned}     
    \end{equation}
    
    Therefore, for $t \geq \tilde T_1$ and $j\in \{\pm 1\}$, by \eqref{signalhitk} and \eqref{maxell'}, we have
    \begin{align*}
        \frac{1}{m} \sum_{r=1}^m\gam^{(t+1)} & = \frac{1}{m} \sum_{r=1}^m\gam^{(t)} - \frac{\eta}{nm^2} \sum_{i=1}^n \ell_{i}^{\prime(t)} \sum_{r=1}^m \sigma^\prime \left( \langle \bw_{j,r}^{(t)}, y_{i} \bmu \rangle\right) \cdot \norm \bmu_2^{2} \\
         & \geq \frac{1}{m} \sum_{r=1}^m\gam^{(t)} - \frac{\eta}{nm^2} \sum_{i: y_i=j} \ell_{i}^{\prime(t)} \sum_{r=1}^m \sigma^\prime \left( \langle \bw_{j,r}^{(t)}, j \bmu \rangle\right) \cdot \norm \bmu_2^{2} \\
        & \geq \frac{1}{m} \sum_{r=1}^m \gam^{(t)} + \frac{\eta \norm \bmu_2^{2} e^ { (\tilde \gamma -1)z_0 -c_1e^{-z_0}}}{24m^2 \left(c_2(t-\tilde T_1) + e^{z_0}\right)^{\tilde \gamma}} ,
    \end{align*}
    by taking an integration inequality of the sum, we get
    \begin{align*}
        \frac{1}{m} \sum_{r=1}^m\gam^{(t)} \geq  \frac{\eta \norm \bmu_2^{2}e^ { (\tilde \gamma -1)z_0 -c_1e^{-z_0}}}{24m^2} \left(\frac{1}{c_2 (\tilde \gamma -1)(e^{z_0})^{\tilde \gamma -1}} - \frac{1}{c_2 (\tilde \gamma -1)(c_2(t-\tilde T_1) + e^{z_0})^{\tilde \gamma -1}} \right).
    \end{align*}
    Next, we give an upper bound for $z_0= \tilde \lambda^{(\tilde T_1)}$. Denote $i_{t} = \arg \min_{i} \lambda_{i}^{(t)}$, since \eqref{C23} holds for each $i\in [n]$, we have
    \begin{align*} 
         \tilde \lambda^{(t+1)} &\leq  \lambda_{i_{t}}^{(t+1)} \leq  \lambda_{i_{t}}^{(t)}+ c_1 e^{- \tilde \lambda^{(t)}} = \tilde \lambda^{(t)} + c_1 e^{- \tilde \lambda^{(t)}},
    \end{align*}
    by Lemma \ref{tensorpower2} and $\tilde \lambda^{(0)}= 0$, we have
    \begin{align*}
        z_0= \tilde \lambda^{(\tilde T_1)} \leq \log(c_1 \tilde T_1 + 1) +c_1.
    \end{align*}
    %According to the condition $\sigma_b^2 \geq \frac{C \epsilon \kappa^2 \max\{ n\norm \bmu_2^{2}, \sigma_p^2d\}}{{\eta nm \min\{ \norm \bmu_2^{2}, \sigma_p^2d\}}}$ specified in Condition \ref{condition}. 
    Notice that $t \geq  \tilde T_1 + \Theta \left(e^{c_1} (\tilde{T}_1 + \frac{1}{c_1}) \right)$. When $\frac{c_2 (t-\tilde T_1)}{e^{z_0}} \leq 2$, we have 
    $$\log\left(\frac{c_2 (t-\tilde T_1)+e^{z_0}}{e^{z_0}}\right) \geq \frac{c_2 (t-\tilde T_1)}{2e^{z_0}} \geq \Theta\left( \frac{c_2 (\tilde T_1+ \frac{1}{c_1})}{2c_1 \tilde{T}_1+2} \right) \geq \frac{1}{\tilde{\gamma}-1} \log 2,$$
    we have $(c_2(t-\tilde T_1) + e^{z_0})^{\tilde \gamma -1} \geq 2 (e^{z_0})^{\tilde \gamma -1}$. Moreover, when $\frac{c_2 (t-\tilde T_1)}{e^{z_0}} \geq 2$, we have we have $(c_2(t-\tilde T_1) + e^{z_0})^{\tilde \gamma -1} \geq 2 (e^{z_0})^{\tilde \gamma -1}$ as well. Therefore, we have
    \begin{align} \label{C30}
        \frac{1}{m}\sum_{r=1}^m \gam^{(t)} \geq \frac{\eta \norm \bmu_2^{2}e^ { (\tilde \gamma -1)z_0 -c_1e^{-z_0}}}{48m^2 c_2 (\tilde \gamma -1)e^{z_0 (\tilde \gamma -1)}} = \frac{\eta \norm \bmu_2^{2}e^ { -c_1e^{-z_0}}}{48m^2 c_2 (\tilde \gamma -1)} = \Omega \left(\frac{n \norm \bmu_2^{2}}{e^{c_1} m \max\{n \norm \bmu_2^{2}, \sigma_p^2d\}} \right),
    \end{align}
    where the last equality is by $z_0 \geq 0$.
    Thus we complete the proof.
\end{proof}

Based on the above lemma, we can demonstrate the test error bound for DP-GD in the following lemma.
\begin{lemma} \label{NGDtesterror}
    Under Condition \ref{condition2}, let $\tilde{T}_2 \leq t \leq \tilde{T}^*$ and satisfies that
     $$\eta \sigma_b \norm {\bmu}_2 \sqrt{t} \leq \frac{n\norm \bmu_2^{2}}{C m \max\{n \norm \bmu_2^{2}, \sigma_p^2d\}}$$
      for some large constant $C$,
     we have 
    \begin{align*}
         \mathcal{R}_{\mathcal{D}}(\bW^{(t})) \leq \exp \left( -  C_2 \left(\frac{ n \norm \bmu_2^{2}}{e^{c_1} m \max\{n \norm \bmu_2^{2}, \sigma_p^2d\}}  \cdot \frac{1}{\sigma_0\sigma_p \sqrt{d} + \frac{\sigma_p m}{\norm \bmu_2} + \frac{mn}{  \sqrt{d}}  +  \eta \sigma_b \sigma_p \sqrt{d t}} \right)^2 \right).
    \end{align*}
\end{lemma}

\begin{proof} [Proof of Lemma \ref{NGDtesterror}]
    For the time period $\tilde{T}_2 \leq t \leq \tilde{T}^*$, notice that for $j\in \{ \pm 1\}$,
    \begin{equation} \label{jsignal}
    \begin{aligned}
         \frac{1}{m} \sum_{r=1}^m \sigma \left(\langle\bw_{j,r}^{(t)}, j \bmu\rangle\right) &= \frac{1}{m} \sum_{r=1}^m\sigma \left(\langle\bw_{j,r}^{(t)}, j \bmu\rangle -\gamma_{j,r}^{(t)}+ \gamma_{j,r}^{(t)} \right) \\
         & \geq \frac{1}{m} \sum_{r=1}^m\sigma \left(\gamma_{j,r}^{(t)} - 0.1\kappa - \theta \right) \\
         & \geq \frac{1}{m} \sum_{r=1}^m \gamma_{j,r}^{(t)} - 1.1 \kappa -\theta  \\
         & \geq \Omega \left(\frac{n \norm \bmu_2^{2}}{e^{c_1} m \max\{n \norm \bmu_2^{2}, \sigma_p^2d\}} \right),
    \end{aligned}
    \end{equation}
    where the first inequality is by Lemma \ref{innerprodectwjrxiNGD}, the second inequality is by $\sigma(z-b) \geq z-b-\kappa$ for $z\geq 0$, and the last inequality is by Lemma \ref{signallowerbound}. Similarly, we have
    \begin{equation} \label{-jsignal}
    \begin{aligned}
         \frac{1}{m} \sum_{r=1}^m \sigma \left(\langle\bw_{j,r}^{(t)}, -j \bmu\rangle\right) &= \frac{1}{m} \sum_{r=1}^m\sigma \left(\langle\bw_{j,r}^{(t)}, -j \bmu\rangle +\gamma_{j,r}^{(t)} - \gamma_{j,r}^{(t)} \right) \\
         & \leq \frac{1}{m} \sum_{r=1}^m\sigma \left( 0.1\kappa +  8\eta \sigma_b \norm {\bmu}_2 \sqrt{t}  \log ^2\left(\frac{32mt}{\delta}\right)  -\gamma_{j,r}^{(t)} \right) \\
         & \leq 0.1 \kappa + \tilde{\Omega} \left(\eta \sigma_b \norm {\bmu}_2 \sqrt{t} \right) ,
    \end{aligned}
    \end{equation}
    where the first inequality is by Lemma \ref{innerprodectwjrxiNGD} and Lemma \ref{sumofnoisebound}, the second inequality is by $\sigma(z) \leq z$ for $z\geq 0$.
    
    Moreover, by the definition of the signal-noise decomposition \eqref{sndecompositionNGD} of the NGD algorithm
    \begin{equation*} 
        \wjr^{(t)}= \wjr^{(0)} + j \cdot \gam^{(t)} \cdot \norm \bmu_2^{-2} \cdot \bmu + \sum_{i=1}^n \orho^{(t)} \cdot \norm {\bxi_i}_2^{-2} \cdot \bxi_i + \sum_{i=1}^n \urho^{(t)} \cdot \norm {\bxi_i}_2^{-2} \cdot \bxi_i - \eta \sum_{k=0}^{t-1} \bb_{j,r,k}.
    \end{equation*}
    By the triangle inequality, utilizing Proposition \ref{proposition:NGDgrowthbound}, Lemma \ref{innerproductxi}, Lemma \ref{sumofnoisenorm}, and Lemma \ref{innerproductw0}, we have
    \begin{align*}
        \norm{\bw_{j,r}^{(t)}}_2  &= \bigg\|\bw_{j,r}^{(0)} + j \cdot \gam^{(t)} \cdot \frac{\bmu}{\|\bmu\|_{2}^{2}} + \sum_{ i = 1}^n \orho^{(t)} \cdot \frac{\bxi_{i}}{\|\bxi_{i}\|_{2}^{2}} + \sum_{ i = 1}^n \urho^{(t)} \cdot \frac{\bxi_{i}}{\|\bxi_{i}\|_{2}^{2}} - \eta \sum_{k=0}^{t-1} \bb_{j,r,k} \bigg\|_{2}\\
        & \leq \|\bw_{j,r}^{(0)}\|_{2} +  \frac{\gamma_{j,r}^{(t)}}{\|\bmu\|_{2}} + \sum_{ i = 1}^n  \frac{\orho^{(t)}}{\|\bxi_{i}\|_{2}} + \sum_{ i = 1}^n  \frac{|\urho^{(t)}|}{\|\bxi_{i}\|_{2}} +  \eta \left\|\sum_{k=0}^{t-1} \bb_{j,r,k} \right\|_2\\
        & \leq 2\sigma_0 \sqrt{d} + \frac{4m \log t}{\norm \bmu_2} + \frac{4mn \log t}{ \sigma_p \sqrt{d}} + 2 \eta \sigma_b \sqrt{2d t \log \left(\frac{2m}{\delta}\right)} \\
        &= \tilde O\left(\sigma_0 \sqrt{d} + \frac{m}{\norm \bmu_2} + \frac{mn}{ \sigma_p \sqrt{d}} + \eta \sigma_b \sqrt{d t} \right) \\
        %& =  \tilde O\left(\frac{m}{\norm \bmu_2} +  \eta \sigma_b \sqrt{d t}\right) ,
    \end{align*}
    %where we utilize the condition $\sigma_0= O\left(\frac{\kappa}{\max\{\norm{\bmu}_2,\sigma_p \sqrt{d}\} \sqrt{\log\left(\frac{mn}{\delta} \right)}} \right)$  specified in Condition \ref{condition2}.
    For a new data sample $(\bx,y)$, we have
    \begin{align*}
        y \cdot f(\bW^{(t)},\bx) &= \frac{1}{m} \sum_{r=1}^m \left[\sigma(\la \bw_{y,r}^{(t)}, y\cdot \bmu \ra) + \sigma(\la \bw_{y,r}^{(t)}, \bxi \ra) \right] \\
        &- \frac{1}{m} \sum_{r=1}^m \left[\sigma(\la \bw_{-y,r}^{(t)}, y\cdot \bmu \ra) + \sigma(\la \bw_{-y,r}^{(t)}, \bxi \ra) \right].
    \end{align*} 
    Note that $\la\bw_{j,r}^{(t)}, \bxi\ra \sim \cN(0, \sigma_{p}^{2}\|\bw_{j,r}^{(t)}\|_{2}^{2})$. Therefore, with probability at least 
    \begin{align*}
         1 - \exp \left( -  C_2 \left(\frac{ n \norm \bmu_2^{2}}{e^{c_1} m \max\{n \norm \bmu_2^{2}, \sigma_p^2d\}}  \cdot \frac{1}{\frac{\sigma_p m}{\norm \bmu_2} +  \eta \sigma_b \sigma_p \sqrt{d t}} \right)^2 \right),
        %& \geq 1- \exp \left( -\frac{C_2 \epsilon n \norm {\bmu}_2^4 \max\{\norm {\bmu}_2^2 , \sigma_p^2 d \}}{{\eta m^4 \sigma_b^2 \sigma_p^2 d \max^2\{n\norm {\bmu}_2^2 , \sigma_p^2 d \} }}  \right),
        \end{align*}
    we have
    \begin{align} \label{-jnoise}
    |\la\bw_{j,r}^{(t)}, \bxi\ra| =  O \left(\frac{n \norm \bmu_2^{2}}{e^{c_1} m \max\{n \norm \bmu_2^{2}, \sigma_p^2d\}} \right),
    \end{align}
    where $C_2$ is a positive constant. Consider that $t=T_2$, 
    \begin{align*}
        y \cdot f(\bW^{(t)},\bx) & \geq \Omega \left(\frac{n\norm \bmu_2^{2}}{m \max\{n \norm \bmu_2^{2}, \sigma_p^2d\}} \right) - \tilde{\Omega} \left( \eta \sigma_b \norm {\bmu}_2 \sqrt{t} \right)  \geq 0,
    \end{align*}
    where the first inequality is by \eqref{jsignal}, \eqref{-jsignal}, and \eqref{-jnoise}, 
    the second inequality is by the condition that  $\eta \sigma_b \norm {\bmu}_2 \sqrt{t} \leq \frac{n\norm \bmu_2^{2}}{C m \max\{n \norm \bmu_2^{2}, \sigma_p^2d\}}$ for some large constant $C$. Therefore, we have
    \begin{align*}
         \mathcal{R}_{\mathcal{D}}(\bW^{(t})) \leq \exp \left( -  C_2 \left(\frac{ n \norm \bmu_2^{2}}{e^{c_1} m \max\{n \norm \bmu_2^{2}, \sigma_p^2d\}}  \cdot \frac{1}{\sigma_0\sigma_p \sqrt{d} + \frac{\sigma_p m}{\norm \bmu_2} + \frac{mn}{  \sqrt{d}} +  \eta \sigma_b \sigma_p \sqrt{d t}} \right)^2 \right).
    \end{align*}
    Thus, we complete the proof.
\end{proof}

\subsection{Privacy Guarantee}
In this subsection, we give the differential privacy result for DP-GD.
\begin{definition}[Differential Privacy \cite{dwork2014algorithmic}]\label{def:DP}
We say that a randomized algorithm $\A$ satisfies $(\gep, \delta)$-DP if, for any two neighboring datasets $S$ and $S'$  and any event $E$ in the output space of $\A$,  it holds 
$
\mathbb{P}(\A (S)\in E) \le e^\gep \mathbb{P}(\A(S')\in E)+ \gd.
$
We say $\A$ satisfies $\gep$-DP if $\gd=0$. 
\end{definition} 
R\'{e}nyi differential privacy (RDP), introduced by \cite{mironov2019}, is a relaxation of the standard DP  framework that provides a more flexible and fine-grained analysis of privacy loss. 
\begin{definition}[RDP \cite{mironov2019}]\label{def:RDP}
For $\lambda > 1$, $\rho > 0$, a randomized algorithm  $\A$ satisfies $(\lambda, \rho)$-RDP, if,  for all neighboring datasets $S$ and $S'$, we have 
    \begin{align*}       D_{\lambda}\big(\A(S)\parallel \A(S')\big):= \frac{1}{\lambda-1}\log \int  \Big( \frac{ P_{\A(S)}(\theta) }{ P_{\A(S')}(\theta) }  \Big)^\lambda    d P_{\A(S')}(\theta) \le \rho,
    \end{align*} 
    where $P_{\A(S)}(\theta)$ and $P_{\A(S')}(\theta)$ are the density of $\A(S) $ and $\A(S')$, respectively. 
\end{definition}

\begin{definition}[$\ell_2$-sensitivity]\label{def:sensitivity}
The $\ell_2$-sensitivity of a function (mechanism) $\mathcal{M}:\mathcal{Z}^n \rightarrow \mathcal{W}$ is defined as 
$
\Delta  = \sup_{S\simeq S'} \|\mathcal{M}(S) - \mathcal{M}(S')\|_2,
$ where $S$ and $S'$ are neighboring datasets. 
\end{definition}
\begin{lemma}[Gaussian mechanism \cite{mironov2019}]\label{lem:gaussian-rdp} Consider a function $\mathcal{M}:  \mathcal{Z}^n\rightarrow \mathcal{R}^d$ with the $\ell_2$-sensitivity $\Delta$,  and a dataset $S\subset\mathcal{Z}^n$.  
{The Gaussian mechanism $\mathcal{G}(S,\sigma)=\mathcal{M}(S)+\mathbf{b}$, where $\mathbf{b}\sim \mathcal{N}(0,\sigma^2\mathbf{I}_d)$, } satisfies $(\lambda,\frac{\lambda \Delta^2}{2\sigma^2})$-RDP. 
\end{lemma}

The following lemma established a connection $(\epsilon,\delta)$-DP and RDP.
\begin{lemma}[From RDP to $(\epsilon,\delta)$-DP \cite{mironov2019}]\label{lemma:RDP_to_DP}
	If a randomized algorithm  $\mathcal{A}$ satisfies $(\lambda,\rho)$-RDP, then $\mathcal{A}$ satisfies $(\rho+\log(1/\delta)/(\lambda-1),\delta)$-DP for all $\delta\in(0,1)$.
\end{lemma}

The following post-processing property enables flexible use of private data outputs while preserving rigorous privacy guarantees.
\begin{lemma}[Post-processing \cite{mironov2019}]\label{lemma:post-processing}
Let  $\A: \mathcal{Z}^n \rightarrow \mathcal{W}_1 $  satisfy $(\lambda, \rho)$-RDP  and $f: \mathcal{W}_1 \rightarrow \mathcal{W}_2$ be an arbitrary function. Then $f \circ \A : \mathcal{Z}^n \rightarrow \mathcal{W}_2$ satisfies $(\lambda, \rho)$-RDP.     
\end{lemma}
 
 We say a sequence of mechanisms $(\A_1,\ldots,\A_k)$ are chosen adaptively if $\A_i$ can be chosen based on the outputs of the previous mechanisms  $\A_1(S),\ldots,\A_{i-1}(S)$ for any $i\in[k]$.  
\begin{lemma}[Composition of RDP \cite{mironov2019}]\label{lem:composition_RDP} For each $i\in[k]$, assume $\A_i$ satisfying  $(\lambda, \rho_i)$-RDP.  The following statements hold true. 
\begin{enumerate}[label=({\alph*})]
\item (Parallel composition) If a mechanism $\A$ simultaneous release of $\A_k$ for all $k$, i.e., $\A=(\A_1,\ldots,\A_k)$, then $\A$ satisfies  $(\lambda, \sum_{i=1}^k \rho_i)$-RDP.
    \item (Adaptive composition)  If a mechanism $\A$ consists of a sequence of  adaptive mechanisms $(\A_1,\ldots,\A_k)$, then $\A$ satisfies $(\lambda, \sum_{i=1}^k \rho_i)$-RDP.
\end{enumerate}
\end{lemma}

In the following, we give our main differential privacy result.
\begin{lemma} [Privacy Guarantee] \label{lemma:dp}
	Assume $d= \Omega(\log\left(\frac{12n}{\delta}\right))$. Then  Noisy GD with $T$ iterations satisfies 
	%$( \frac{  \lambda m\sum_{t=1}^T\Delta_{\ell_2,t}^2}{ \sigma_b^2}+\frac{\log(2/\delta)}{\lambda-1}, \delta )$-DP for any $\lambda>1$. 
	$( \frac{T\lambda(2\|\bmu\|_2^2+ 3\sigma_p^2 d)}{\sigma_b^2 n^2 m}+\frac{\log(2/\delta)}{\lambda-1}, \delta )$-DP for any $\lambda>1$.
\end{lemma}

\begin{proof}[Proof of Lemma~\ref{lemma:dp}]
  Let $S$ and $S'$ be two neighboring datasets that differ on the $k$-th data point.  For each step $t$, for any $j\in\{+1,-1\}$ and $r\in[m]$, given $\bW^{(t)}$,  there holds
    \begin{align}
        &\|\nabla_{\bw_{j,r}} L_S(\bW^{(t)})-\nabla_{\bw_{j,r}} L_{S'}(\bW^{(t)})\|_2\nonumber\\
        &=\Big\|\frac{1}{nm}\nabla_{\bw_{j,r}} \ell\left[y_k \cdot f(\bW^{(t)},\bx_k) \right]-\frac{1}{nm}\nabla_{\bw_{j,r}} \ell\left[y_k \cdot f(\bW^{(t)},\bx_k) \right]\Big\|_2 \nonumber\\
        &\le  \frac{1}{nm}\Big\| \ell_{k}^{\prime(t)} \cdot \sigma^\prime \left( \langle \bw_{j,r}^{(t)}, y_{i} \bmu \rangle\right) \cdot j \bmu   \Big\|_2 +  \frac{1}{nm}\Big\|    \ell_{k}^{\prime(t)} \cdot \sigma^\prime \left( \langle \bw_{j,r}^{(t)}, \bxi_{k} \rangle\right) \cdot jy_{k} \bxi_{k}\Big\|_2\nonumber\\
        &\le  \frac{1}{nm}\big( \big\| \bmu \big\|_2+ \big\| \bxi_k \big\|_2 \big).
    \end{align}
    From Lemma~\ref{innerproductxi}, if $d= \Omega(\log (\frac{6n}{\delta_1}))$, we know with probability at least $1-\delta_1$, there holds
    \[ \|\bxi_k\|_2^2 \le \frac{3\sigma_p^2 d}{2}.  \]
    Hence, with probability at least $1-\delta_1$,  the $\ell_2$-sensitivity of the gradient $\nabla_{\bw_{j,r}} L_S(\bW^{(t)})$ at each iteration is
    \begin{align}\label{eq:sensitivity}
        \Delta_{\ell_2}=  \frac{1}{nm}\Big( \|\bmu\|_2+\sqrt{\frac{3\sigma_p^2 d}{2} } \Big).
    \end{align}
    Lemma~\ref{lem:gaussian-rdp} implies that the mechanism $\nabla_{\bw_{j,r}} L_S(\bW^{(t)})+\bb_{j,r,t}$ satisfies $(\lambda,\frac{\lambda\Delta_{\ell_2}^2}{2\sigma_b^2})$-RDP. Furthermore, applying post-processing property of RDP we know $\bw_{j,r}$ satisfies $(\lambda,\frac{\lambda\Delta_{\ell_2}^2}{2\sigma_b^2})$-RDP.
    Finally, from parts (a) and (b) in Lemma~\ref{lem:composition_RDP}, we know $\bW^{(t)}$ is $(\lambda,\frac{ \lambda m\Delta_{\ell_2}^2}{ \sigma_b^2})$-RDP, and further $\bW^{(T)}$ is $(\lambda,\frac{ T\lambda m\Delta_{\ell_2}^2}{ \sigma_b^2})$-RDP.  

Using the connection between $(\epsilon,\delta)$-DP and RDP (see Lemma~\ref{lemma:RDP_to_DP}), we get $\bW^{(T)}$ satisfies $(\frac{ T \lambda m\Delta_{\ell_2}^2}{ \sigma_b^2}+\frac{\log(1/\delta_2)}{\lambda-1},\delta_2)$-DP. Note that this event happens with probability at least $1-\delta_1$ over the randomness of $\bxi_k$. Then, it holds that $\bW^{(T)}$ satisfies $(\frac{  T \lambda m\Delta_{\ell_2}^2}{ \sigma_b^2}+\frac{\log(1/\delta_2)}{\lambda-1},\delta_1+\delta_2)$-DP.  Setting $\delta_1=\delta_2=\frac{\delta}{2}$ and noting \eqref{eq:sensitivity} holds, we know $\bW^{(T)}$ is $( \frac{  T \lambda(2\norm{\bmu}_2^2+3\sigma_p^2d)}{ \sigma_b^2 n^2 m}+\frac{\log(2/\delta)}{\lambda-1}, \delta )$-DP if $d= \Omega(\log (\frac{12n}{\delta}))$. 
%Finally, we considering the best choice of $\lambda$. 
    Thus we complete the proof.
\end{proof}

\subsection{A Condition for Good Test Error and DP Guarantee}
By utilizing Lemma \ref{NGDtesterror} and Lemma \ref{lemma:dp}, we can identify the conditions such that we can simultaneously achieve good test error and DP guarantees. These conditions are further summarized in Condition \ref{condition2}.
\begin{lemma} \label{lemma37}
	Suppose that 
	\begin{align*}
		&\frac{Cm^3 \kappa^2  \max\{ \norm \bmu_2^{2}, \sigma_p^2d\}}{ \norm \bmu_2^{2} \min\{\norm{\bmu}_2^2,\sigma_p^2d\}} \leq \eta \leq  \frac{m}{C  \norm \bmu_2^{2}}, \\
		& \sigma_p^2 \leq \min\left\{ \frac{\norm \bmu_2^2}{Cm^4}, \frac{n \norm \bmu_2^2}{Cd}\right\}, \\
		& 	\sigma_0 \leq \frac{1}{C m \sigma_p \sqrt{d}},\\
		& d\geq C m^4 n^2,
	\end{align*}
	by choosing $\sigma_b= \Theta\left( \sqrt{\frac{\norm \bmu_2^{2}}{\eta m^3  \max\{\norm \bmu_2^{2}, \sigma_p^2d\}}}\right)$, and $\tilde{T}_2=  \Theta \left( \frac{nm}{\eta \max\{ n\norm \bmu_2^{2}, \sigma_p^2d\}}\right)$, we have $	\mathcal{R}_{\mathcal{D}}(\bW^{(\tilde{T}_2))}  \leq 0.01$. Moreover, the DP-GD with $\tilde{T}_2$ iterations satisfies $\left( \frac{C_4 m^3  \max^2\{ \norm \bmu_2^{2}, \sigma_p^2d\}}{n^2 \norm \bmu_2^{4}} \log \frac{2}{\delta}, \delta \right)$-DP for some positive constant $C_4$.
\end{lemma}

\begin{proof} [Proof of Lemma \ref{lemma37}]
	According to the choice of $\sigma_b$ and the condition that 
	$$\eta \geq \frac{Cm^3 \kappa^2 \max\{ n\norm \bmu_2^{2}, \sigma_p^2d\}  \max\{ \norm \bmu_2^{2}, \sigma_p^2d\}}{ n\norm \bmu_2^{4} \min\{\norm{\bmu}_2^2,\sigma_p^2d\}} ,$$
    we have $\eta \sigma_b  \geq \Theta \left( \frac{\kappa}{\min \{\norm{\bmu}_2,\sigma_p \sqrt{d}\}}\right) $, it follows that
	$$\tilde T_1=  \Theta \left(\frac{\kappa^2}{\eta^2 \sigma_b^2 \min\{\norm{\bmu}_2^2,\sigma_p^2d\}} \right) = O(1).$$
	 Moreover, notice that by $\eta \leq \frac{nm}{C \max\{ n\norm \bmu_2^{2}, \sigma_p^2d\}}$,  we have $c_1= \frac{3\eta \max\{ n\norm \bmu_2^{2}, \sigma_p^2d\}}{nm}= O(1)$, therefore
	$$\tilde T_2= \tilde T_1 + \Theta \left(e^{c_1} (\tilde{T}_1 + \frac{1}{c_1}) \right)= \Theta\left(\frac{1}{c_1}\right)= \Theta \left( \frac{nm}{\eta \max\{ n\norm \bmu_2^{2}, \sigma_p^2d\}}\right),$$
	therefore, by Lemma~\ref{lemma37}, the DP-GD with $\tilde T_2$ iterations satisfies 
	$\left( \frac{\tilde T_2 \lambda(2\|\bmu\|_2^2+ 3\sigma_p^2 d)}{\sigma_b^2 n^2 m}+\frac{\log(2/\delta)}{\lambda-1}, \delta \right)$-DP for any $\lambda>1$. Notice that we choose $\sigma_b= \Theta\left( \sqrt{\frac{n\norm \bmu_2^{4}}{\eta m^3  \max\{n \norm \bmu_2^{2}, \sigma_p^2d\}  \max\{\norm \bmu_2^{2}, \sigma_p^2d\}}}\right)$, then we have
	\begin{align*}
		\frac{\tilde T_2 (2\|\bmu\|_2^2+ 3\sigma_p^2 d)}{\sigma_b^2 n^2 m} &= \Theta \left(\frac{ \max\{\|\bmu\|_2^2,\sigma_p^2 d\}}{\eta \sigma_b^2 n \max\{n \|\bmu\|_2^2,\sigma_p^2 d\}} \right) = \Theta \left(\frac{m^3  \max^2\{ \norm \bmu_2^{2}, \sigma_p^2d\}}{n^2 \norm \bmu_2^{4}} \right) .
	\end{align*}
	 Therefore,  the DP-GD with $\tilde T_2$ iterations satisfies $\left( \frac{C_4 m^3  \max^2\{ \norm \bmu_2^{2}, \sigma_p^2d\}}{n^2 \norm \bmu_2^{4}} \log \frac{2}{\delta}, \delta \right)$-DP for some positive constant $C_4$, by choosing $\lambda= 1+ \frac{n^2 \norm \bmu_2^{4}}{m^3  \max^2\{ \norm \bmu_2^{2}, \sigma_p^2d\}}$.
	
	 Furthermore, according to our choice of $\sigma_b$, we have
	\begin{align*}
		\eta \sigma_b  \|\bmu\|_2 \sqrt{\tilde T_2}= O \left( \sqrt{\frac{\eta \sigma_b^2 nm \max\{ \norm \bmu_2^{2}, \sigma_p^2d\}}{ \max\{ n\norm \bmu_2^{2}, \sigma_p^2d\}}}\right) \leq  \frac{n\norm \bmu_2^{2}}{C m \max\{n \norm \bmu_2^{2}, \sigma_p^2d\}},
	\end{align*}
	similarly, we have $$\eta \sigma_b  \sigma_p \sqrt{d \tilde T_2}\leq  \frac{n\norm \bmu_2^{2}}{C m \max\{n \norm \bmu_2^{2}, \sigma_p^2d\}}.$$
	Therefore, by Lemma \ref{NGDtesterror}, we have
	\begin{align*}
		\mathcal{R}_{\mathcal{D}}(\bW^{(\tilde{T}_2})) &\leq \exp \left( -  C_2 \left(\frac{ n \norm \bmu_2^{2}}{e^{c_1} m \max\{n \norm \bmu_2^{2}, \sigma_p^2d\}}  \cdot \frac{1}{\sigma_0\sigma_p \sqrt{d} + \frac{\sigma_p m}{\norm \bmu_2} + \frac{mn}{  \sqrt{d}} +  \eta \sigma_b \sigma_p \sqrt{d \tilde{T}_2}} \right)^2 \right)  \leq 0.01,
	\end{align*}
	by choosing $C$ to be large enough, and the condition that 
    $$\sigma_p \leq \frac{n \norm \bmu_2^{3}}{Cm^2 \max\{n \norm \bmu_2^{2}, \sigma_p^2d\}}, \sigma_0 \leq \frac{n\norm \bmu_2^{2}}{C m \sigma_p \sqrt{d} \max\{n \norm \bmu_2^{2}, \sigma_p^2d\}}, d\geq \frac{C m^4 \max^2\{n \norm \bmu_2^{2}, \sigma_p^2d\}}{\norm \bmu_2^{4}}.$$
    Thus we complete the proof by utilizing the SNR condition that $\frac{\sigma_p^2 d}{n \norm{\bmu}_2^2}= O(1)$ to ensure that we can achieve a good DP result.
\end{proof}

\end{document}